\documentclass[letterpaper]{article} 
\usepackage{aaai24arXiv}  
\usepackage{times}  
\usepackage{helvet}  
\usepackage{courier}  
\usepackage[hyphens]{url}  
\usepackage{graphicx} 
\usepackage{natbib}  
\usepackage{caption} 
\usepackage{algorithm}
\usepackage{algpseudocode}
\usepackage{makecell}
\usepackage{subfigure}

\usepackage{rotating}

\usepackage{newfloat}
\usepackage{listings}
\DeclareCaptionStyle{ruled}{labelfont=normalfont,labelsep=colon,strut=off} 
\floatstyle{ruled}
\newfloat{listing}{tb}{lst}{}
\floatname{listing}{Listing}
\usepackage{diagbox}
\usepackage{amsmath} 
\usepackage{amsthm} 

\usepackage[utf8]{inputenc} 
\usepackage{booktabs}       
\usepackage{amsfonts}       
\usepackage{nicefrac}       
\usepackage{microtype}      
\usepackage{xcolor,colortbl}         
\usepackage{tabularx}
\usepackage{multirow} 

\newtheorem{theorem}{Theorem}
\newtheorem{corollary}{Corollary}
\newtheorem{lemma}{Lemma}

\newcommand{\ie}{\textit{i.e.}}
\newcommand{\eg}{\textit{e.g.}}

\newcommand{\focus}[1]{}

\title{Curriculum-Enhanced Residual Soft An-Isotropic Normalization for Over-smoothness in Deep GNNs}

\author{
    Jin Li\textsuperscript{\rm 1, \rm 2},
    Qirong Zhang\textsuperscript{\rm 1},
    Shuling Xu\textsuperscript{\rm 1},
    Xinlong Chen\textsuperscript{\rm 1},
    Longkun Guo\textsuperscript{\rm 1, \rm 3},
    Yang-Geng Fu\textsuperscript{\rm 1\thanks{Corresponding author}}
}
\affiliations{
    \textsuperscript{\rm 1}College of Computer and Data Science, Fuzhou University, Fuzhou, China\\
    \textsuperscript{\rm 2}AI Thrust, Information Hub, HKUST (Guangzhou), Guangzhou, China\\
    \textsuperscript{\rm 3}Shandong Provincial Key Laboratory of Computer Networks, \\Shandong Fundamental Research Center for Computer Science, Jinan, China\\

    \{jslijin2015, shuling\_xu\}@outlook.com, \{jiusizhang1, fjxinlong, longkun.guo\}@gmail.com, fu@fzu.edu.cn
}

\usepackage{bibentry}

\begin{document}

\maketitle

\begin{abstract}
Despite Graph neural networks' significant performance gain over many classic techniques in various graph-related downstream tasks, their successes are restricted in shallow models due to over-smoothness and the difficulties of optimizations among many other issues. In this paper, to alleviate the over-smoothing issue, we propose a soft graph normalization method to preserve the diversities of node embeddings and prevent indiscrimination due to possible over-closeness. Combined with residual connections, we analyze the reason why the method can effectively capture the knowledge in both input graph structures and node features even with deep networks. Additionally, inspired by Curriculum Learning that learns easy examples before the hard ones, we propose a novel label-smoothing-based learning framework to enhance the optimization of deep GNNs, which iteratively smooths labels in an auxiliary graph and constructs many gradual non-smooth tasks for extracting increasingly complex knowledge and gradually discriminating nodes from coarse to fine. The method arguably reduces the risk of overfitting and generalizes better results. Finally, extensive experiments are carried out to demonstrate the effectiveness and potential of the proposed model and learning framework through comparison with twelve existing baselines including the state-of-the-art methods on twelve real-world node classification benchmarks.
\end{abstract}

\section{Introduction}
\label{Sec:Introduction}
\focus{Part 1: GNN: introduction, applications, short history}
\focus{Introduction to GNNs:}Graph neural networks (GNNs) \cite{p1_survey} are widely used state-of-the-art techniques to solve many tasks on graph (\eg, semi-supervised node classification \cite{p69_semi_supervised_node_classification}, link prediction \cite{p67_Neo-gnns_link_prediction}, graph classification \cite{p68_graph_classification}, and community detection \cite{p70_deep_community_detection}, etc).
\focus{Their applications:}Also, GNNs have achieved outstanding results recently in many domains including texts \cite{p95_GNN_for_nlp_question_generation}, images \cite{p96_GNN_for_image_classification}, traffic \cite{p97_GNN_for_traffic}, molecule \cite{p94_GNN_for_molecule}, and even electroencephalogram (\ie, EEG) \cite{p33_EEG} compared to classic methods (\eg, CNN and RNN).
\focus{Their short history:}They are previously derived in spectral domain based on the eigen-decomposition of graph Laplacian (\eg, Spectral-CNN~\cite{p2_spectralCNN} and ChebNet~\cite{p119_ChebNet}).
Graph Convolutional Network (GCN)~\cite{p4_GCN} is proposed to accelerate it via a linear approximation of universal filters on graph signals and also give an intuitive interpretation on spatial domain, \ie,  \textit{message passing}, which is further developed by SGC~\cite{p37_SGC}, GAT~\cite{p8_GAT}, GIN~\cite{p7_GIN}, and MPNN~\cite{p36_MPNN}.

\focus{Part 2: deep graph and limitation issues:}
\focus{why deep?}More recently, some deep GNNs \cite{p39_GCNII,p46_deep_deepgcns} are proposed to further improve the expressive power of GNNs in light of successes of other deep neural networks, \eg, CNN and RNN.
\focus{why cannot deep?}But unfortunately, GNNs are not easy to go deep and often suffer from severe performance degradation due to, \eg, over-smoothness \cite{p51_towards_deeper_gnn}, difficulty in optimization \cite{p55_meannorm}, memory limitation \cite{p64_training_grakph_neural_networks_1000}, time consumption \cite{p64_training_grakph_neural_networks_1000}, and over-squashing \cite{p35_over-squashing}.
\focus{What are our works?}In this paper, we mainly focus on the first two problems, \ie, improvements will be devised with respect to the following two aspects: 1) structures; and 2) the learning process.

\focus{Part 3: Existing approaches to solving over-smoothing:}
\focus{Summarization}In order to alleviate over-smoothness, various kinds of techniques \cite{p52_deeper_gnn_tricks_overview} are proposed to prevent over-closeness of node embeddings or reduce the extent of aggregations including graph normalization \cite{p54_nodenorm}, residual connections \cite{p39_GCNII}, and random dropping \cite{p57_dropnode}, etc.
\focus{Their strengths and concerns?}\focus{graph normalization:}Existing graph normalization techniques directly operate with norms, means, variances or distances of embeddings, and can effectively reduce the over-closeness.
However, it's still unclear what or how much knowledge they can preserve in deep layers via only these numeric-related operations.
\focus{residual connections:}Some residual connections methods are proven to keep useful feature semantics as GNNs go deep, but they may risk missing structural knowledge, \eg, GCNII \cite{p39_GCNII}.
\focus{random dropping:}Random dropping approaches are devised mainly for regularization without theoretical guarantees for alleviating information loss in over-smoothness and seldom perform competitively compared to other state-of-the-arts.

\focus{Part 4: How do we alleviate over-smoothing issue?}
\focus{The categorization of our solution in structure and its short introduction:}Thus for structures, we propose \textit{R-SoftGraphAIN}, a residual connections-based soft graph normalization layer, which can be viewed as a combination of a \textit{novel} soft graph normalization operation and two \textit{improved} residual connections.
\focus{effects compared to other methods:}Compared to Pairnorm \cite{p53_pairnorm}, it can normalize embeddings in an an-isotropic manner instead of equally treating all nodes.
Compared to GCNII, it can be shown to preserve relatively high frequent structural knowledge instead of over-emphasizing features and can be viewed as a generalization of GCNII.
\focus{theoretical justification:}In Sec. \ref{Sec:ResGraphAIN},  theoretical analysis is carried out to show its following characteristics as the depth approaches infinity: 1) The mean distance of pairwise node embeddings will be kept nearly constant similar to Pairnorm; 2) The diversities of these signals are maximized; 3) It tends to extract the most $d$ lowest frequent components of structural knowledge; 4) It never forgets the original feature's information.
\focus{Experimental evaluations:}Experimental evaluations prove its effectiveness due to significant performance gain compared to others.

\focus{Part 5: How we ease the training process of deep GNNs?}On the other hand, in order to ease the optimization of deep GNNs, we borrow ideas from Curriculum Learning (CL) \cite{p74_curriculum_learning_survey_1,p75_curriculum_learning_survey_2}.
\focus{The basic idea in Curriculum Learning:}In CL, models are encouraged to first learn from easy examples  and then examples with   gradually increased difficulty \cite{p76_original_curriculum_learning_2009_examples_from_easy_to_hard}.
\focus{More general ideas in Curriculum Learning and their applications in many domains:}The idea has been generalized to devise better curriculum applied to various scenarios in many domains including texts and images via, \eg, designing multifarious tasks with increasing difficulties \cite{p78_curriculum_increase_task_difficult}, gradually unleashing expressive powers of models \cite{p77_curriculum_by_gaussian_temperature_smooth_2020_increase_model_capacity_or_expressiveness}, or defining \textit{pacing} functions to decide to which extent to learn from a task \cite{p83_curriculum_with_scoring_and_pacing_functions_2019}.
\focus{Transition to our learning framework:}However, graphs contain specific structures and semi-supervised learning has a special setting, which may require a careful curriculum design, but few prior works focus on this (see Sec. \ref{Sec:Related_work_Curriculum_Learning}).
Therefore in this paper, we give a simple yet effective label-smooth-based example called \textit{SmoothCurriculum}.
\focus{Our learning framework:}More specifically, we first employ \textit{label propagation} to estimate unknown labels, and all labels are iteratively smoothed in an auxiliary graph built via a pre-trained teacher model in order to construct gradual non-smooth tasks.
\focus{Effects of our learning framework:}From the analysis in Sec. \ref{Sec:SmoothCurriculum}, this learning framework encourages graph encoders to extract increasingly complex knowledge and learn to gradually discriminate nodes from coarse to fine\footnote{E.g., for collecting residential information of a person,   first the information of the country and then  the city  he lives in will be collected.}, which intuitively emphasizes relatively global knowledge and alleviates possible label noise, thus reducing the risk of overfitting and generalizing better.
\focus{Experimental evaluation}Obvious performance improvements across various real-world datasets reveal the potential of this simple framework.

\focus{Part 6: summarizing contributions of this paper:}The contribution of this paper can be summarized as follows:
\begin{itemize}
    \item We propose a residual connections-favored soft graph normalization structure (called \textit{R-SoftGraphAIN}) for preserving knowledge from both input graph topology and features and  retaining the diversities of node embeddings in deep layers to consequently alleviate over-smoothness.
    \item We design a novel label-smoothing-based curriculum learning framework (called \textit{SmoothCurriculum}) to ease the difficulty of optimization of deep GNNs and better their generalization  via implicit coarse-to-fine node discrimination.
    \item   Extensive experiments were carried out to demonstrate the effectiveness and potential of our method compared to twelve existing baselines including state-of-the-arts.
\end{itemize}

\section{Preliminaries and Related work}
\label{Sec:Related_Work}

\paragraph{Notation}
\label{Notation}

Let $G=(V,E)$ be an undirected graph with node set $V$ and edge set $E$, where $n=|V|, m=|E|$ represent the numbers of its nodes and edges respectively.
We denote by $A \in \{0,1\}^{n\times n}$ and $X \in \mathbb{R}^{n \times d}$ its adjacency and feature matrix where node $i$ has feature $x_i = X_{i,:} \in \mathbb{R}^d$ and a ground-truth label $y_i=Y_i\in \mathbb{N}$.
Define $I_n \in \mathbb{R}^{n \times n}$ as an identity matrix, $\mathbf{0_n}, \mathbf{1_n}\in \mathbb{R}^{n \times 1}$ as all-zero/one vectors.

\paragraph{GCN and SGC}
\label{Sec:Notation_And_Preliminaries_GCN_and_SGC}

GCN can be formulated as follows:
\begin{equation}
\begin{split}
H^{(0)}  = X, \quad
H^{(l+1)}  = \sigma \left(\hat{A} H^{(l)} W^{(l)}\right) \in \mathbb{R}^{n \times d}.
\end{split}
\end{equation}
SGC simplifies GCN by dropping its non-linear activation functions and its forward pass can be described as follows:
\begin{equation}
    \begin{split}
        H^{(l)} = \hat{A}^l X W \in \mathbb{R}^{n \times d}, \quad \forall l \in [0,L),
    \end{split}
\end{equation}
where $L$ is the number of layers and $\sigma(\cdot)$ denotes a non-linear activation function (\eg, ReLU, or Softmax for the last layer).
$\hat{A} = \tilde{D}^{-\frac{1}{2}} \tilde{A} \tilde{D}^{-\frac{1}{2}}$ is the symmetrically normalized matrix of $\tilde{A} = A + I_n$ and diagonal matrix $\tilde{D}_{i,i} = \sum_{j=1}^n \tilde{A}_{i,j}$.
Sometimes, the probability transition matrix $\hat{A}_{rw} = \tilde{D}^{-1} \tilde{A}$ is employed for aggregating.
$H^{(l)} \in \mathbb{R}^{n \times d}$ and $W,W^{(l)} \in \mathbb{R}^{d \times d}$ represent the embeddings and the trainable parameters.

\paragraph{Deep GNNs and Over-smoothness}
\label{Sec:Related_work_GNNs_and_Over_smoothness}
\focus{Shortcomings of shallow GNNs:}Majority of GNN variants are shallow networks (\eg, no more than three layers), thus restricting their expressive power and limiting distant message passing.
\focus{introduction to prior works}Some prior works make efforts to deepen GNNs via modifications or tricks which can be categorized into three classes: residual connections, graph normalization (\eg, Pairnorm \cite{p53_pairnorm}, Nodenorm \cite{p54_nodenorm}, Meannorm \cite{p55_meannorm}), and random dropping (\eg, DropEdge \cite{p56_dropedge} and DropNode \cite{p57_dropnode}).
\focus{Specifically introducing residual connections:}Residual connections contain common residual connection (from last layer) \cite{p46_deep_deepgcns}, initial connection (from the first layer) \cite{p39_GCNII}, dense connection (from every layer) \cite{p51_towards_deeper_gnn,p49_break_the_ceiling}, and jump connection (from every layer to the last layer only) \cite{p50_JKNet}.
See supplementary materials or \cite{p52_deeper_gnn_tricks_overview} for some more related work or a more detailed survey.
\focus{Links to our work:}Our method appropriately combines \textit{improved} residual connections and a \textit{novel} soft graph normalization enabling effective feature and structural knowledge extraction and preservation even with a sufficiently large depth.

\paragraph{Curriculum Learning}
\label{Sec:Related_work_Curriculum_Learning}
\focus{short introduction and applications in curriculum learning:}CL has become a popular kind of training strategies for networks in many applications including texts
\cite{p81_curriculum_for_text_2020}, images \cite{p80_curriculum_for_image_classification_CV_2020}, speeches \cite{p82_curriculum_for_speech_2020}, reinforcement learning \cite{p79_curriculum_Reinforcement_Learning_survey}, etc.
\focus{basic ideas and improvements:}The basic idea is to give examples from easy to hard, and has been developed a lot \cite{p74_curriculum_learning_survey_1}, but few are designed specifically for graph-rated tasks \cite{p84_curriculum_for_graph_1_Graph_Classification,p85_curriculum_for_graph_2_Graph_Contrastive_Learning}.
\focus{Why curriculum learning?}But note that besides over-smoothness, another non-negligible cause hindering deep GNNs is just the difficult optimization.
Thus in this work, we hope to ease it via a novel curriculum learning framework based on iterative label smoothing on an auxiliary graph.
\focus{Notations for Curriculum Learning:}Here we define a \textit{curriculum} as a \textit{task sequence} $\mathcal{T}_1, \mathcal{T}_2, \cdots, \mathcal{T}_{n_T}$ with gradually increasing difficulties where $\mathcal{T}_i = \left( D^{(i)}, f_{\theta}^{(i)}, \ell^{(i)}(\cdot), p^{(i)} \right)$ denotes a task meaning that a model $f_{\theta}^{(i)}$ learns from the data $D^{(i)}$ with a loss function $\ell^{(i)}(\cdot)$ and \textit{pacing} strategy $p^{(i)}$ (\eg, the time it spends).
$f_{\theta}^{(i)}$ is some modified or restricted version of $f_{\theta}$.


\section{SmoothCurriculum-improved R-SoftGraphAIN}


In this section, we propose a novel model for alleviating the over-smoothness of graph neural networks by incorporating  two main ingredients (\ie, \textit{R-SoftGraphAIN} for GNN structures, and the adaptive curriculum design).
Although treated as  a whole for reporting  their performance in Sec. \ref{Sec:Experiments}, 
we individually  introduce each ingredient in the following  for better clarity and briefness.

\subsection{R-SoftGraphAIN}
\label{Sec:ResGraphAIN}
We will describe   our  normalization method and show how it can be improved via residual and initial connections in the following paragraphs.

\paragraph{Spectral Analysis on Over-smoothness}
\label{Sec:R-SoftGraphAIN_Spectral_analysis}

\focus{Part 1: Analysis of over-smoothness:}\focus{1): External manifestation and direct cause of over-smoothness:}Over-closeness (\ie, embeddings are too close) is the external manifestation of over-smoothness leading to indiscrimination via a classifier.
\focus{2): Essential reason of over-smoothness:}But it's just a direct cause instead of an essential reason analyzed by priors in the spectral domain on SGC as follows:
\begin{equation}
    \begin{split}
        &h^{(L)} = \hat{A}^L x =  U \Lambda^L U^T x =  \sum_{i=1}^n \lambda_i^L u_i u_i^T  x, \\
        &\lim_{L \to \infty} h^{(L)} = u_1 u_1^T x = \tilde{D}^{\frac{1}{2}} \mathbf{1}_n \mathbf{1}_n^T \tilde{D}^{\frac{1}{2}} x \propto \tilde{D}^{\frac{1}{2}} \mathbf{1}_n,
    \end{split}
\end{equation}
%
%
%
where $x$ is a signal, $\hat{A}=U \Lambda U^T$ contains decreasing eigenvalues $\{\lambda_i\}$ with respective eigenvectors $\{u_i\}$ and $\lambda_1 = 1 > \lambda_2$.
The conclusion on GCN is similar with a very different non-trivial analysis \cite{p40_GNN_exponentially_lose_expressive}.
The analysis tells us deep GNNs tend to: 1) hardly keep important structural knowledge, 2) gradually forget the semantics contained in features, thus essentially leading to over-smoothness.

\focus{Part 2: Pairnorm and its strengths/concerns:}\focus{1): the motivation of Pairnorm:}Pairnorm \cite{p53_pairnorm} attempts to numerically solve over-closeness \focus{2): the practice of pairnorm:}via direct manipulation on mean distance formulated as: $H^{{(l+1)}} = C \sqrt{n} \cdot H'/ \left\| H' \right\|_F, \  H'=(I_n-1/n\cdot \mathbf{1}_n \mathbf{1}_n^T)\hat{A}H^{(l)}$ with a constant $C>0$. \focus{3): the concerns of pairnorm:}We conjecture its performance is limited even getting rid of indiscrimination due to normalizing embeddings: 1) only element-wisely and isotropically without modeling the complex relationships between nodes and between signals; 2) only numerically without interpretable knowledge preservation.
\focus{Part 3: from Pairnorm to our works:} These shortcomings motivate our \textit{GraphAIN} considering normalization an-isotropically and in a distribution/knowledge-aware manner.
\focus{Part 4: theorems, visualizations and relations:}

\paragraph{Soft Graph An-Isotropic Normalization}
\label{Sec:R-SoftGraphAIN_SoftGraphAIN}
\focus{Part 1: the motivation of GraphAIN:}As mentioned above, we propose \textit{GraphAIN} to deal with the comprehensive distribution and keep diverse meaningful knowledge, which can be described as follows:
\focus{Part 2: the practice of GraphAIN:}
\begin{equation}
    \label{equ:GraphAIN}
    \begin{split}
        &H_t = B_{t-1} \left( B_{t-1}^T B_{t-1}  \right)^{-\frac{1}{2}}  \in \mathbb{R}^{n \times d} , \quad \forall \ t \ge 1, \\
        &B_{t-1} = T \cdot \hat{A} H_{t-1} \in \mathbb{R}^{n \times d} , \quad H_0 = X,
    \end{split}
\end{equation}
where $H_t = H^{(t)}$ and $X$ denote the embedding matrix of the $t$-th layer and the original features.
And $T = I_n-\frac{1}{n} \mathbf{1}_n \mathbf{1}_n^T \in \mathbb{R}^{n \times n}$ represents a centering operator and $P^{1/2}$ refers to the square root of a positive matrix $P$.
\focus{Part 3: the motivation and theoretical justification:}The following statement theoretically justifies our idea that we are normalizing the covariance matrix of $H$ instead of directly normalizing the embedding for an individual node or signal independently:

\focus{basic characteristics:}\begin{theorem}\label{theorem:hard_GraphAIN_basic_characteristics}
$\forall \ t \ge 1$, GraphAIN satisfies that:

1) $\mathbf{1}_n^T B_t = \mathbf{1}_n^T H_t = \mathbf{0}_n$;
\quad 2) $H_t^T H_t = I_d$;

3) $T H_t = H_t$;
\quad \quad \quad \quad \quad 4) $B_{t} = \Bar{A} H_{t}$;

where $\Bar{A} = T\hat{A}T$ denotes a doubly centered version of $\hat{A}$.
\end{theorem}

\focus{deep understandings:}\noindent Next theorem theoretically analyzes \textit{GraphAIN} from an optimization perspective and gives a deep understanding of combination of normalization and aggregations in GNNs:
\begin{theorem}
\label{theorem: hard_GraphAIN_optimization}
GraphAIN can be viewed as an iterative process in order to solve the following restricted optimization problem via \textbf{Projected Gradient Ascent} method:
\begin{equation}
    \label{equ:objective_function_hard_GraphAIN}
    \begin{split}
        \max_H \quad f(H) = \frac{1}{2} \cdot \mathrm{tr}\left( H^T \Bar{A} H \right), \quad s.t. \ \  H^TH=I_d,
    \end{split}
\end{equation}
where $H$ is initialized to input graph signals $X \in \mathbb{R}^{n \times d}$ and set the step size $
\eta = 1$.
\end{theorem}
\focus{additional interpretation:}\noindent All proofs can be found in supplementary materials.
This theorem reveals what \textit{GraphAIN} can learn.
In fact, from another point of view, the optimal solution to this problem can be obtained via \textbf{\textit{Lagrange multiplier}} method as follows:
\begin{equation}
\label{equ:Lagrange_Multiplier_Method_for_objective_function_hard_GraphAIN}
\begin{split}
    \mathcal{L}(H,\Lambda_L) = \mathrm{tr}\left( H^T \Bar{A} H \right) - \mathrm{tr}\left( \Lambda_L \left( H^TH-I_d \right) \right),
\end{split}
\end{equation}
\noindent where $\mathcal{L}$ and $\Lambda_L$ are the Lagrange function and multipliers.
Let $\partial \mathcal{L} / \partial H=0$, we get $\Bar{A}H=\Lambda_L H$, which means that $H$ tends to the eigenvectors corresponding to the top-$d$ eigenvalues of $\Bar{A}$ with sufficient steps.
Thus similar to Spectral Clustering, it can capture essential structural knowledge due to the similarity between $\Bar{A}$ and $\hat{A}$.
\focus{Part 4: more interpretations or intuitions in spatial domain}In addition to spectral interpretations, we give some intuitions in spatial domain: 1) it can easily solve over-closeness, due to the facts that $\left\| H \right\|_F^2=d$ and $\sum_i\sum_j \left\| H_{i,:} - H_{j,:} \right\|_2^2=2n\cdot \sum_i \left\|H_{i,:}\right\|_2^2 - 2 \cdot \left\| \sum_i H_{i,:} \right\|_2^2 = 2n \cdot d$ meaning the average pairwise distance is kept completely constant similar to Pairnorm; 2) the spatial variance in
any direction is normalized to $1$ for maximally preserving the diversity of knowledge in a circular distribution.

\focus{Part 5: its drawbacks:}However, it still suffers from performance degradation due to the following four potential drawbacks: 1) too absolute; 2) numerical instability; 3) high time complexity; 4) risk of forgetting original features during iterations.
The first three issues can be relieved via a soft version (\ie, \textit{SoftGraphAIN}):
\focus{Part 6: the practice of SoftGraphAIN:}\begin{equation}
    \label{equ:SoftGraphAIN}
    \begin{split}
        H_t \approx B_{t-1} \left[  a \cdot U_{d_0} \Lambda_{d_0}^{-\frac{1}{2} \cdot b} U_{d_0}^T + \left( 1-a \right) \cdot I_d \right],
    \end{split}
\end{equation}
where $B_{t-1}^T B_{t-1} \approx U_{d_0} \Lambda_{d_0} U_{d_0}^T$ is the $d_0$-truncated SVD calculating only the top-$d_0 \le d$ eigenvectors and eigenvalues contained in $U_{d_0},\Lambda_{d_0} \in \mathbb{R}^{d \times d}$, and $a,b \in [0,1]$ are another hyper-parameters controlling the extent of normalizing.
\focus{Part 7: interpretations of the soft operation and how they work}Formally, they transform the singular values $S \approx S_{d_0}=\Lambda_{d_0}^{1/2} \in \mathbb{R}^{d \times d}$ into $\left( 1 - a \right) \cdot S_{d_0} + a \cdot S_{d_0}^{1-b}$ thus flexibly reducing the absoluteness, possible noise in useless channels, risk of numerical zero-divisions, and empirical time-inefficiency.

\paragraph{Additional Residual Combination}
\label{sec:Additional_Residual_Combination}

\focus{Part 1: the practice of the R-SoftGraphAIN:}\textit{R-SoftGraphAIN} can effectively alleviate the last drawback mentioned above via some residual and initial connections formulated as follows:
\begin{equation}
    \label{equ:residual_connections}
    \begin{split}
        B_{t} = \alpha \cdot T \hat{A} H_{t} + \beta \cdot H_{t} + \gamma \cdot X \in \mathbb{R}^{n \times d}, \ \forall \ t \ge 1,
    \end{split}
\end{equation}
where the non-negative hyper-parameters meet $\alpha + \beta + \gamma=1$.
\focus{Part 2: the direct motivation of these connections:}Intuitively, the commonly used residual connections can alleviate gradient-vanishing and the initial ones are expected to constantly supplement some feature information during aggregations in case of oblivion.
\focus{Part 3: theoretical analysis from an optimization perspective:}Moreover, the motivation can be theoretically justified via the following similar theorem:
\begin{theorem}
\label{theorem: R-soft_GraphAIN_optimization}
Residual-favored GraphAIN can be viewed as an iterative process in order to solve the following restricted optimization problem via \textbf{Projected Gradient Ascent} method:
\begin{equation}
    \label{equ:objective_function_R_Soft_GraphAIN}
    \begin{split}
        \max_H \ f(H) = &\frac{1}{2} \cdot \mathrm{tr}\left( H^T \Bar{A} H \right) - \frac{1}{2} \cdot \frac{\gamma}{\alpha} \cdot \left\| H - X \right\|_F^2 \\ &s.t. \quad  H^TH=I_d,
    \end{split}
\end{equation}
where $H$ is initialized to input graph signals $X \in \mathbb{R}^{n \times d}$ and set the step size $
\eta = \alpha \in [0,1]$.
\end{theorem}
\focus{Part 4: intuitive interpretation of this theorem:}\noindent This theoretically reveals that \textit{R-SoftGraphAIN} never forgets the original features as GNNs go deep, simultaneously relieving two essential reasons in Sec. \ref{Sec:R-SoftGraphAIN_Spectral_analysis} and thus alleviating over-smoothness.
Furthermore, from this we can get some intuitions on the roles of $\alpha, \beta, \gamma$: $\alpha$ allows a sufficiently small step size ensuring better convergence, $\gamma$ estimates the contribution of features.
$\beta$ can give some freedom to $\alpha$ and $\gamma$.
\focus{Part 5: further improvement: fuzzy residual}In order to further improve its performance, we generalize these connections as \textit{fuzzy} connections.
We use Eq.~\ref{equ:SoftGraphAIN} and Eq.~\ref{equ:residual_connections} to substitute Eq.~\ref{equ:residual_connections}, and replace $X$ in Eq.~\ref{equ:residual_connections} by $H_1$ due to possible misalignment in dimensions.
A GCN-based implementation of the whole structure is summarized in Algorithm~1 in supplementary materials with a line-by-line description therein.

\paragraph{Relations to Others}
\label{Sec:Relation_to_others}
\focus{Part 1: summarization:}In this paragraph, we detailedly compare ours with other related methods.
\focus{Part 2: perspective of knowledge preservation:}\focus{1):SGC}SGC suffers from over-smoothness due to both structural and feature knowledge loss.
\focus{2):Pairnorm}Pairnorm numerically solves over-closeness without interpretable knowledge preservation.
\focus{3):Meannorm; 4):Spectral clustering}Meannorm and Spectral Clustering (SC) can keep the $2$-th and lowest $d$ frequent components in structures respectively while keeping little feature information.
\focus{5):GCNII}GCNII proves to be a universal approximator of any function on features, but it ignores structural semantics.
\focus{6):Ours:}Compared to them, ours can keep both features and structural knowledge inheriting both advantages.
\focus{Part 3: perspective of relationship modelling:}\focus{Pairnorm, meannorm, nodenorm:}From another perspective, Pairnorm, Meannorm, and Nodenorm \cite{p54_nodenorm} are only element-wise, signal-independent, and node-independent, respectively.
\focus{ours:}However, our method considers the comprehensive distributions and effectively models the relationships between nodes and between signals, thus uttermost preserving the diversity of the embeddings.
\focus{Visualization:}
Fig.~3 in supplementary materials shows our superiority, where ours is similar to and even outperform SC while others suffer from over-smoothness to different extents.

\subsection{SmoothCurriculum}
\label{Sec:SmoothCurriculum}

\focus{introduction to our curriculum learning framework, and the organization of the following paragraphs:}In this section, we propose a \textit{simple} yet \textit{effective} curriculum learning framework based on \textit{label-smoothing} on an auxiliary graph to ease the  hardness of optimizing the proposed \textit{R-SoftGraphAIN}.
Inspired by Curriculum Learning, the key idea of our framework is first to learn the low-frequent knowledge contained in labels before the high-frequent ones, and then to employ  an easy-to-hard learning process that favors a better generalization.
The framework will be described in detail regarding several of its important modules (\eg, label estimation and smoothing, graph construction, and curriculum designs), intuitions, and interpretations. 


\paragraph{Label Estimation and Auxiliary Graph}
\label{Sec:Label_Estimation_and_Auxiliary_Graph}

Deep GNNs are powerful yet risk overfitting due to limited labeled data, especially in the semi-supervised setting.
Thus we hope to enlarge the training set via label estimation.
\focus{Part 1: label propagation and its shortcomings:}One of the most commonly used classic techniques is \textit{Label Propagation}, whose iterative process is: $f \leftarrow P \cdot f, f_L \leftarrow Y_L$ initializing $f_L = Y_L, f_U = \mathbf{0}$.
Furthermore, its limit can be formulated as:
$Y_U^{(e)} = \lim f_U = \left( I - P_{UU} \right)^{-1} P_{UL}Y_L$, where $U,L$ represent unlabeled and labeled node sets, $P=D^{-1}A$, $P_{UL}$ is a sub-matrix of $P$ respect to lines $U$ and columns $L$, and $Y_L,Y_U^{(e)}$ are the known and estimated labels.
But it suffers from two shortcomings: 1) impossible propagation due to possible disconnectivity; 2) impractical matrix inversion with a large $|U|$.
\focus{Part 2: implicit label propagation}Thus, we estimate $Y_U$ \textit{implicitly} via a teacher model $f_t(\cdot)$ pre-trained on the labeled data $D_o = (X_L,Y_L)$, which can distill shared knowledge from distant nodes or disconnected components to favor more accurate estimation.

\focus{Part 3: auxiliary graph building:}Moreover, we expect to capture and encode the similarities of nodes' ground-truth labels into the structure or communities of an auxiliary graph $G_{aux}$.
It can be built as follows:
1) $G_{aux}=G$, input graph for graphs with noisy or missing features;
2) $G_{aux} = G_f$, a KNN-graph built according to node features for graphs with heterophily;
3) $G_{aux} = G_e$, a KNN-graph built according to node embeddings output by the teacher $f_t(\cdot)$ for others.
More specifically, a KNN-graph $G(\mathbf{h})$ of a set of vectors $h_1, \cdots, h_n$ is built as follows: link every $h_i$ to its top-$k$ nearest vectors via a KNN algorithm with \textit{Gaussian} distance, drop the edge directions, and then calculate the weights via similarity scores $W^{(aux)}_{i,j} = \mathrm{ReLU}\left( h_i^Th_j \right)^{\gamma'}$ with a distribution-controlling hyper-parameter $\gamma'>0$ for edge weights $W^{(aux)} \in \mathbb{R}^{n \times n}$.

\paragraph{Label Smoothing and Curriculum Design}
\label{Sec:label_smoothing_and_curriculum_design}
\focus{the practice of iterative label smoothing:}To get multi-scale label signals, we iteratively smooth labels on $G_{aux}$ with initial signal $Y^{[0]}=Y$ from $f_t\left(\cdot \right)$ as: $Y^{[i+1]}=P_{aux} \cdot Y^{[i]}, \ \forall \ i \in [0,n_T)$, where $P_{aux}=D_{aux}^{-1}W^{(aux)}$ and $D_{aux}$ are the probability and degree matrix on $G_{aux}$, respectively.
Note that this simplified Label Propagation without fixing $f_L$ will definitely encounter over-smoothness similar to SGC, but it's just what we desire (see next paragraph).
\focus{Curriculum design}After that, a curriculum $\mathcal{C}$ can be defined as $\mathcal{T}_0, \mathcal{T}_1, \cdots, \mathcal{T}_{n_T}$, where task $\mathcal{T}_i = \left( D^{(i)}, f_{\theta}, \ell(\cdot), p^{(i)} \right)$, \ie, the graph encoder $f_{\theta}$ and the loss $\ell(\cdot)$ are shared in all tasks but the training data $D^{(i)}=\left( X, Y^{(i)} \right)$ and pacing strategies $p^{(i)}$ vary.
Here, we prepare $Y^{(i)} = Y^{[n_T-i]}, \ \forall \ i \in [0,n_T]$.
In other words, the encoder $f_\theta$ will be encouraged to learn tasks from $\mathcal{T}_0$ to $\mathcal{T}_{n_T}$ where easy tasks containing easy data $D^{(i)}$ are solved before the relatively harder ones with \textit{even paces}.
And finally, it will be fine-tuned to solve the original task with $D_o=(X_L,Y_L)$.

\paragraph{Analysis and Interpretations}
\label{Sec:Analysis_and_Interpretations}
\focus{Part 0: intuitions:}Next we give some analysis and intuitions on what \textit{SmoothCurriculum} exactly does and how it guides the training in the following aspects:
\focus{1): from optimization:}1) optimizing from convex to non-convex: as claimed in \cite{p76_original_curriculum_learning_2009_examples_from_easy_to_hard,p74_curriculum_learning_survey_1}, curriculum learning with priority for easy tasks can be equivalently understood as landscape smoothing for empirical loss contributing to a more convex optimization problem, which guides models to find a local minima with less vibration and better generalizability.
\focus{2): from spectral domain:}2) learning spectral knowledge from low- to high-frequency: $Y^{[i]} = \left( P_{aux} \right)^{i} Y^{[0]} =  U \Lambda^i U^T Y^{[0]}$ and $\Lambda^i = \mathrm{diag}\left( \lambda_1^i, \lambda_2^i, \cdots, \lambda_n^i \right)$ with always decreasing eigenvalues.
Let $i$ vary \textit{decreasingly}, and consider important values $\{i_j, j \in [1,n]\}$ where at time $i_j$, $Y^{[i_j]} \approx  \sum_{t=1}^j \lambda_t^{i_j} u_t u_t^T Y^{[0]}$ with the dominating top-$j$ eigenvalues $\{ \lambda^i_k, k\in [1,j] \}$.
Then from $Y^{[i_j]}$ to $Y^{[i_{j+1}]}$, old knowledge will be reviewed due to $\lambda_{t}^{i_j} \le \lambda_{t}^{i_{j+1}}$ and some relative high-frequent component $u_{j+1}u_{j+1}^TY^{[0]}$ as new information will be injected to label signals.
Additionally, if we independently consider a single signal $y^{[0]}$, then $u_{j+1}u_{j+1}^Ty^{[0]}=\left(u_{j+1}^Ty^{[0]}\right)u_{j+1} \propto u_{j+1}$ introducing a new channel for spectral embedding encoding some new details leading to more complex clustering structures.
Thus models can learn to discriminate nodes from coarse to fine.
\focus{3): from spatial domain: }3) learning spatial knowledge from global to local: Intuitively, the shared commonsense is illustrated first due to $\lim_{i \to \infty} \left( P_{aux} \right)^{i} Y^{[0]} = \mathbf{1}_n \mathbf{1}_n^T Y^{[0]}$, which encodes the global label frequency.
Then the pieces of information in big communities, small communities, and local environments are presented in order because over-smoothness happens quickly in high-density regions but slowly otherwise.
In other words, it also spatially favors coarse-to-fine node discrimination via perceiving the multi-scale community structure or density varieties of $G_{aux}$.
\focus{Part 1: why combine with R-SoftGraphAIN}
\begin{table*}[t]
\small
\centering
\begin{tabular}{c|cc|cc|cc|cc|cc|cc|c}
\toprule
Method & \multicolumn{2}{c|}{Cora} & \multicolumn{2}{c|}{Citeseer} & \multicolumn{2}{c|}{Pubmed}&\multicolumn{2}{c|}{CS}&\multicolumn{2}{c|}{Physics}&\multicolumn{2}{c|}{Computers}& \multirow{2}{*}{Avg.Rank}\\ \cmidrule{1-13} 
   \#Layes& 32 & 64  & 32 & 64 & 32 & 64 & 32 & 64& 32 & 64& 32 & 64\\ \midrule
GCN & 31.90 & 27.56 & 36.66 & 25.40  & 44.22 & 32.65  & 41.29 & 34.23 & 79.87 & 75.34 & 58.30 & 37.58 & 12.75 \\

SGC & 63.84 & 55.39 & 67.50 & 63.08 & 70.70 & 65.33  & 70.52 & 72.51 & 91.46 & 90.77 & 37.44 & 37.50 & 10.33  \\

ChebNet & 31.90 &20.63 & 33.43 & 24.90 & 48.67 & 45.37 & 29.28 & 23.24 & 70.35 & 50.74 & 58.58 & 50.12 & 12.67 \\ 
GAT & 72.28  & 31.92 & 59.08 & 22.90 & 78.72 & 41.88  & 85.85 & 12.83 & 91.87 & 17.75 & 76.05 & 37.18 & 10.92 \\ 

GIN & 60.92 & 31.90 & 47.32 & 23.10 & 72.94 & 40.23  & 52.90 & 20.79 & 83.02 & 24.98 & 39.18 & 37.50 & 12.42 \\ 

Pairnorm       & 65.00 & 66.24 & 44.20 & 41.48 & 72.12 & 71.72  & 72.71 & 68.62 & 88.51 & 89.11 & 74.96 & 74.35 & 9.67 \\

GCNII          &  \textbf{85.29}    & \textbf{85.34}    & 73.24 & 73.00 & 79.81 & 79.88  & 71.67 & 72.11 & 93.15 & 92.79 & 37.56 & 37.50 & 6.67 \\

JKNet          & 73.23 & 72.54 & 50.68 & 52.22 & 63.77 & 69.10  & 81.82 & 82.84 & 90.92 & 89.88 & 67.99 & 67.78 & 9.17 \\

GPRGNN         & 83.13 & 82.48 & 71.01 & 70.96 & 78.46 & 78.92  & 89.56 & 89.33 & 93.49 & 93.26 & 41.94 & 78.30 & 6.83 \\

DAGNN          & 83.39 & 82.16 & 72.59 & 71.00 &  \underline{80.58} &  80.44  & 89.60 & 89.47 & 93.31 & 93.52 & 79.73 & 79.23 & 4.92 \\

APPNP          & 83.68 & 83.66 & 72.13 & 72.02 & 80.24 & 80.08  & 91.61 & 91.58 & 93.75 & 91.61 & 43.02 & 41.42 &5.67  \\

BernNet        & 81.38 & 17.72 & 70.82 & 39.10 & 70.24 & 32.86 & 91.54 & 9.20 & 92.27 & 19.38 & 81.06 & 12.81 & 10.67\\ 
\midrule

Ours(GCN) &  \underline{85.12} &  \underline{84.87} & \underline{74.42} & \textbf{74.50} &  \textbf{81.28}    &  \textbf{81.58}  & \textbf{92.11} & \textbf{92.05} & \underline{94.22} & \underline{94.20} & \textbf{85.21} & \textbf{85.13} & 1.42 \\
   
Ours(GAT) & 84.60 &84.68 & 74.26 & 74.04 & 80.46 & 80.20 & 91.30 & 91.26 & 94.02 & 94.02 & 84.82 & \underline{85.04} & 3.33 \\ 
  
Ours(GIN) &84.18 & 83.80 & \textbf{74.88} & \underline{74.16} & 80.32 & \underline{80.94} & \underline{91.79} & \underline{91.71} & \textbf{94.56} & \textbf{94.53} & \underline{85.16} & \textbf{85.13} & 2.17 \\   
 \bottomrule
\end{tabular}
\caption{Results of node classification tasks on Cora, Citeseer, Pubmed, CS, Physics, and Computers}
\label{tab:comparison_1}
\end{table*}

\begin{table*}[t!]
\small
\centering
\begin{tabular}{c|cc|cc|cc|cc|cc|cc|c}
\toprule
Method & \multicolumn{2}{c|}{Photo}& \multicolumn{2}{c|}{Texas} & \multicolumn{2}{c|}{Wisconsin} & \multicolumn{2}{c|}{Cornell} & \multicolumn{2}{c|}{Actor} & \multicolumn{2}{c|}{OGBN-ArXiv} & \multirow{2}{*}{Avg.Rank}\\ \cmidrule{1-13} 
 \#Layes& 32 & 64  & 32 & 64  & 32 & 64 & 32 & 64 & 32 & 64 & 32 & 64\\ \midrule
GCN & 58.47 & 50.21 & 62.16 & 62.16  & 57.84 & 57.84  & 56.76 & 56.76  & 25.16 & 25.16 & 46.38 &42.95 & 10.00\\
SGC  & 26.08 & 24.57 & 56.41 & 56.96  & 51.29 & 52.16  & 58.57 & 55.41  & 26.17 & 25.88 & 34.22 & 23.14 & 11.67\\
ChebNet & 65.28 & 64,83 & 64.86 & 64.86 & 52.94 & 52.94  & 55.86 & 52.25 & 25.46 & 25.46 & 41.00 & 35.16 & 10.33 \\ 
GAT & 83.73 & 25.36 & 65.41 & 64.86 & 53.73 & 53.33  & 54.05 & 55.14 & 25.54 & 25.62 & 59.78 & 36.63 & 9.33 \\ 
GIN & 65.98 & 25.27 & 62.17 & 60.00 & 47.06 & 50.98  & 54.59 & 55.14 & 24.36 & 23.45 & 65.17 & 60.56 & 11.33 \\ 
PairNorm  & 82.66 & 79.55 & 41.08 & 40.68  & 52.84 & 52.94 & 36.89 & 40.68  & 24.33 & 23.23 & 63.32 & 43.57 & 11.91 \\
GCNII & 62.95 & 65.12 & 69.19 & 65.41  & 70.31 & 59.02  & 74.16 & 56.92  & 34.28 & 34.64 & 72.60 & 70.07 & 5.33\\
JKNet  & 78.42 & 79.73 & 61.08 & 66.49  & 52.76 & 56.08  & 57.30 & 51.49  & 28.80 & 28.26 & 66.31 & 65.80 & 8.25 \\
GPRGNN  & 91.74 & 91.28 & 62.27 & 61.08  & 71.35 & 64.90  & 58.27 & 52.16 & 29.88 & 32.43 & 70.18 & 69.98 & 6.25\\
DAGNN  & 89.96 & 87.86 & 57.68 & 60.27  & 50.84 & 51.76 & 58.43 & 52.43 & 27.73 & 25.45 & 71.46 & 70.58 & 8.92\\
APPNP & 59.62 & 63.63 & 60.68 & 64.32 & 54.24 & 59.90 & 58.43 & 54.69 & 28.65 & 28.19 & 66.94 & 66.90 & 8.25\\

BernNet & 91.59 & 16.53 & 61.08 & 16.76 & 63.53 & 24.71 & 52.43 & 11.89 & 28.19 & 20.66 & 45.16 & 37.18 & 11.92 \\ \midrule

Ours(GCN) & \textbf{92.06} & \textbf{92.03} & \textbf{85.41} & \textbf{84.86} &  \textbf{83.14} & \textbf{84.71} &  \textbf{82.70} & \textbf{82.62} & \textbf{38.42} & \textbf{38.49} & 74.07 & 73.95 &1.33 \\ 
Ours(GAT) & 91.98    & \underline{92.00} & 78.92 & 78.38 & 73.73 & 74.90  & 77.84 & \underline{81.62} & \underline{34.97} & \underline{35.54} & \underline{74.37} & \underline{74.41} & 2.50 \\ 
Ours(GIN) & \underline{92.03}    & 91.98 & \underline{82.57} & \underline{83.12} & \underline{81.78} & \underline{81.20}  & \underline{81.08} & 81.08 & 34.51 & 34.50 & \textbf{75.02} & \textbf{74.85} & 2.25 \\ 
\bottomrule
\end{tabular}
\caption{Results of node classification tasks on Photo, heterophilous graphs (\eg, Texas), and a large-scale graph OGBN-ArXiv}
\label{tab:comparison_2}
\end{table*}

\begin{figure*}[t]
    \centering
    \footnotesize
	\begin{tabular}{c@{}c@{}c}
	Cora & Citeseer & Pubmed\\
        \includegraphics[width=0.325\textwidth]{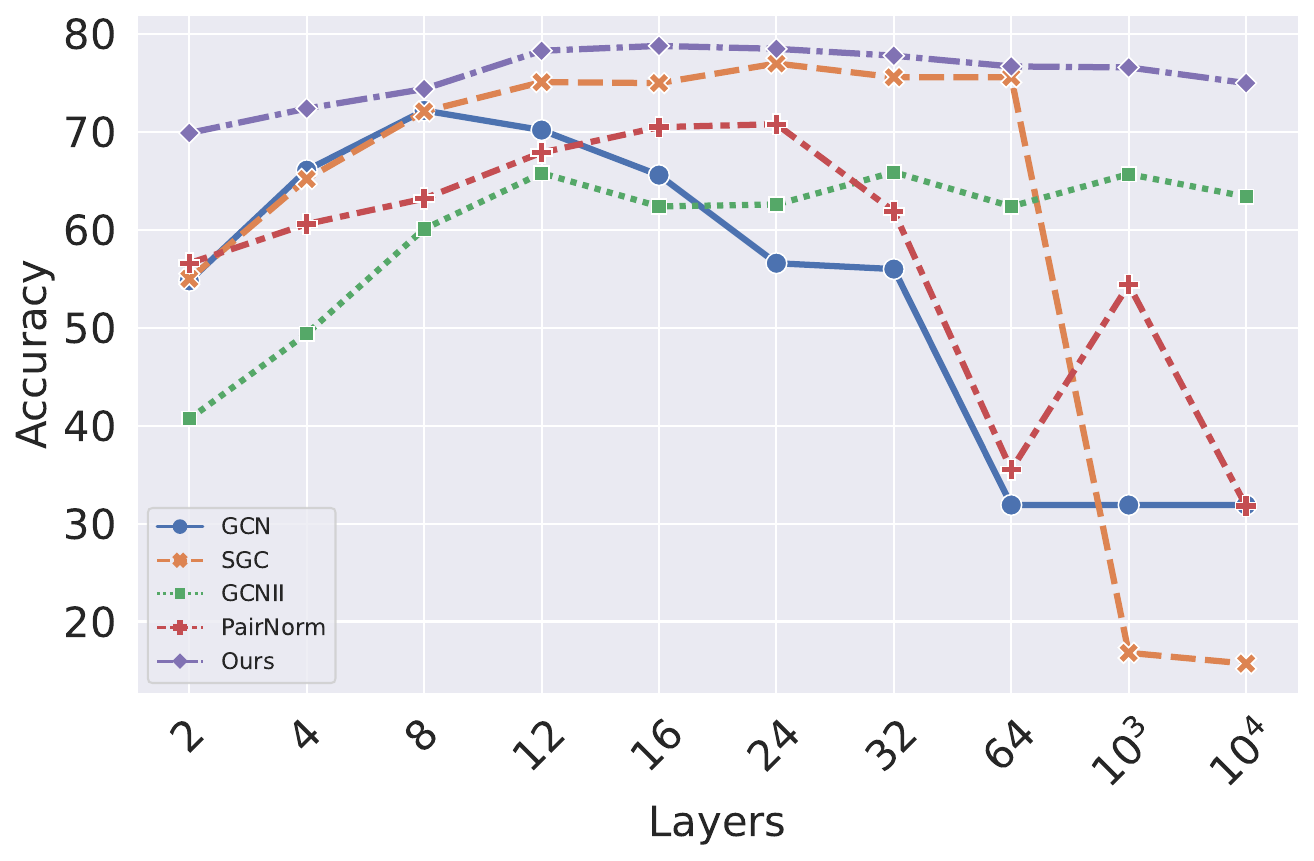}&
		\includegraphics[width=0.325\textwidth]{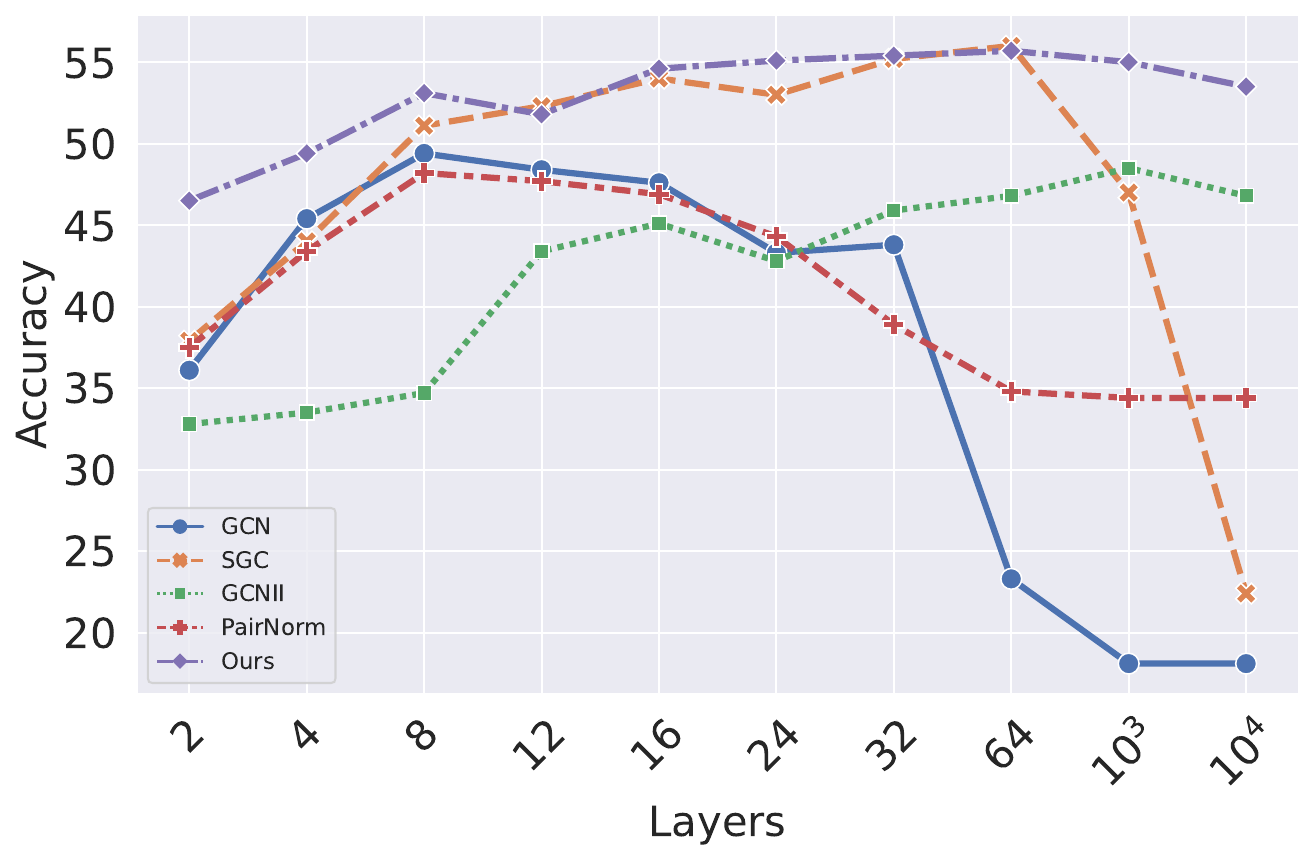}&
        \includegraphics[width=0.325\textwidth]{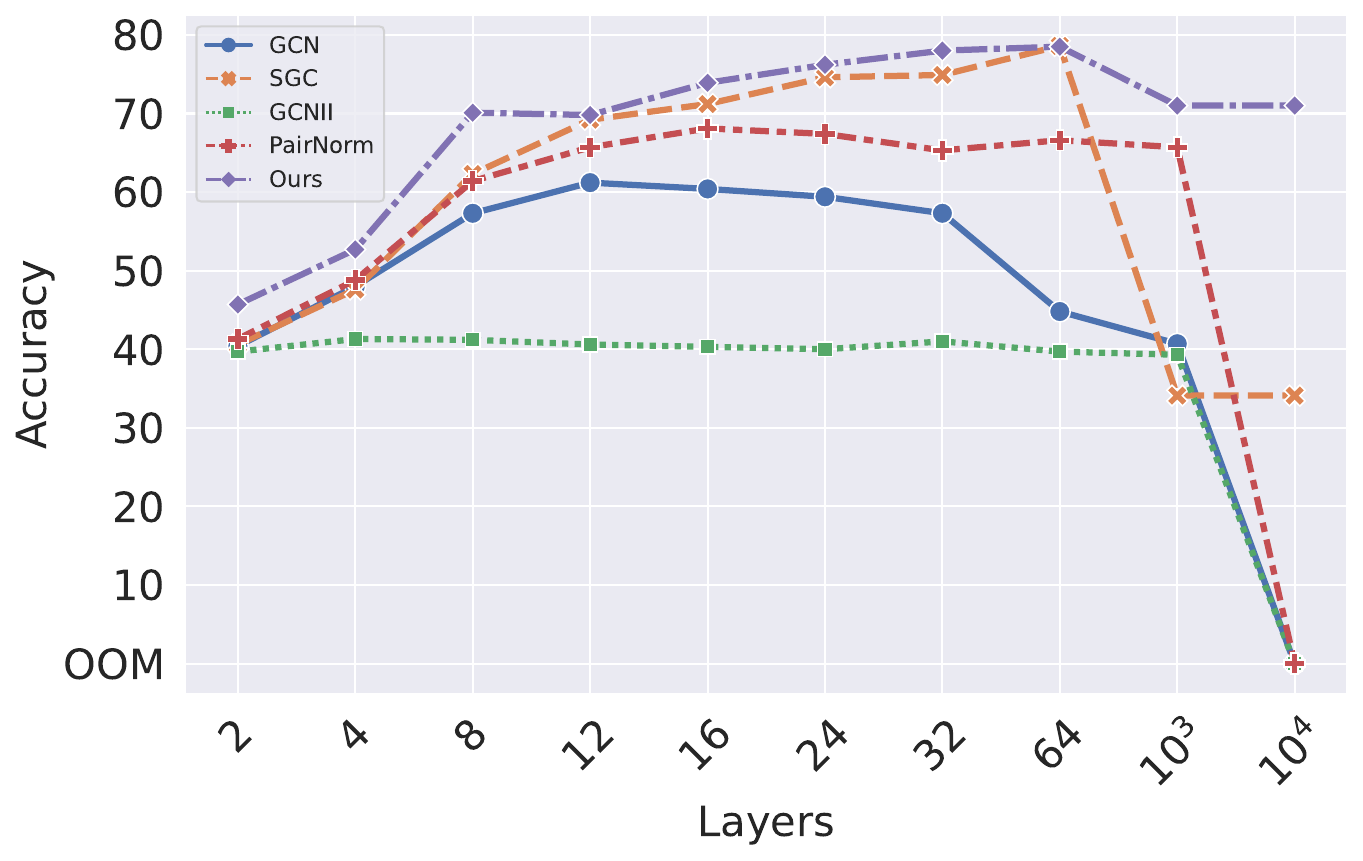}    
	\end{tabular}
	\caption{Results of different models with varying layers in node classification tasks with noisy features}
	\label{fig:noisy_features}
\end{figure*}

\section{Experimental Results}\label{Sec:Experiments}
In this section, we conduct extensive experiments to evaluate the effectiveness of our method (applied with GCN, GAT, and GIN) by comparing it with twelve baselines on twelve real-world graph benchmarks on semi-supervised node classification tasks.
Note that our method can be applied to more sophisticated  spatial propagation-based GNN backbones to further improve its performance, but we prefer basic ones to keep it simple and evaluate its potential.
Due to space limitations, some experimental details are given in supplementary materials including dataset descriptions, implementations, omitted results (\eg, with other layers, with different splits, comparisons with more baselines on heterophilous graphs, as well as standard errors), hyper-parameters (searching spaces and specific configurations), and some more visualizations.

\paragraph{Experimental Settings}

\label{Experiments_experimental_settings}

They are performed on an Ubuntu system with a single GeForce RTX $2080$Ti GPU ($12$GB Memory) and $40$ Intel(R) Xeon(R) Silver $4210$ CPUs.
And the proposed model is implemented by Pytorch \cite{p90_pytorch} and optimized with Adam Optimizer.
For a fair comparison, twelve real-world public benchmarks are chosen, including two kinds: 1) eight graphs with homophily: four widely used scientific citation networks (\ie, Core, Citeseer, Pubmed \cite{p86_citation_networks_cora_citeseer_pubmed}, and a large-scale graph OGBN-ArXiv \cite{p98_OGB}), scientific co-authorship networks Physics and CS \cite{p88_WikiCS}, as well as Amazon purchasing system Computers and Photo \cite{p87_Amazon}; 2) four graphs with heterophily: webpage datasets Texas, Wisconsin, and Cornell \cite{p91_Geom_GCN_Heterophily_graph_Texas_Wisconsin_Cornell} as well as an actor co-occurrence network Actor \cite{p92_Actor}.
Their statistics and adopted splits are summarized in Tab.~4 in supplementary materials.
We adopt the standard semi-supervised training/validation/testing splits for them following prior works \cite{p4_GCN,p39_GCNII,p52_deeper_gnn_tricks_overview}.
Furthermore, twelve baselines or state-of-the-art GNN models are applied for comparison including four vanilla classic models (GCN, SGC, GAT, and GIN), two spectral-based methods (ChebNet~\cite{p119_ChebNet} and BernNet~\cite{p118_BernNet}), a normalization-based method Pairnorm \cite{p53_pairnorm}, some residual connections-based methods including GCNII \cite{p39_GCNII}, GPRGNN \cite{p65_GPRGNN}, APPNP \cite{p66_APPNP}, JKNet \cite{p50_JKNet}, and DAGNN \cite{p51_towards_deeper_gnn}.
For a fair comparison with spectral-based baselines, we view the orders of Laplacian used in filters as the depths.

\paragraph{Node Classification with Homophily and Heterophily}
\label{Experiments_performance}

We call the proposed method applied to GCN, GAT, and GIN \textit{Ours(GCN)}, \textit{Ours(GAT)}, and \textit{Ours(GIN)}, respectively, where Ours(GCN) is the default, \ie, \textit{Ours}.
We run each experiment five times with different initializations, and report the average accuracies in Tab.~\ref{tab:comparison_1} and Tab.~\ref{tab:comparison_2} with varied numbers of layers on these benchmarks.
The standard errors and results with some other layers are given in supplementary materials.
From Tab.~\ref{tab:comparison_1} and \ref{tab:comparison_2}, 
it is shown that our model consistently achieved the best results against these state-of-the-art counterparts in almost all listed layers.
Notably, for Amazon Computers Dataset, we get a performance improvement compared to DAGNN of more than $6.8\%$ and $7.4\%$ in $32$ and $64$ layers, respectively.
As observed from Tab.~\ref{tab:comparison_2}, our method outperforms any other listed model by very large margins on four heterophilous graphs and the large-scale graph OGBN-ArXiv.
In supplementary materials, we also provide results with fully supervised random splits compared to the listed counterparts (see Tab.~16) and more baselines on these heterophilous graphs (see Tab.~17).
These results demonstrate its potential to alleviate over-smoothness in deep layers.


\paragraph{Node Classification with Noisy Features}
\label{Experiments_nosiy_features}

Sometimes features can provide enough meaningful supervision signals for node prediction, which veils the ability of a GNN for \textit{structure understanding}.
In this subsection, we evaluate our model in a challenging task called \textit{Node Classification with Noisy Features} where all node features are substituted by noise sampled from Standard Normal Distribution $\mathcal{N}(0,1)$ while only the structure of input graph remains.
This task is more difficult than that in \cite{p53_pairnorm}, since:
1) The feature substitution is conducted for all nodes instead of the nodes out of the training set only.
2) We make the features noisy instead of replacing them with zeros.
Intuitively, this task tests \textit{how deeply GNNs can understand the input structure}, \ie, whether they can capture more useful structural knowledge for alleviating the adverse effect of noise in features.
We adopt our standard model \textit{Ours(GCN)} itself as the teacher and the original graph as the auxiliary graph.
As observed from Fig. \ref{fig:noisy_features}, our model outperforms most evaluated baselines in nearly all layers by a significant margin, showing its effective extraction and preservation of structural semantics.
While SGC with no more than $64$ layers can achieve better results in some layers, it suffers from over-smoothness severely with sufficient large layers (\eg, $10^3$ or $10^4$ layers).

\paragraph{Ablation, Hyper-parameter Studies, and Visualizations}
\label{Sec:Experiments_ablation_studies_visualizations}
In order to demonstrate the effects of each individual part of our model and learning framework, we conduct extensive ablation studies following \cite{p39_GCNII} and report the results on seven graph benchmarks in Tab.~\ref{tab:ablation_studies}.
In the following, we take \textit{Ours(GCN)} as our standard model and independently drop each of the five parts: SoftGraphAIN (SG), residual connections (RC), R-SoftGraphAIN (R-SG), label smoothing (LS), and the holistic curriculum learning framework (CL), where \textit{w.o. X} means that we drop the part \textit{X}.
From Tab.~\ref{tab:ablation_studies}, we observe that every part contributes a portion to the  performance gain, among which R-SG is the most significant since it reduces the risk of features and structure forgetting simultaneously.
To facilitate a better understanding, we plot the varying effects of softly normalizing extents (the hyper-parameter $a$ in Eq.~\ref{equ:SoftGraphAIN}) in Fig.~\ref{fig:hyper-parameter_a}, from which we can see a comprehensive ascending-and-then-descending trend, showing the benefits of this soft version compared to the hard one.
In supplementary materials, we study some other hyper-parameters (\eg, $\alpha$ and $k_{KNN}$), and detailedly visualize the embeddings produced by our method and some counterparts.




\begin{table*}[t!]
\footnotesize
\centering
\begin{tabular}{c|c|ccccccc}
\toprule
Method & \#Layers & Cora & Citeseer & Pubmed & CS & Physics & Computers & Photo \\ \midrule
\multirow{2}{*}{w.o. SG}   & 32   & \multicolumn{1}{c}{83.30±0.12} & \multicolumn{1}{c}{72.98±0.50} & 80.26±0.98  & \multicolumn{1}{c}{91.79±0.07} & \multicolumn{1}{c}{94.23±0.25} & \multicolumn{1}{c}{82.34} & 90.26\\
       & 64    & 84.60±0.29 & 72.62±0.98  & 80.30±0.23 & 91.70±0.07  &  94.25±0.13  & 84.50  & 91.69\\ \midrule
\multirow{2}{*}{w.o. RC}   & 32   & \multicolumn{1}{c}{82.88±0.79} & \multicolumn{1}{c}{72.82±1.01} & 78.84±0.59 &\multicolumn{1}{c}{91.40±0.26} & \multicolumn{1}{c}{92.99±0.19} & \multicolumn{1}{c}{83.78} & 91.42\\
       & 64    & 82.40±0.33  & 71.04±0.55  & 78.60±0.27 & 89.03±0.62  & 92.59±0.63  & 84.25  & 91.14\\\midrule
       
\multirow{2}{*}{w.o. R-SG} & 32   & \multicolumn{1}{c}{40.22±5.71} & \multicolumn{1}{c}{29.32±3.47} & 46.28±4.91 &\multicolumn{1}{c}{51.61±18.23}         & \multicolumn{1}{c}{77.20±11.95}         & \multicolumn{1}{c}{61.52} & 83.77\\
       & 64    & 36.74±4.36  & 27.70±2.94  & 28.61±3.64 & 28.51±4.70  & 60.94±8.23  & 49.94  & 70.09\\\midrule
\multirow{2}{*}{w.o. LS}   & 32   & \multicolumn{1}{c}{83.96±0.65} & \multicolumn{1}{c}{73.64±0.68} & 81.04±0.21 & \multicolumn{1}{c}{92.02±0.05} & \multicolumn{1}{c}{94.10±0.11} & \multicolumn{1}{c}{84.52} & 91.50\\
       & 64    & 84.00±0.22  & 74.34±0.79  & 80.86±0.79 &  91.95±0.08 & 94.06±0.08  &  84.91  & 91.64\\\midrule

\multirow{2}{*}{w.o. CL}   & 32   & \multicolumn{1}{c}{82.40±0.65} & \multicolumn{1}{c}{73.20±0.77} & 79.34±0.39 & \multicolumn{1}{c}{89.60±0.20} & \multicolumn{1}{c}{92.46±0.61} & \multicolumn{1}{c}{84.00} & 90.58\\
       & 64    & 82.58±0.87  & 72.54±0.42  & 79.00±0.54 & 89.23±0.12  & 92.22±0.56  & 84.44  & 90.41\\\midrule
\multirow{2}{*}{Ours(GCN)} & 32   & \multicolumn{1}{c}{85.12±0.15}    & \multicolumn{1}{c}{74.42±0.26}    & 81.28±0.18    & \multicolumn{1}{c}{ 92.50±0.09 }    & \multicolumn{1}{c}{ 94.45±0.06}    & \multicolumn{1}{c}{ 85.21 }    &  92.06 \\ 
       & 64    & 84.87±0.26     & 74.50±0.23     & 81.58±0.66    &  92.41±0.07     &  94.52±0.04     &  85.13     &  92.03 \\ \bottomrule
\end{tabular}
\caption{Ablation studies on seven benchmarks including Cora, Citeseer, Pubmed, CS, Physics, Computers, and Photo}
\label{tab:ablation_studies}
\end{table*}

\begin{figure*}[t!]
    \centering
    \footnotesize
	\begin{tabular}{c@{}c@{}c}
	Cora & Citeseer & Pubmed\\
        \includegraphics[width=0.325\textwidth]{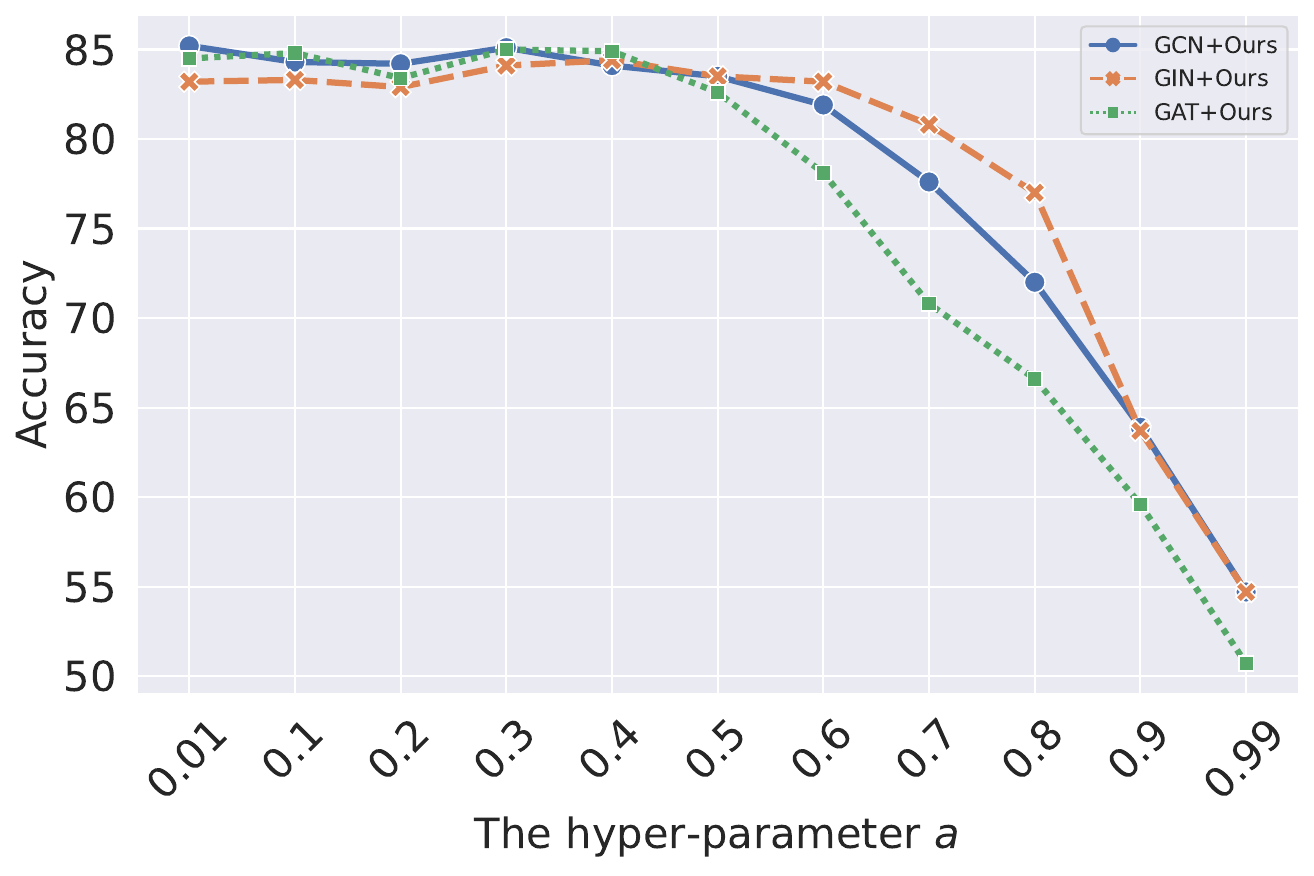}&
		\includegraphics[width=0.325\textwidth]{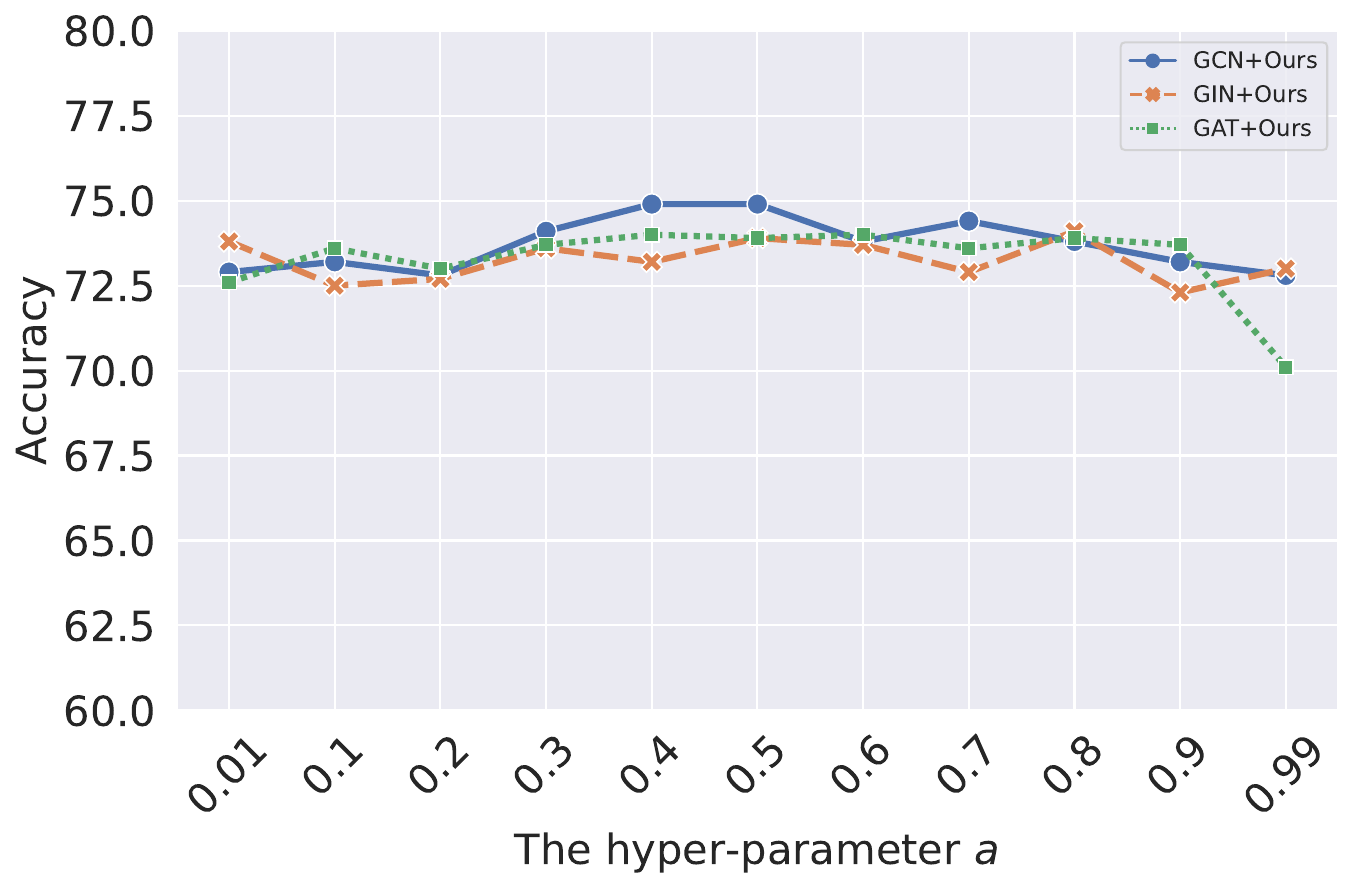}&
        \includegraphics[width=0.325\textwidth]{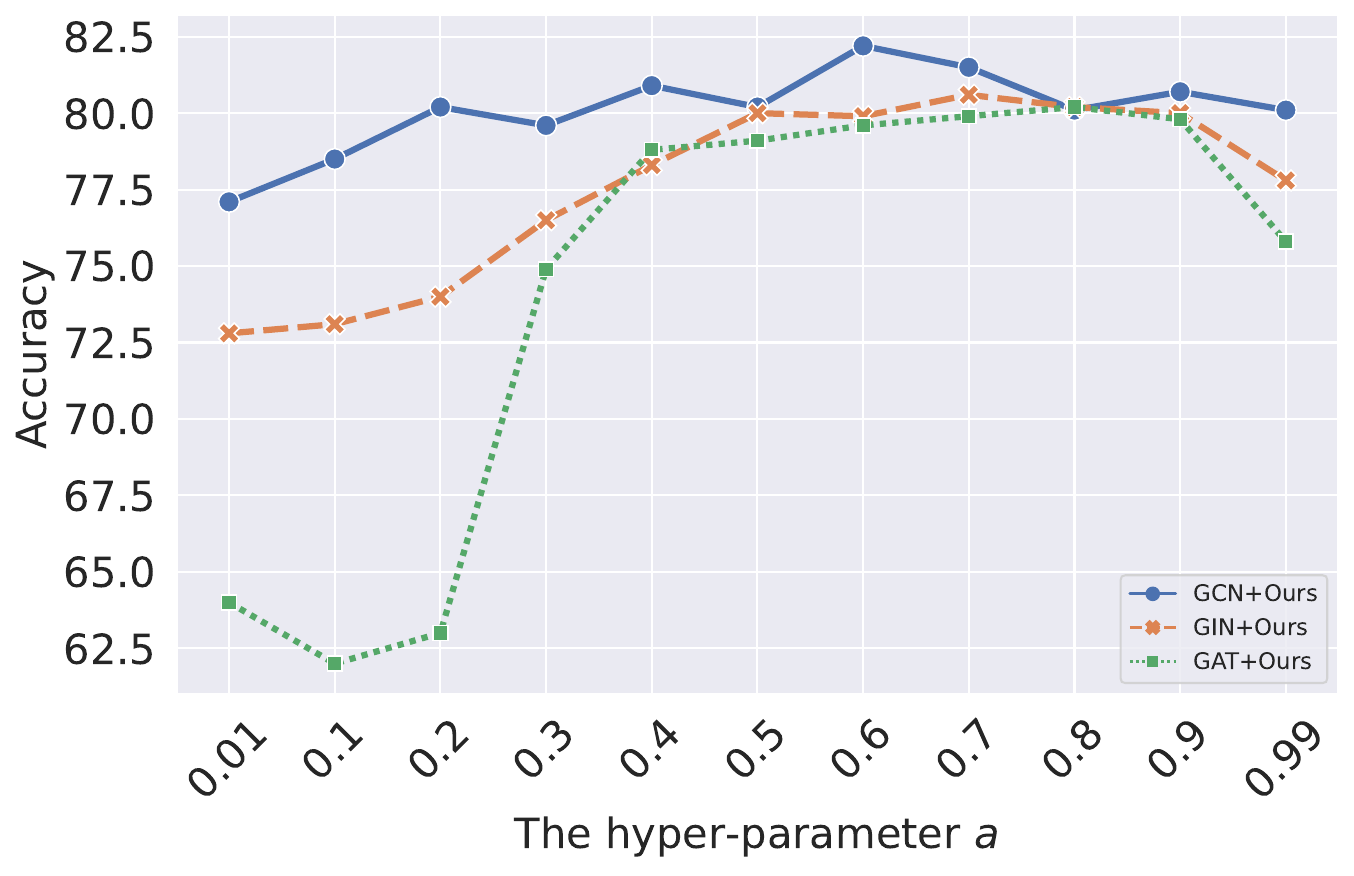}    
	\end{tabular}
	\caption{Results of Ours(GCN/GAT/GIN) against the varying hyper-parameter $a \in [0,1]$ controlling the normalizing extent}
	\label{fig:hyper-parameter_a}
\end{figure*}

\paragraph{Discussion On the Time and Space Complexities}
\label{sec:time_and_space_complexities}
The theoretical time complexity is analysed $O(L d^2 d_0)$ where $d,d_0 \ll n$ if partial-SVD or truncated-SVD~\cite{p117_TruncatedSVD} is utilized.
And it would become $O(L d^3)$ with a full SVD decomposition.
However, it is efficient under the high-parallelizability implemented via Pytorch.
The theoretical space complexity is $O\left(L\left(n+d\right)d\right)$.

\section{Conclusion}
\label{Sec:Conclusion_and_Limitations}
\focus{methods1:}In this paper, we propose \textit{R-SoftGraphAIN} to alleviate the over-smoothness of deep GNNs, by novelly employing soft normalization of the covariance matrix with appropriately incorporated residual connections.
\focus{details1:}We show in theory that the technique can maximally preserve the diversities of knowledge from both structures and features even at a sufficiently large depth against over-smoothness.
\focus{method2 and details:}Furthermore, in order to ease the difficulty of the optimization of deep GNNs, a label-smoothing-based curriculum learning framework (called \textit{SmoothCurriculum}) is proposed to intuitively encourage the encoder to digest knowledge from low- to high-frequency and to learn to discriminate nodes from coarse to fine.
\focus{experiments:}Extensive experiments were carried out against semi-supervised node classification tasks to show the effectiveness of our model by demonstrating its practical performance gain compared to twelve state-of-the-art baselines on twelve real-world graph benchmarks.
In future work, we will explore more applications of our method such as link prediction, graph classification, and community detection tasks.


\clearpage


\section*{Acknowledgements}
This work is supported by the National Science Foundation of China (Nos. 12271098) and Taishan Scholars Young Expert Project of Shandong Province (No. tsqn202211215) and the University-Industry Cooperation Project of Fujian Province, China (2023H6008).
The complete version of this paper can be found at \url{https://arxiv.org/abs/2312.08221} with all codes at \url{https://github.com/jslijin/Research-Paper-Codes/}.

\bibliography{aaai24}

\appendix

\clearpage

\appendix

\begin{algorithm*}[htbp!]
\caption{Fuzzy R-SoftGraphAIN (based on GCN)}
\label{algo:Fuzzy_R-SoftGraphAIN}

\vspace{0.2cm}

\textbf{Input}: adjacency matrix $A\in \{0,1\}^{n \times n}$, node features $H_0 = X \in \mathbb{R}^{n \times f}$ (with $n$ the node num and $f$ the feature dim)

\vspace{0.2cm}

\textbf{Output}: embeddings $H=H_L \in \mathbb{R}^{n\times d}$ (with $d$ the hidden dim)

\vspace{0.2cm}


\textbf{Hyper-Parameters}:

\vspace{0.1cm}

\quad \quad the coefficients in the soft graph normalization: $a, b \in [0,1]$ ($a+b\le 1$),

\vspace{0.1cm}

\quad \quad the coefficients in the skip connections: $\alpha, \gamma \in [0,1]$ ($\alpha + \gamma \le 1$), 

\vspace{0.1cm}

\quad \quad the coefficients in the fuzzy connections: $p, q \in [0,1]$,

\vspace{0.1cm}

\quad \quad the parameter in partial SVD algorithm: $d_0\le n$, 

\vspace{0.1cm}

\quad \quad the layer num: $L\in \mathbb{N_+}$, 

\vspace{0.1cm}

\quad \quad the non-linear activation function: $\sigma(\cdot)$ ($\operatorname{ReLU(\cdot)}$ default)

\vspace{0.2cm}

\textbf{Trainable-Parameters}: $\{W_{t}\}_{1\le t \le L}$ (in GCN)

\vspace{0.2cm}

\textbf{Note}: Here we consider general situations, where $W_1 \in \mathbb{R}^{f\times d}$ if $f \ne d$, otherwise $W_t \in \mathbb{R}^{d\times d}$ for $\forall t \in [1,L]$.

\vspace{0.2cm}

\begin{algorithmic}[1] 

\State $\tilde{A} \gets A + I_n$

\vspace{0.2cm}

\State $\tilde{D} \gets \mathrm{diag} \left( \tilde{A} \cdot \mathbf{1}_n \right)$  

\vspace{0.2cm}

\State $\hat{A} \gets \tilde{D}^{-\frac{1}{2}} \tilde{A} \tilde{D}^{-\frac{1}{2}}$

\vspace{0.2cm}

\State $T \gets I_n-\frac{1}{n} \mathbf{1}_n \mathbf{1}_n^T \in \mathbb{R}^{n \times n}$ 

\vspace{0.2cm}

\State $\beta \gets 1-\alpha - \gamma$, $\ q' \gets 1$

\Comment{calculate basic matrices $\hat{A}$ and $T$, and initialize $q'$}

\vspace{0.2cm}

\State $B_{0} \gets  T \cdot \hat{A} H_{0} W_1$

\Comment{obtain $B_0$ from $H_0$, for $t=1$}

\vspace{0.2cm}

\State $U_{d_0}, \Lambda_{d_0}^{\frac{1}{2}} \gets d_0 \mathrm{\text{-}SVD} \left(B_{0}^TB_{0}\right)$

\Comment{partial SVD with the parameter $d_0\le n$}

\Comment{$U_{d_0}, \Lambda_{d_0}^{\frac{1}{2}} \in \mathbb{R}^{n \times n}$ with $d-d_0 \  \ \mathbf{0}_n$s}

\vspace{0.2cm}

\State $ H_1 \gets \sigma \left( B_{0} \cdot \left[  a \cdot U_{d_0} \Lambda_{d_0}^{-\frac{1}{2} \cdot b} U_{d_0}^T + \left( 1-a \right) \cdot I_d \right] \right)$

\Comment{obtain $H_1$}

\vspace{0.2cm}


\State $S_{last} \gets H_1$, $\ \ S_{init} \gets H_1$

\Comment{initialize two fuzzy links}

\vspace{0.2cm}

\For{each $t \in [2,L]$}


\vspace{0.2cm}

\State $B_{t-1} = \alpha \cdot T \hat{A} H_{t-1} W_t + \beta \cdot S_{last} + \gamma \cdot S_{init}$
\label{appendix:algo_state1}

\Comment{obtain $B_{t-1}$ from $H_{t-1}$}

\vspace{0.2cm}

\State $U_{d_0}, \Lambda_{d_0}^{\frac{1}{2}} \gets d_0 \mathrm{\text{-}SVD} \left(B_{t-1}^TB_{t-1}\right)$

\Comment{partial SVD for $B_{t-1}^TB_{t-1}$ with $d_0\le n$}

\Comment{$U_{d_0}, \Lambda_{d_0}^{\frac{1}{2}} \in \mathbb{R}^{n \times n}$ with $d-d_0 \  \ \mathbf{0}_n$s}

\vspace{0.2cm}

\State $H_t \gets \sigma\left(B_{t-1} \cdot \left[  a \cdot U_{d_0} \Lambda_{d_0}^{-\frac{1}{2} \cdot b} U_{d_0}^T + \left( 1-a \right)  I_d \right]\right)$

\Comment{obtain $H_t$ from the SVD information related to $B_{t-1}$}

\vspace{0.2cm}

\State $q' \gets q' \cdot q$

\Comment{update $q'\in [0,1]$}

\vspace{0.2cm}

\State $S_{last} \gets S_{last} \cdot p + H_t$

\vspace{0.2cm}

\State $S_{init} \gets S_{init} + q' \cdot H_t$

\Comment{update two fuzzy links $S_{last}$ and $S_{init}$ (will be used in Line~\ref{appendix:algo_state1})}

\vspace{0.2cm}

\EndFor

\vspace{0.2cm}

\State \textbf{return} $H_L$

\vspace{0.2cm}

\end{algorithmic}
\end{algorithm*}

\section{The Organization of Supplementary Material}
The supplementary material is organized as follows:
\begin{itemize}
    \item some more related work in Sec.~\ref{sec_appendix:related_work};
    \item Algorithm~\ref{algo:Fuzzy_R-SoftGraphAIN} and its detailed description (including the details on \textit{Fuzzy Connections}) omitted in the main text due to space limit (see Sec.~\ref{sec_app:algorithm_1});
    \item an additional pre-processing node filtering operation in our implementation of the proposed curriculum learning framework (see Sec.~\ref{sec_app:node_filtering});
    \item the downstream linear classifier and loss function for node classification tasks in Sec.~\ref{sec_app:liner_classifier};
    \item omitted proofs in Sec.~\ref{sec_appendix:proofs};
    \item experimental details including dataset descriptions, more experimental results, implementation details, and hyper-parameter configurations in Sec.~\ref{sec_appendix:experimental_details};
    \item extensive visualizations in Sec.~\ref{sec_appendix:more_visualization};
    \item the time/space complexities in Sec.~\ref{sec_appendix:time_and_space_complexities};
    \item additional hyper-parameter studies on $\alpha$, $\gamma$, and $k_{KNN}$ in Sec.~\ref{sec_appendix:Additional_Hyper-parameter_Studies};
    \item discussion on the connections between R-SoftGraphAIN and SmoothCurriculum in Sec.~\ref{sec_appendix:connection_between_two_parts}.
\end{itemize}


%

\section{More Possibly Related Literature}
\label{sec_appendix:related_work}
In this section, we give a short review of more related work in order to enrich the background of this paper and improve its completeness.
However, we think the literature referenced in the main text is enough to understand the proposed method for readers.

Though some methods also aim to alleviate over-smoothness in deep GNNs, they have different design motivations or perspectives.
PDE-GCN~\cite{p101_Pde-gcn} and \cite{p102_Graph-coupled} are both ODE- or PDE-based methods.
\cite{p99_Grand} and \cite{p104_Neural_sheaf_diffusion} are designed from different diffusion processes.
\cite{p100_Beyond_homophily} is mainly designed by dealing with heterophily and does not focus on the over-smoothness.
\cite{p103_pathgcn} discards the traditional aggregation operator in GCN and proposes a novel path-based or random walk-based operator.
Its authors argue that the over-smoothness will not happen in this novel aggregation process and one can naturally avoid the necessity of alleviating this issue.

Besides, many works consider improving the training processes of models from different perspectives.
\cite{p106_Rethinking_graph_regularization} is based on regularization.
\cite{p105_GraphMix} and \cite{p107_Regularizing_Graph_Neural_Networks_via_Consistency_Diversity_Graph_Augmentations} are augmentation-based methods.
And compared to them, we propose a novel curriculum learning framework to find local minima with better generalizing characteristics.
\cite{p108_Sape} proposes an optimization algorithm sharing a similar concept of training from low to high frequency.
However, we are essentially different works from different domains with the aim of solving different problems.
\cite{p109_label_propagation_1} and \cite{p110_label_propagation_2} are also label propagation-related works.
However, they seem to have little connection with ours.

\section{Algorithm 1 (Fuzzy R-SoftGraphAIN)}
\label{sec_app:algorithm_1}
Recall Sec.~\ref{sec:Additional_Residual_Combination} in the main text.
For clearer understanding and better reproducibility, here we give an implementation of the comprehensive structure of the proposed \textit{Fuzzy R-SoftGraphAIN} in Algorithm~\ref{algo:Fuzzy_R-SoftGraphAIN}.

We give a brief description of Algorithm~\ref{algo:Fuzzy_R-SoftGraphAIN}:
\begin{itemize}
    \item Lines $1 \to 3$ compute the symmetrically normalized adjacency matrix $\hat{A}$ from $A$.
    \item Lines $4 \to 5$ compute the centering operator $T$, which can normalize the row mean vector of a node embedding matrix into $\mathbf{0}$, as well as the parameter $\beta$, and initializes $q'=q^0=1$.
    \item Lines $6 \to 8$ obtain $H_1$ from $H_0$ without any skip connection.
    \item Line $9$ initializes the matrices $S_{last}$ and $S_{init}$, which represent the items in fuzzy \textit{last} residual connection and fuzzy \textit{initial} connection, respectively (please kindly see the specific introduction in the following subsection).
    \item Line $10 \to 17$ describe the comprehensive updating process of node embedding matrix $H_t$.
    \item Line $11$ (corresponding to Eq.~$8$ in the main text) computes $B_{t-1}$ from $H_{t-1}$ with \textit{fuzzy connections} (please kindly see the specific introduction in the following subsection).
    \item Line $12$ does a partial singular value decomposition retaining the $d_0$ most dominating components of $B_{t-1}^TB_{t-1}$.
    \item Line $13$ computes $H_t$ from these components.
    \item Line $14$ updates $q'$ and keeps $q'=q^{t-1}$.
    \item Line $15 \to 16$ update the items $S_{last}$ and $S_{init}$ in \textit{fuzzy connections}.
    \item Line $18$ returns the final node embedding matrix $H_L$, which will be fed into a linear classifier to solve the subsequent downstream tasks such as node classification (see Sec.~\ref{sec_app:liner_classifier}).
\end{itemize}

\begin{figure*}[htbp]
\centering
\includegraphics[width=0.86\textwidth]{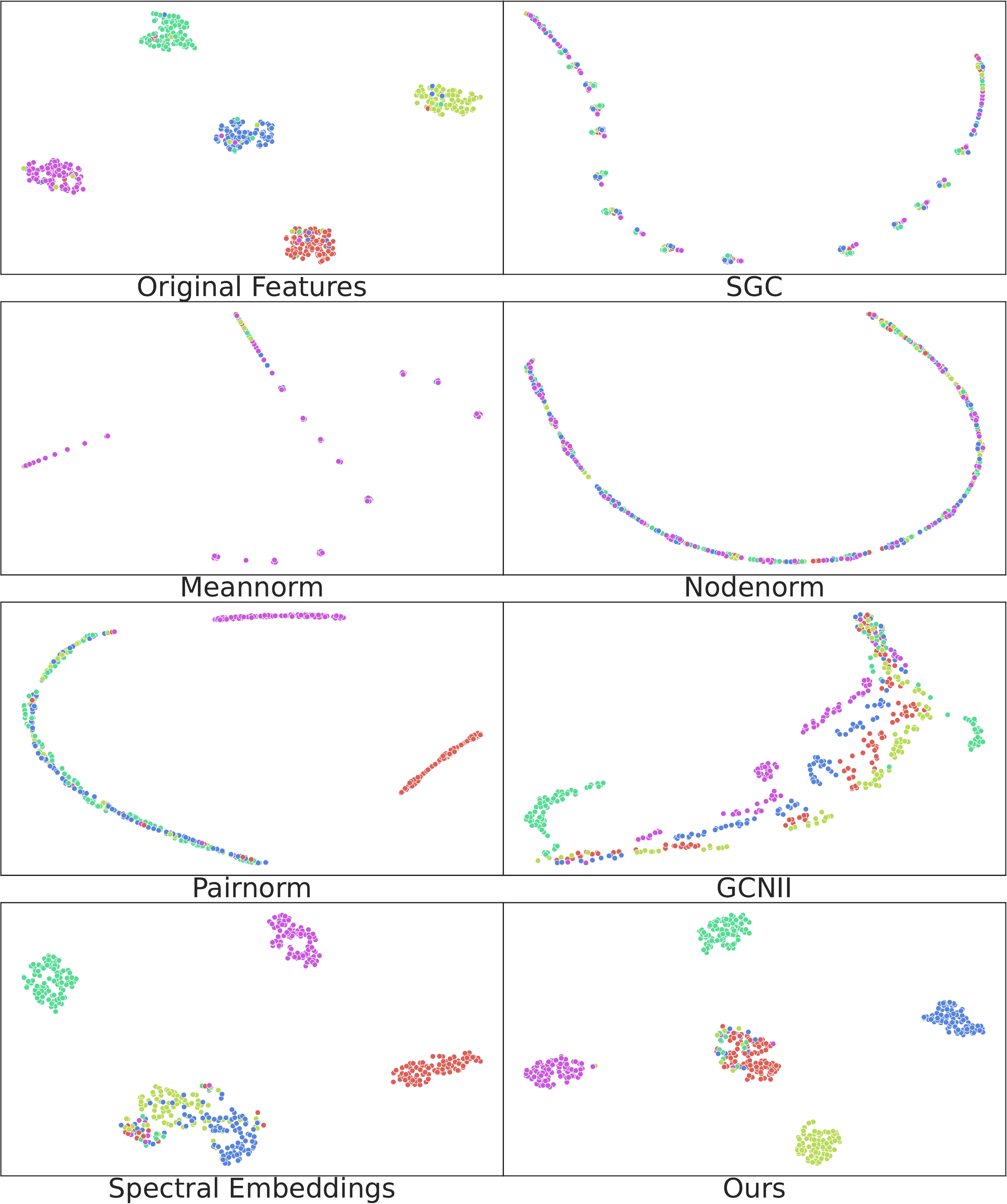}
\caption{Scatter plots for different deep models on synthetic data (considering only forward pass without trainable parameters). 
This vividly illustrates the phenomenon of the over-smoothness where node embeddings tend to collapse to a line or several line segments. We can clearly observe that ours can successfully keep the clusters away from blending, thus effectively alleviating the over-smoothness.
}
\label{fig_appendix:vis_comp1}
\end{figure*}

\subsection{Fuzzy Connections: Smoother Versions For Skip (Residual and Initial) Connections}
To counteract the negative impacts of possible noise in the node embeddings in an individual layer (\ie, $H_{t-1}$ for a residual connection and $H_{1}$ for an initial connection~\footnote{Note that we have substituted $X$ by $H_1$ due to possible mismatch between their dimensions (see the last two lines in Sec.~\ref{sec:Additional_Residual_Combination} in the main text), which makes our implementation more convenient.}), we design a smoother version of these two connections called fuzzy residual connection and fuzzy initial connection, respectively.
They can be specifically formulated by expanding Line $12$ in Algorithm~\ref{algo:Fuzzy_R-SoftGraphAIN} as follows.
\begin{equation}
\begin{split}
S_{init}^{(t)} &= \sum_{i=1}^{t-1} q^{i-1}\cdot H_{i} \in \mathbb{R}^{n \times d}, \\
S_{last}^{(t)} &= \sum_{i=1}^{t-1} q^{t-1-i}  \cdot H_{i} \in \mathbb{R}^{n \times d}.
\end{split}
\end{equation}
Then we can reformulate Line $8$ in Algorithm~\ref{algo:Fuzzy_R-SoftGraphAIN} as:
\begin{equation}
B_{t-1} = \alpha \cdot T \hat{A} H_{t-1} W_t + \beta \cdot S_{last}^{(t)} + \gamma \cdot S_{init}^{(t)}.
\end{equation}
One can easily check that $S_{init}^{(t)} \to H_1$ and $S_{last}^{(t)} \to H_{t-1}$ when $q \to 0$, which is a special degenerate case, \ie, vanilla initial and residual connections.

\begin{table*}[htbp]
\centering
\small
\begin{tabular}{c|ccccc|c}
\toprule
Dataset & Nodes & Edges & Ave.Degree & Features & Classes &  train/val/test \\ \midrule
Cora & 2,708 & 5,429 & 4.0 & 1,433 & 7 & 140/500/1000 \\
Citeseer & 3,327 & 4,732 & 2.84 & 3,703 & 6 & 120/500/1000 \\
PubMed & 19,717 & 44,338 & 4.5 & 500 & 3 & 60/500/1000 \\ \midrule
Coauthor CS & 18,333 & 81,894 & 8.93 & 6,805 & 15 & 300/450/17583 \\
Coauthor Physics & 34,493 & 247,962 & 14.38 & 8,415 & 5 & 100/150/34243 \\
Amazon Computers & 13,381 & 245,778 & 36.74 & 767 & 10 & 200/300/12881 \\
Amazon Photo & 7,487 & 119,043 & 31.8 & 745 & 8 & 160/240/7087 \\ \midrule
Texas & 183 & 309 & 3.38 & 1,703 & 5 & 87/59/37 \\
Wisconsin & 251 & 499 & 5.45 & 1,703 & 5 & 120/80/51 \\ 
Cornell & 183 & 295 & 3.22 & 1,703 & 5 & 87/59/37 \\ 
Actor & 7600 & 33544 & 8.83 & 931 & 5 & 3648/2432/1520 \\ \midrule
OGBN-ArXiv & 169,343 & 1,166,243 & 13.77 & 128 & 40 & 91446/30482/47415 \\ 
\bottomrule
\end{tabular}
\caption{Dataset statistics of twelve real-world benchmarks with their splits}
\label{tab_appendix:dataset_statistics}
\end{table*}

\subsection{How To Apply It On Other GNN Encoders}
Since our method does not essentially depend on or alter the aggregating operations employed in GCN, it can be naturally viewed as a plug-and-play module for nearly all spatial aggregation-based GNN encoders.
In our experiments, in addition to GCN~\cite{p4_GCN}, we also evaluate our method based on GAT~\cite{p8_GAT} and GIN~\cite{p7_GIN} (\ie, \textit{Ours(GAT)} and \textit{Ours(GIN)}), and we found that on some graph benchmarks (\eg, OGBN-ArXiv) they perform quite competitive or better than \textit{Ours(GCN)} even adopting the exactly same hyper-parameters as \textit{Ours(GCN)}'s without any further tuning.
Here, we give a short description of how to achieve that. 

From Line $6$ and Line $11$ in Algorithm~\ref{algo:Fuzzy_R-SoftGraphAIN}, one can easily find that the formula $\hat{A}H_{t-1}W_{t}$ keeps intact as that in the $t$-th layer of GCN, which might imply that it can be substituted by some other aggregating formulas in another GNN encoders, \ie, a universal form of $H_t = \operatorname{Agg}(H_{t-1};W_{t})$.

More specifically, in GAT, it is:
\begin{equation}
    H_t^{(i)} = {\bigg\|}_{k=1}^K \sigma \left(\sum_{j \in \mathcal{N}_i} \alpha_{ij}^k {W_t}^k H_{t-1}^{(j)} \right),
\end{equation}
where $H_t^{(i)}$ denotes the embeddings of node $i\in V$ at layer $t\in [1,L]$, $K$ is the number of attention heads, $\mathcal{N}_i$ is the first-order neighborhood of node $i$, and $\alpha_{ij}^k$ denoting the attention coefficient from node $i$ to node $j$ via head $k$ can be represented as follows:
\begin{equation}
    \alpha_{ij} = \frac{\operatorname{exp}\left(\operatorname{LeakyReLU}\left(\left[H^{(i)} \| H^{(j)}\right]W'\right)\right)}{\sum_{j\in \mathcal{N}_i} \operatorname{exp}\left(\operatorname{LeakyReLU}\left(\left[H^{(i)} \| H^{(j)}\right]W'\right)\right)},
\end{equation}
where $W'$ is another trainable matrix.

As for GIN, $H_t = \operatorname{Agg}(H_{t-1};W_{t})$ has a simple formula:
\begin{equation}
H_t^{(i)} = \operatorname{MLP} \left( \left( 1 + \epsilon \right) \cdot H_{t-1}^{(j)} + \sum_{j \in \mathcal{N}_i} H_{t-1}^{(j)} \right).
\end{equation}

Note that although we can always apply it to more sophisticated GNN encoders in addition to the selected basic ones for the sake of further performance improvement, we want to keep it simple and evaluate its potential.
Some expanding work might be left as future work.

\section{Additional Node Filtering Pre-Processing Operation in Curriculum Learning Framework}
\label{sec_app:node_filtering}
We know that there would exist possible noises in the label estimation process (Sec.~\ref{Sec:Label_Estimation_and_Auxiliary_Graph}).
Furthermore, we found in experiments that the negative impacts introduced by these nettlesome noises could be accumulated during the subsequent smoothing process (Sec.~\ref{Sec:label_smoothing_and_curriculum_design}), which would potentially damage the final performance to some non-negligible extent.
To alleviate this issue, we apply an additional node filtering operation after label estimation based on the unconfidence scores estimated by the following normalized entropy function of these pseudo labels (\ie, lots of probability distributions):
\begin{equation}
H(p) = \frac{- \sum_{i=1}^{C} p_i \cdot \log\left( p_i \right)}{\log C} \in [0,1],
\end{equation}
where we consider $p \in \mathbb{R}^{C}$ as a probability distribution or a pseudo label of some node $v\in V$ and $C$ is the number of classes in the node classification tasks.
We filter those nodes with too large normalized entropy values out by substituting their pseudo labels via $C$-dimensional all-zero vectors because we think that they are unconfident predictions with little meaningful knowledge.

And in this process, we introduce a hyper-parameter $mask\_ratio \in [0,1]$ to control how many nodes we hope to filter out (see the hyper-parameter Tab.~\ref{tab_appendix:Hyper_parameters_configurations} and Tab.~\ref{tab_appendix:hyper-parameters_search_spaces}).

Besides, after this filtration, some pseudo-labels would become invalid probability distributions during the smoothing process.
To fix this issue, we re-normalize all the pseudo-labels~\footnote{to make the sum of elements in every pseudo-label equal to $1$}.

Please kindly notice that there would be some other alternative methods to alleviate the negative impacts due to noise.
Here we just give an example, which is employed in our experiments.

\begin{figure*}[htbp]
    \centering
    \footnotesize
	\begin{tabular}{c@{}c@{}c}
	Cora & Citeseer & Pubmed\\
        \includegraphics[width=0.325\textwidth]{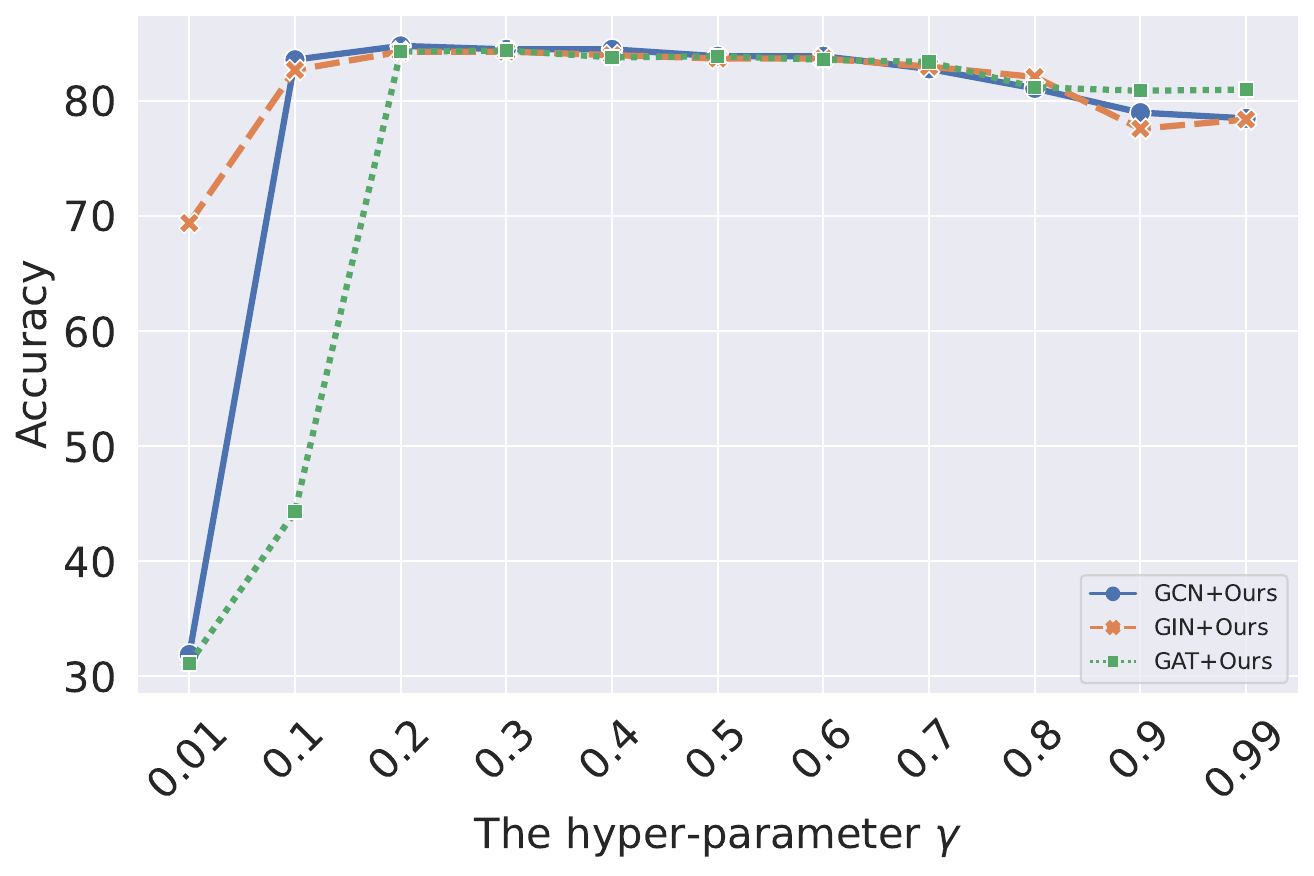}&
		\includegraphics[width=0.325\textwidth]{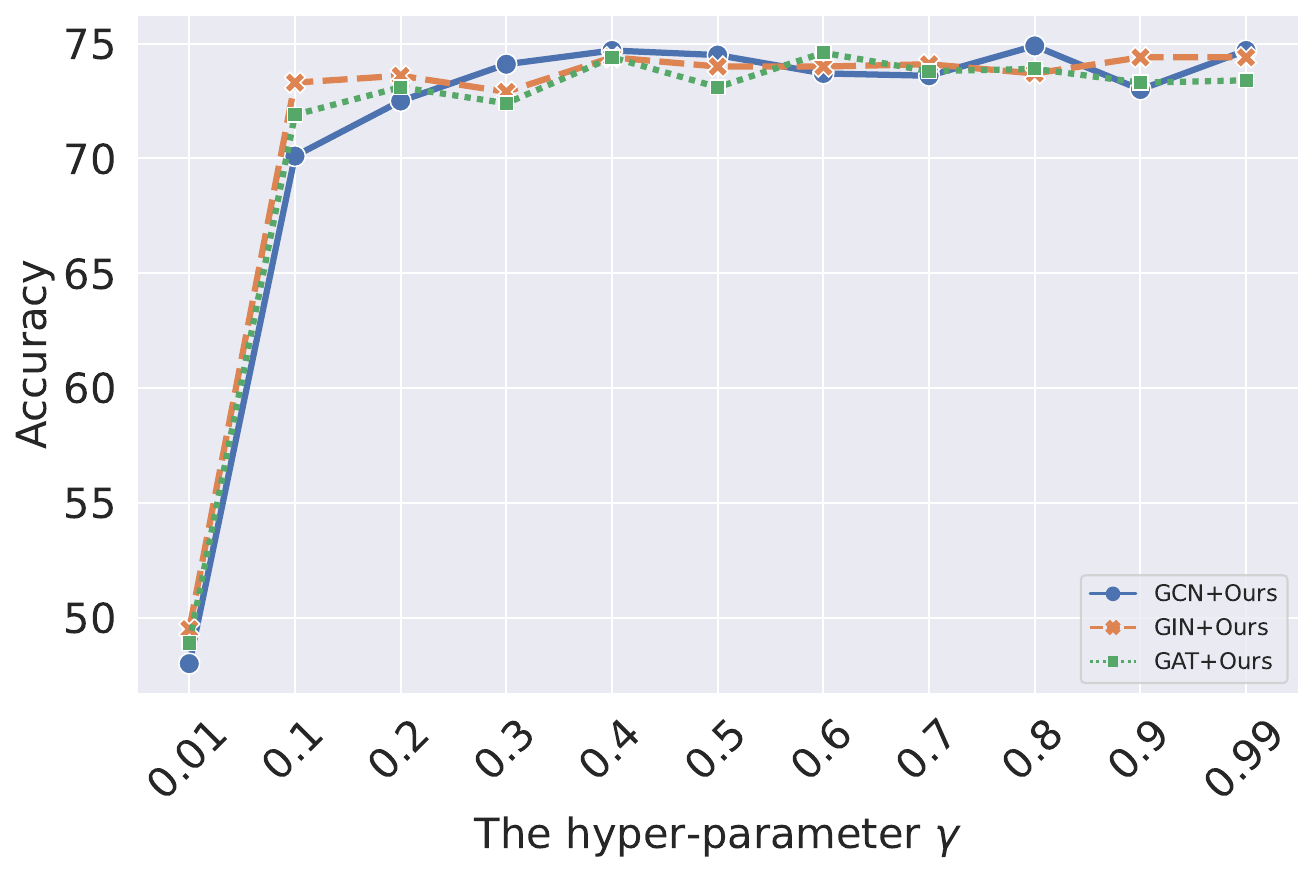}&
        \includegraphics[width=0.325\textwidth]{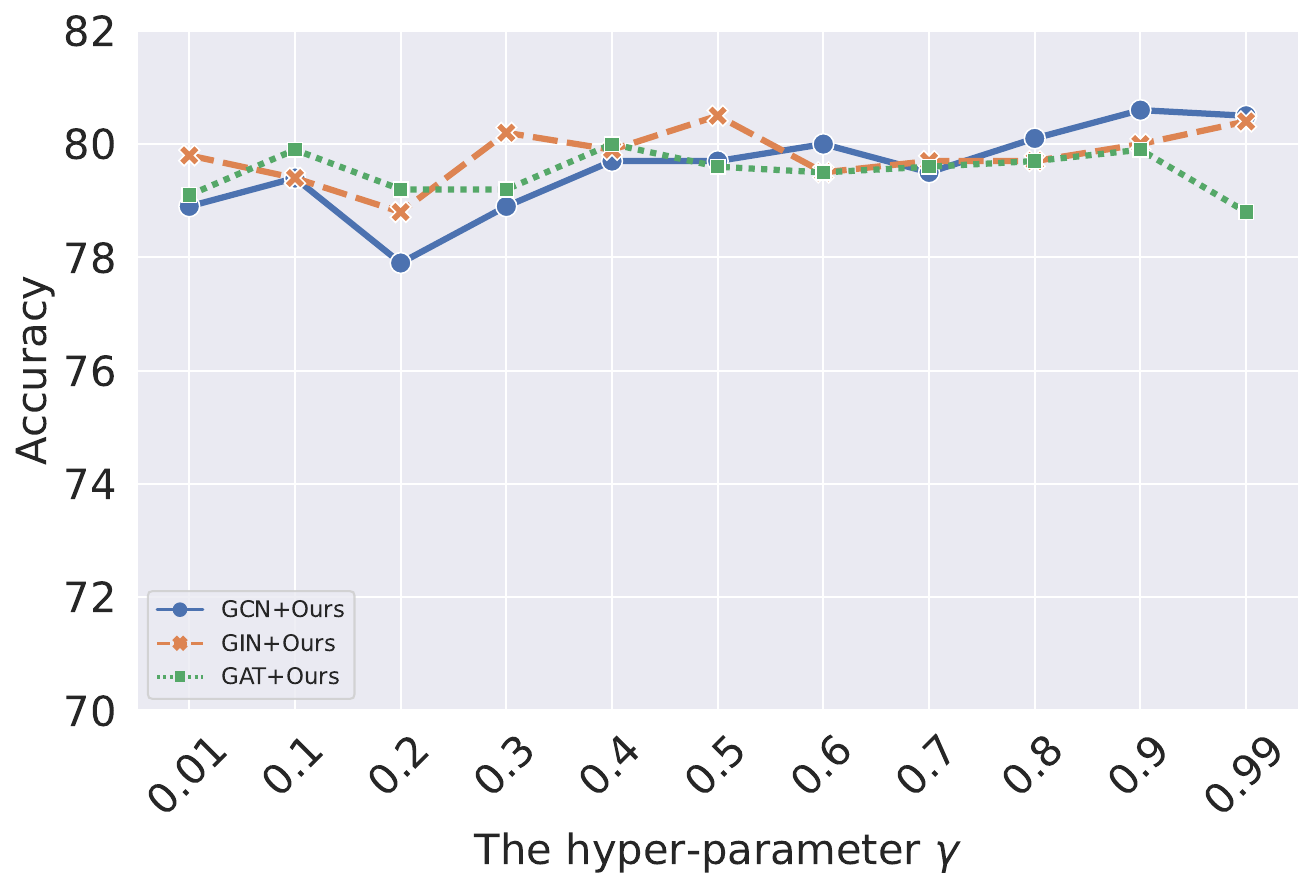}    
	\end{tabular}
	\caption{Results of Ours(GCN/GAT/GIN) against the varying hyper-parameter $\gamma \in [0,1]$ controlling the coefficient of the trade-off in Skip Connections, \ie, between initial connections and residual connections}
	\label{app_fig:hyper-parameter_gamma}
\end{figure*}

\begin{figure*}[htbp]
    \centering
    \footnotesize
	\begin{tabular}{c@{}c@{}c}
	Cora & Citeseer & Pubmed\\
        \includegraphics[width=0.325\textwidth]{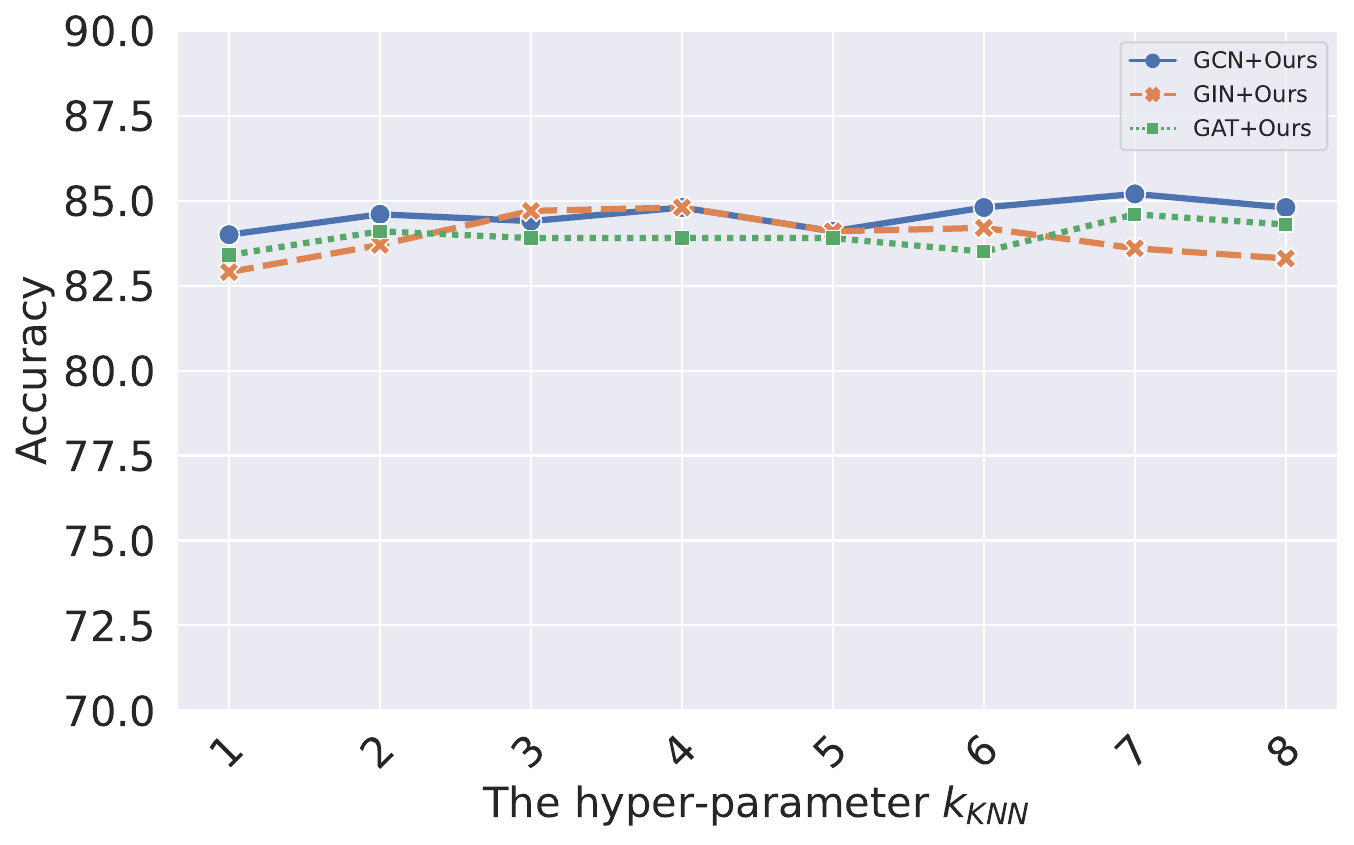}&
		\includegraphics[width=0.325\textwidth]{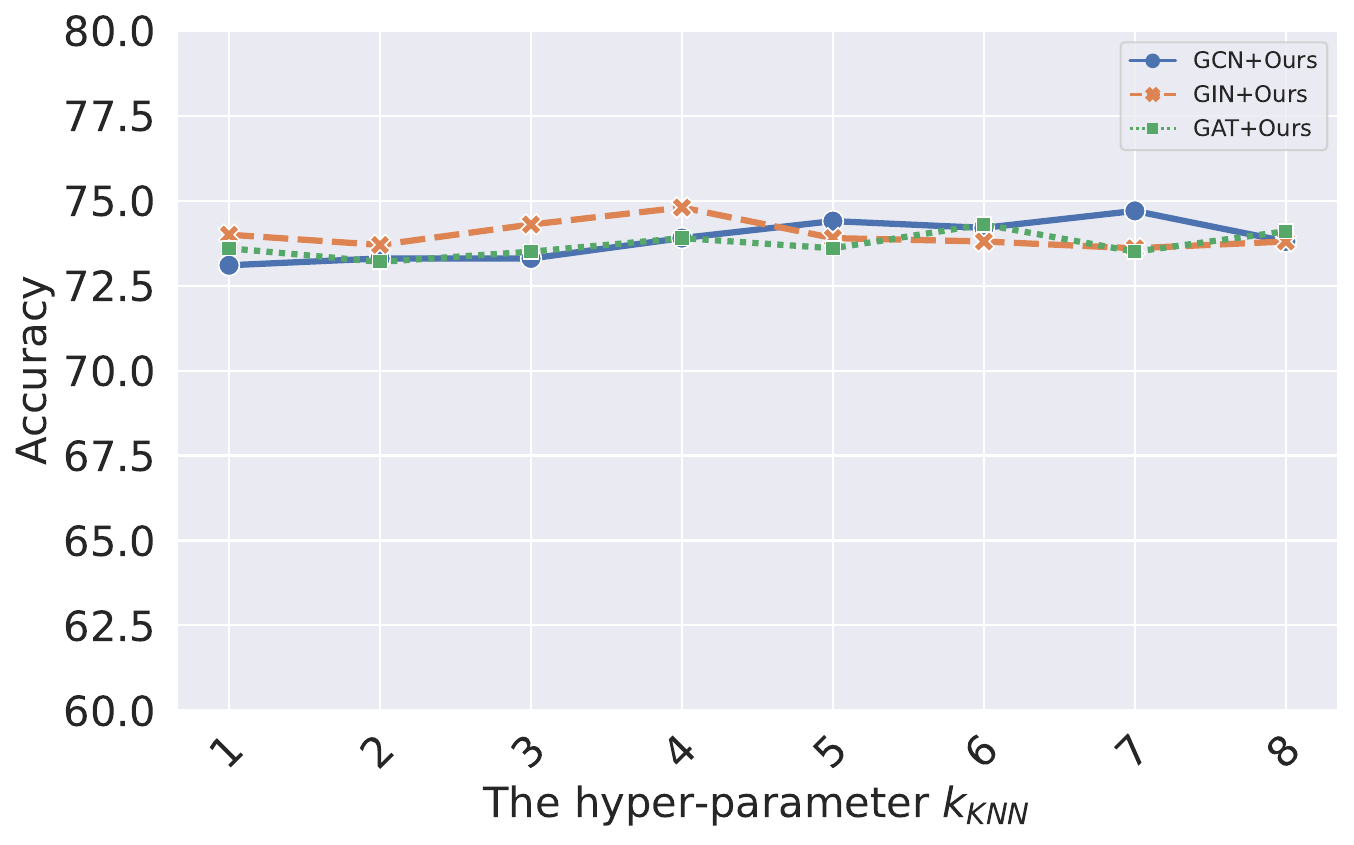}&
        \includegraphics[width=0.325\textwidth]{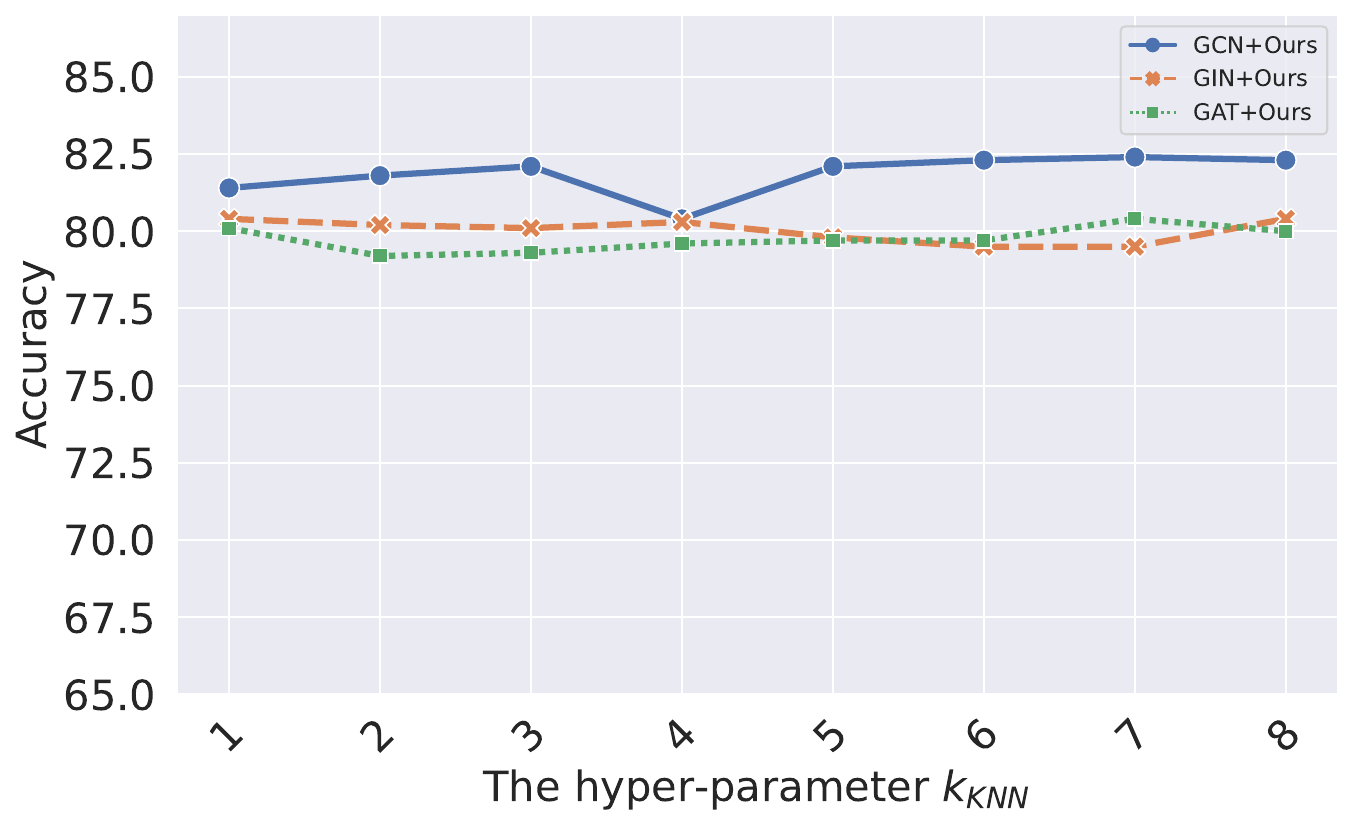}    
	\end{tabular}
	\caption{Results of Ours(GCN/GAT/GIN) against the varying hyper-parameter $k_{KNN} \in \mathbb{N}_+$ controlling the number of the closest neighbors chosen by the vanilla KNN algorithm for every node}
	\label{app_fig:hyper-parameter_k}
\end{figure*}

\section{Downstream Linear Classifier and Loss Function for Node Classification Tasks}
\label{sec_app:liner_classifier}
We employ the following downstream linear classifier and the cross-entropy loss function:
\begin{equation}
\begin{split}
L_{CE}(H^{(L)},y) &= - \frac{1}{|V|} \sum_{v\in V} \sum_{c\in C} y_v(c) \cdot \log \hat{y}_v(c), \\
\hat{y}_v &= \operatorname{Softmax} \left(H^{(L)}_v W_{cls}\right) \in \mathbb{R}^{1 \times C},
\end{split}
\end{equation}
where $H^{(L)}$ is the final node embedding matrix, $y$ is the ground-truth one-hot node labels, and $W_{cls} \in \mathbb{R}^{d \times C}$ is a trainable projection matrix.
We train our model end-to-end by minimizing this loss function.

Note that we should substitute the $y$ vector with the corresponding node pseudo-labels used in different stages of the proposed curriculum learning framework.

\section{Omitted Proofs}
\label{sec_appendix:proofs}

In this section, we give the proofs of the theorems and corollaries omitted in the main text.

\begin{theorem}
\label{theorem_app:hard_GraphAIN_basic_characteristics}
(Corresponding to Theorem $1$ in the main text)

$\forall \ t \ge 1$, GraphAIN satisfies that:

1) $\mathbf{1}_n^T B_t = \mathbf{1}_n^T H_t = \mathbf{0}_n$;
\quad 2) $H_t^T H_t = I_d$;

3) $T H_t = H_t$;
\quad \quad \quad \quad \quad 4) $B_{t} = \Bar{A} H_{t}$;

where $\Bar{A} = T\hat{A}T$ denotes a doubly centered version of $\hat{A}$.
\end{theorem}

\begin{proof}
Recall that \textit{GraphAIN} can be formulated as follows:
\begin{equation}
    \label{equ_app:GraphAIN}
    \begin{split}
        &H_t = B_{t-1} \left( B_{t-1}^T B_{t-1}  \right)^{-\frac{1}{2}}  \in \mathbb{R}^{n \times d} , \quad \forall \ t \ge 1, \\
        &B_{t-1} = T \cdot \hat{A} H_{t-1} \in \mathbb{R}^{n \times d} , \quad H_0 = X,
    \end{split}
\end{equation}

1) Because $T = I_n - \frac{1}{n} \mathbf{1}_n \mathbf{1}_n^T$,  $\mathbf{1}_n^T T = \mathbf{0}_n^T$ holds. That is, it is true that $\mathbf{1}_n^T B_t = \mathbf{0}_d^T, \forall \ t \ge 0$, and consequently $\mathbf{1}_n^T H_t = \mathbf{0}_d^T, \forall \ t \ge 1$.

2) According to the definition: $\forall \ t \ge 1, \ H_t^T H_t =     \left( B_{t-1}^T B_{t-1} \right)^{-\frac{1}{2}}  \left( B_{t-1}^T   B_{t-1}  \right) \left( B_{t-1}^T B_{t-1} \right)^{-\frac{1}{2}}  = I_n   $.

3) $\mathbf{1}_n^T H_t = \mathbf{0}_d$, so $T \cdot H_t = \left( I_n - \frac{1}{n} \mathbf{1} \mathbf{1}^T \right) H_t = H_t - \frac{1}{n} \mathbf{1} \mathbf{1}^T H_t = H_t$.

4) $\forall \ t \ge 1, B_{t} = T \hat{A} H_t = T \hat{A} T H_t = \Bar{A} H_{t}$.
\end{proof}

\begin{table*}[htbp]
\small
\centering
\begin{tabular}{c|c|c|c|c|c}
\toprule
\diagbox{Method}{Cora}  
& 2 & 8 & 16 & 32 & 64 \\ \midrule
GCN & 81.45±0.37 & 72.72±2.37 & 63.40±4.40 & 31.90±0.70 & 27.56±3.36 \\
SGC & 79.00±0.46 & 78.02±1.06 & 75.26±1.06 & 63.84±3.63 & 55.39±2.51  \\
ChebNet & 80.33±1.19	&	78.70±1.51&	68.40±1.41&	31.90±1.77&	20.63±5.81
\\
GAT &
82.40±0.05&
78.30±0.47&
76.50±0.34&
72.28±0.61&
31.92±0.04
\\
GIN &80.54±0.36	&	76.74±1..92&	68.72±1.30&	60.92±3.10&	31.90±0.00
\\
PairNorm & 78.30±1.33 & 71.06±1.74 & 66.80±2.71 & 65.00±3.74 & 66.24±1.58 \\
GCNII & 82.19±0.77 & 84.23±0.42 & \underline{84.69±0.51} & \textbf{85.29±0.47} & \textbf{85.34±0.32} \\
JKNet & 79.06±0.11 & 75.66±0.38 & 72.97±3.94 & 73.23±3.59 & 72.54±3.65 \\
GPRGNN & 82.53±0.49 & 84.19±0.40 & 83.69±0.55 & 83.13±0.60 & 82.48±0.26 \\
DAGNN & 80.30±0.78 & 84.28±0.59 & 84.14±0.59 & 83.39±0.59 & 82.16±0.35 \\
APPNP & 82.06±0.46 & 83.59±0.40 & 83.64±0.48 & 83.68±0.48 & 83.66±0.36 \\
BernNet &
81.88±0.55&
83.44±1.06&
82.44±0.40&
81.38±0.44&
17.72±6.44
\\\midrule
Ours(GCN) & 84.20±0.12 & \textbf{85.22±0.20} & \textbf{84.96±0.21} & \underline{85.12±0.15} & 84.60±0.29  \\
Ours(GAT) &
\textbf{84.62±0.08}&		\underline{84.92±0.26}&	84.44±0.30&	84.60±0.31&	\underline{84.68±0.31}\\
Ours(GIN) &\underline{84.24±0.33}	&	83.96±0.35&	84.08±0.41&	84.18±0.58&	83.80±0.16
\\
\bottomrule

\end{tabular}
\caption{Accuracy comparison of node classification tasks on Cora}
\label{tab_appendix:comparative_experiments_cora}
\end{table*}

\begin{lemma}
\label{lemma:projection}
Consider the following linear regression with an orthogonal condition:
\begin{equation}
    \begin{split}
        Q^* = \arg\min_Q\|Q A-B\|_F \quad \mathrm{s.t.}\quad Q^T
Q=I_d,
    \end{split}
\end{equation}
and then we have $Q^* = UV^T$,
where $BA^T = U \Sigma V^T$ is the eigen-decomposition of $BA^T$, $U,V \in \mathbb{R}^{d \times d}$ are two orthogonal matrices, and $\Sigma \in \mathbb{R}^{d \times d}$ is a diagonal matrix containing all the singular values of $BA^T$.
\end{lemma}

\begin{proof}
\begin{align*}
Q^* &= \arg\min_Q ||Q A-B\|_F^2 \\
&= \arg\min_Q  \langle Q A-B, Q A-B \rangle_F  \\
&= \arg\min_Q  \|Q A\|_F^2 + \|B\|_F^2 - 2 \langle Q A , B \rangle_F \\
&= \arg\min_Q  \|A\|_F^2 + \|B\|_F^2 - 2 \langle Q A , B \rangle_F  \\
&= \arg\max_Q  \langle Q , B A^T \rangle_F  \\
&= \arg\max_Q  \langle Q, U\Sigma V^T \rangle_F  \\
&= \arg\max_Q  \langle U^T Q V , \Sigma \rangle_F  \\
&= \arg\max_Q  \langle S , \Sigma \rangle_F  \quad \text{where } S = U^T Q V, 
\end{align*}
where $S$ is an orthogonal matrix and $\Sigma$ is a diagonal matrix. Then $S$ should equal  $I_d$ so as to maximize the last objective.
That is $S=U^T Q^* V = I_d$, which means the solution is $Q^*=UV^T$.
\end{proof}

\begin{corollary}
\label{corollary:projection}
Consider the following projection problem:
\begin{equation}
    \begin{split}
        Q^* = \arg\min_Q\|Q-B\|_F \quad \mathrm{s.t.}\quad Q^T
Q=I_d,
    \end{split}
\end{equation}
if $B = U \Sigma V^T$ is the eigen-decomposition of $B$, then:
\begin{equation}
    \begin{split}
        Q^* = UV^T.
    \end{split}
\end{equation}
\end{corollary}

\begin{proof}
Incorporate $A=I_n$ into Lemma \ref{lemma:projection}, the conclusion can be immediately achieved.
\end{proof}

\begin{table*}[htbp]
\small
\centering
\begin{tabular}{c|c|c|c|c|c}
\toprule
\diagbox{Method}{Citeseer} 
 & 2 & 8 & 16 & 32 & 64 \\ \midrule
GCN &  69.46±0.29 & 57.74±6.30 & 48.70±3.48 & 36.66±8.61 & 25.40±2.34  \\
SGC &  67.92±0.85 & 68.42±0.46 & 68.08±0.73 & 67.50±1.43 & 63.08±0.29  \\
ChebNet &69.27±0.28&
63.20±1.92&	58.73±2.62&	33.43±3.10&	24.90±1.25
\\
GAT &
71.38±0.04&	64.69±0.27&	62.20±0.25&	59.08±0.44&	22.90±2.23
\\
GIN &68.02±0.41&		59.80±0.79&	53.55±1.98&	47.32±3.90&	23.10±0.00
\\
PairNorm & 65.80±1.35 & 54.81±6.46 & 46.26±2.69 & 44.20±1.23 & 41.48±4.63  \\
GCNII &  67.81±0.89 & 70.62±0.63 & 72.97±0.71 & 73.24±0.78 & 73.00±0.75 \\
JKNet &  66.98±1.82 & 60.56±1.41 & 54.33±7.74 & 50.68±8.73 & 52.22±6.99  \\
GPRGNN &  70.49±0.95 & 71.47±0.58 & 71.39±0.73 & 71.01±0.79 & 70.96±0.38  \\
DAGNN &  70.91±0.68 & 72.44±0.54 & 73.05±0.62 & 72.59±0.54 & 71.00±0.55  \\
APPNP & 71.67±0.78 & 72.04±0.52 & 72.13±0.53 & 72.13±0.59 & 72.02±0.46 \\
BernNet &
72.02±0.53&
71.44±0.22&
70.70±0.91&
70.82±0.56&
39.10±12.92
\\\midrule
Ours(GCN) & \textbf{74.84±0.32} & \textbf{74.72±0.26} & \underline{74.12±0.08} & \underline{74.42±0.26} & \textbf{74.50±0.23} \\
Ours(GAT) &\underline{74.82±0.36}&
\underline{74.24±0.51}&	\textbf{74.24±0.47}&	74.26±0.11&	74.04±0.19
\\
Ours(GIN) &74.58±0.81&		74.22±0.08&	74.04±0.33&	\textbf{74.88±0.31}&	\underline{74.16±0.29}
\\
\bottomrule
\end{tabular}
\caption{Accuracy comparison of node classification tasks on Citeseer}
\label{tab_appendix:comparative_experiments_citeseer}
\end{table*}

\begin{theorem}
\label{theorem_app:hard_GraphAIN_optimization}
(Corresponding to Theorem $2$ in the main text)

GraphAIN can be viewed as an iterative process in order to solve the following restricted optimization problem via \textbf{Projected Gradient Ascent} method:
\begin{equation}
    \label{equ_app:objective_function_hard_GraphAIN}
    \begin{split}
        \max_H \quad f(H) = \frac{1}{2} \cdot \mathrm{tr}\left( H^T \Bar{A} H \right), \quad s.t. \ \  H^TH=I_d,
    \end{split}
\end{equation}
where $H$ is initialized to input graph signals $X \in \mathbb{R}^{n \times d}$ and  $\eta = 1$ is set as the step size.
\end{theorem}

\begin{proof}
We consider the following equivalent objective via adding an additional constant:
\begin{equation}
    \begin{split}
        H^* = \arg \max_H f(H) = \frac{1}{2} \cdot \mathrm{tr}\left( H^T \left( \Bar{A} - I_n \right) H \right).
    \end{split}
\end{equation}
Then we take the derivative of $f(H)$:
\begin{equation}
    \begin{split}
        \frac{\partial f(H)}{ \partial H} = \left( \Bar{A} - I_n \right) H = \Bar{A} H - H.
    \end{split}
\end{equation}
Projected Gradient Ascent includes two steps in every iteration: First, we use original gradient ascent, that is:
\begin{equation}
    \begin{split}
        H' = H + \eta \cdot \frac{\partial f(H)}{\partial H},
    \end{split}
\end{equation}
where $\eta$ is step-size, $H$ is the old $H$, $H'$ is the new one.

If we use $\eta=1$, then we get the following equation:
\begin{equation}
    \begin{split}
        H' = H + \frac{\partial f(H)}{\partial H}
        = H + \Bar{A} H - H
        = \Bar{A} H.
    \end{split}
\end{equation}
Second, since $H'$ might not satisfy the constraint, we project it to the set satisfying the constraint:
\begin{equation}
    \begin{split}
        H'' &= H' \left( H'^T H' \right)^{-\frac{1}{2}}.
    \end{split}
\end{equation}
The projection operation makes sense because $H''$ is the nearest orthogonal matrix satisfying the constraint around $H'$.
To see this, we do singular value decomposition for $H' = U \Lambda V^T$, then:
\begin{equation}
    \begin{split}
        H'' &= H' \left( H'^T H' \right)^{-\frac{1}{2}} \\
            &= U \Lambda V^T \left( V \Lambda U^T U \Lambda V^T \right)^{-\frac{1}{2}}\\
            &= U \Lambda V^T V \Lambda^{-1} V^T\\
            &= U V^T.
    \end{split}
\end{equation}

Thus according to Corollary \ref{corollary:projection}, we know that $H''$ can be seen as a projection of $H'$ onto the set consists of all orthogonal matrixs.
\end{proof}

\begin{theorem}
\label{theorem_app:R-soft_GraphAIN_optimization}
(Corresponding to Theorem $3$ in the main text)

Residual-favored GraphAIN can be viewed as an iterative process of solving the following restricted optimization problem via \textbf{Projected Gradient Ascent} method:
\begin{equation}
    \label{equ_app:objective_function_R_Soft_GraphAIN}
    \begin{split}
        \max_H \ f(H) = &\frac{1}{2} \cdot \mathrm{tr}\left( H^T \Bar{A} H \right) - \frac{1}{2} \cdot \frac{c}{a} \cdot \left\| H - X \right\|_F^2 \\ &s.t. \quad  H^TH=I_d,
    \end{split}
\end{equation}
where $H$ is initialized to input graph signals $X \in \mathbb{R}^{n \times d}$ and set the step size $
\eta = a \in [0,1]$.
\end{theorem}
For convenience, we slightly abuse the symbols used in the main text.
\begin{proof}
Consider the following generalized objective:
\begin{equation}
\label{equ:obj1}
    \begin{split}
        f(H)=\frac{1}{2}\cdot \mathrm{tr}\left( H^T \left( \Bar{A} - p \cdot I_n \right) H \right) - \frac{1}{2} \cdot \beta \cdot \left \| H - X \right \|_F^2.
    \end{split}
\end{equation}
Then we take the derivative of $f(H)$:
\begin{equation}
    \begin{split}
        \frac{\partial f(H)}{\partial H} &= \left( \Bar{A} - p \cdot I_n \right) \cdot H - \beta \cdot \left( H - X \right) \\
        &= \Bar{A} \cdot H - \left( \beta + p \right) \cdot H + \beta \cdot X.
    \end{split}
\end{equation}
We apply the SGD updating formulation with step size $\eta$ as follows:
\begin{equation}
\label{equ:xxx}
    \begin{split}
        H' &= H + \eta \cdot \frac{\partial f(H)}{\partial H} \\
        &= H + \eta \cdot \left( \Bar{A} \cdot H - \left( \beta + p \right) \cdot H + \beta \cdot X \right) \\
        &= \eta \cdot \Bar{A} \cdot H + \left( 1 - \eta \cdot \left( \beta + p \right) \right) \cdot H + \eta \cdot \beta \cdot X.
    \end{split}
\end{equation}
Then let the coefficients equal to $a,b$ and $c$, respectively, we can get the following equations:
\begin{equation}
    \begin{cases}
        \eta = a \\
        1-\eta \cdot \left( \beta + p \right) = b \\
        \eta \cdot \beta = c.
    \end{cases}
\end{equation}
Thus, we  achieve:
\begin{equation}
    \begin{cases}
        \eta = a \\
        \beta = \frac{c}{a} \\
        p = \frac{1-b-c}{a}.
    \end{cases}
\end{equation}
Then Eq. \ref{equ:xxx} can be transformed into:
\begin{equation}
    \begin{split}
        H' = a \cdot \Bar{A} H + b \cdot H + c \cdot X.
    \end{split}
\end{equation}
This is exactly the propagation rule in \textit{Residual-favored GraphAIN}.

Moreover, according to the restriction $a+b+c=1$, we can determine the value of $p$ as follows:
\begin{equation}
    \begin{split}
        a+b+c &= \eta + \left( 1-\eta \cdot \left( \beta + p \right) \right) + \eta \cdot \beta \\
        &= \left( 1-p \right) \cdot \eta + 1 = 1,
    \end{split}
\end{equation}
which means  $p=1$ is true.

Then let $\beta = \frac{c}{a}$ and $p=1$, and the objective Eq. \ref{equ:obj1} becomes:
\begin{equation}
    \begin{split}
        f(H)=\frac{1}{2}\cdot \mathrm{tr}\left( H^T \left( \Bar{A} - I_n \right) H \right) - \frac{1}{2} \cdot \frac{c}{a} \cdot \left \| H - X \right \|_F^2.
    \end{split}
\end{equation}
This equation is equivalent to Equ. \ref{equ_app:objective_function_R_Soft_GraphAIN} due to the condition that $H^TH=I_d$, \ie, $\mathrm{tr}\left( H^T H \right) = \left \| H \right \|_F^2=d$.

Note that the step size $\eta = a$, and the projection formulation of \textit{Residual-favored GraphAIN} is the same as \textit{GraphAIN} (see Theorem \ref{theorem: hard_GraphAIN_optimization}), which completes the proof.
\end{proof}

\begin{table*}[htbp]
\small
\centering
\begin{tabular}{c|c|c|c|c|c}
\toprule
\diagbox{Method}{Pubmed}  
& 2 & 8 & 16 & 32 & 64 \\ \midrule
GCN &  77.40±0.37 & 77.24±0.71 & 70.02±4.23 & 44.22±0.93 & 32.65±1.15 \\
SGC &  77.50±0.66 & 70.90±0.59 & 71.34±0.09 & 70.70±0.32 & 65.33±0.52 \\
ChebNet &78.53±0.34&	77.63±0.43&	72.87±0.57&	48.67±3.20&	45.37±0.57
\\
GAT &
77.62±0.10&	77.39±0.09&	76.28±0.15&	78.72±0.22&	41.88±1.78
\\
GIN &77.54±0.30&		75.84±1.18&	72.69±3.45&	72.94±4.37&	40.23±0.31
\\
PairNorm & 75.50±0.41 & 74.82±1.03 & 74.34±0.68 & 72.12±3.01 & 71.72±3.15 \\
GCNII &  78.05±1.53 & 79.34±0.51 & 80.03±0.50 & 79.81±0.27 & 79.88±0.17 \\
JKNet & 77.24±0.92 & 76.92±1.03 & 64.37±8.80 & 63.77±9.21 & 69.10±7.33 \\
GPRGNN & 78.73±0.63 & 78.90±0.47 & 78.78±1.02 & 78.46±1.03 & 78.92±1.45 \\
DAGNN & 77.74±0.57 & 79.68±0.37 & 80.32±0.38 & \underline{80.58±0.51} & 80.44±0.46 \\
APPNP & 79.46±0.47 & 80.02±0.30 & 80.30±0.30 & 80.24±0.33 & 80.08±0.35 \\
BernNer &
79.12±0.66&
78.32±0.79&
77.92±0.79&
70.24±10.07&
32.86±8.26
\\\midrule
Ours(GCN) &  \textbf{81.18±0.19} & \textbf{80.96±0.15} & \textbf{81.14±0.17} & \textbf{81.28±0.18} & \textbf{81.58±0.66} \\ 
Ours(GAT) &\underline{80.64±0.25}&	80.34±0.27&	80.16±0.18&	80.46±0.34&	80.20±0.24
\\
Ours(GIN) &81.28±0.15&		\underline{80.70±0.39}&	\underline{80.72±0.24}&	80.32±0.24&	\underline{80.94±0.36}
\\
\bottomrule
\end{tabular}
\caption{Accuracy comparison of node classification tasks on Pubmed}
\label{tab_appendix:comparative_experiments_pubmed}
\end{table*}

\section{Experimental Details}
\label{sec_appendix:experimental_details}
In this section,  details of more experiments were provided here to further evaluate the performance of our model as complementary to  the experiments in the main text.

\subsection{Dataset Descriptions}

Dataset statistics including their splits are summarized in Tab.~\ref{tab_appendix:dataset_statistics}.
In \cite{p52_deeper_gnn_tricks_overview}, a detailed introduction can be found including all twelve homogeneous and heterophilous graphs.


\subsection{More experimental results}
In this subsection, to further evaluate our model we report  more experimental results   including comparative experiments on homogeneous and heterophilous graphs with more layers, ablation studies with specific standard errors, and results in node classification tasks with noisy features.
Some results of layers $10^3$ and $10^4$ are also reported  for reference in tasks with noisy features (see next subsection for details), where our models are always implemented based on SGC.

\begin{itemize}
    \item Common semi-supervised node classification tasks on homogeneous graphs: Tab.~\ref{tab_appendix:comparative_experiments_cora}, Tab.~\ref{tab_appendix:comparative_experiments_citeseer},
    Tab.~\ref{tab_appendix:comparative_experiments_pubmed},
    Tab.~\ref{tab_appendix:comparative_experiments_CS_Physics},
    Tab.~\ref{tab_appendix:comparative_experiments_Computers_Photos};
    \item Common semi-supervised node classification tasks on heterophilous graphs: Tab.~\ref{tab_appendix:comparative_experiments_Texas_and_Wisconsin},
    Tab.~\ref{tab_appendix:comparative_experiments_Cornell_and_Actor};
    \item Ablation studies: Tab.~\ref{tab_appendix:ablation_studies_Cora_Citeseer_Pubmed}, Tab.~\ref{tab_appendix:ablation_studies_CS_Physics_Computers_and_Photo_1}, and Tab.~\ref{tab_appendix:ablation_studies_CS_Physics_Computers_and_Photo_2};
    \item Node classification tasks on a large-scale graph OGBN-ArXiv in its public fixed split: Tab.~\ref{tab_appendix:arxiv};
    \item Node classification tasks in fully supervised settings (average results with multiple random $60\%/20\%/20\%$ splits following~\cite{p39_GCNII}) on three citation graphs: Tab.~\ref{tab_appendix:multi_splits};
    \item Results of node classification tasks on heterophilous graphs compared to more state-of-the-art counterparts (including some models specifically designed for heterophilous graphs): Tab.~\ref{tab_appendix:more_heterophiliy};
    \item Node classification tasks with noisy features: Tab.~\ref{tab_appendix:noisy_features_Cora}, Tab.~\ref{tab_appendix:noisy_features_Citeseer}, Tab.~\ref{tab_appendix:noisy_features_Pubmed}.
\end{itemize}

\begin{table*}[htbp]
\small
\centering
\begin{tabular}{c|ccc|ccc}
\toprule
Method & \multicolumn{3}{c|}{CoauthorCS} & \multicolumn{3}{c}{CoauthorPhysics} \\ \cmidrule{1-7} 
 \#Layers & 16 & 32 & 64 & 16 & 32 & 64  \\ \midrule
GCN & 53.19±7.23 & 41.29±5.11 & 34.23±8.31 & 85.23±2.18 & 79.87±3.86 & 75.34±1.12  \\
SGC & 71.75±3.65 & 70.52±3.96 & 72.51±0.89 & 92.34±0.20 & 91.46±0.48 & 90.77±0.67  \\
ChebNet &48.43±8.82&	29.28±6.85&	23.24±0.00&77.62±2.09&	70.35±3.41&	50.74±0.00

\\
GAT &
85.32±0.22&	85.85±0.22&	12.83±2.58&	91.84±0.16&	91.87±0.11&	17.75±1.53
\\
GIN &70.77±3.02&	52.90±3.87&	20.79±5.21&
85.22±2,61&	83.02±2.53&	24.98±10.76

\\
PairNorm & 75.17±5.15 & 72.71±3.21 & 68.62±6.79 & 90.18±1.17 & 88.51±0.95 & 89.11±1.49 \\
GCNII & 58.94±2.63 & 71.67±2.68 & 72.11±4.19 & 92.13±1.31 & 93.15±0.92 & 92.79±0.52 \\
JKNet & 81.31±3.21 & 81.82±3.32 & 82.84±3.15 & 91.24±0.97 & 90.92±1.61 & 89.88±1.99  \\
GPRGNN & 89.39±0.39 & 89.56±0.47 & 89.33±0.62 & 93.64±0.31 & 93.49±0.59 & 93.26±0.46  \\
DAGNN & 91.13±0.50 & 89.60±0.71 & 89.47±0.68 & 93.77±0.29 & 93.31±0.60 & 93.52±0.32 \\
APPNP & 91.64±0.53 & 91.61±0.49 & 91.58±0.36 & 93.96±0.36 & 93.75±0.61 & 91.61±0.33  \\
BernNet &
91.53±0.57&
91.54±0.24&
9.19±7.43&
91.51±0.76&
92.67±0.63&
19.38±0.99
\\\midrule
Ours(GCN) & \textbf{92.47±0.07} & \textbf{92.50±0.09} & \textbf{92.41±0.07} & \textbf{94.51±0.04} & \underline{94.45±0.06} & \underline{94.52±0.04}  \\ 
Ours(GAT) &91.32±0.06&	91.30±0.07&	91.26±0.03&
94.00±0.08&	94.02±0.05&	94.02±0.12

\\
Ours(GIN) &\underline{92.06±0.21}&	\underline{91.79±0.06}&	\underline{91.71±0.04}&
\underline{94.44±0.04}&	\textbf{94.56±0.12}&	\textbf{94.53±0.16}

\\
\bottomrule
\end{tabular}
\caption{Node classification accuracy comparison on CS and Physics}
\label{tab_appendix:comparative_experiments_CS_Physics}
\end{table*}

\subsection{Implementation details}
In this subsection, we give some details on the implementations of our models against different tasks.

For seven homogeneous graphs in common semi-supervised node classification tasks, we adopt Our model itself based on GCN as a teacher model for fair comparison (\ie, we don't use other pre-trained models) and we build the auxiliary graphs using $G_{aux} = G_{emd}$ because we think $G_{emd}$ can be a rough estimation of similarities of ground-truth labels of those nodes in the input graph.

For four heterophilous graphs, we adopt different implementations for adaptation to these specific scenes.
Specifically, we build and aggregate in the auxiliary graphs $G_{aux}=G_{fea}$ and also employ Ours as the teacher to teach Ours itself for fairness.
We choose $G_{fea}$ instead of $G_{emd}$ because we assume that on these heterophilous graphs, features are information more related to ground-truth labels of nodes compared to their noisy structures.

Additionally, for node classification tasks with noisy features, we maintain the basic settings in the common semi-supervised node classification tasks (\eg, we use the same splits for all benchmarks).
We employ $G_{ori}$ as the auxiliary graph and also use Ours as a teacher similar to the other scenes.
But here we utilize SGC-based implementations for our models because we hope to conveniently implement models with $10^3$ and $10^4$ layers.
More specifically, we first use a simple linear layer for efficient dimension reduction and use the non-parametric SGC with $L$ layers for propagation and finally use $k$-layer MLP to do feature transformation due to limited non-linearity where $k \in [l,5]$ is a hyper-parameter chose via their validation sets.
For layers no more than $64$, we directly train Our model in an end-to-end manner.

But for deeper layers (\eg, $10^3$ and $10^4$ layers), we tend to stop the gradient and avoid building the computation graphs.
Thus this trick can efficiently reduce the time for training and the out-of-memory issue, though the expressivity is restricted.
In other words, it's a trade-off between  runtime, space complexity, and nonlinearity.
So to supplement non-linearity, we tend to use an MLP instead of a simple linear layer.

We employ a $3090$ GPU (with $24$G Memory) to run other models due to possible OOM issues while a $2080$Ti GPU (with $12$G Memory) for ours.
\textit{Pairnorm} means GCN-based Pairnorm \cite{p53_pairnorm}.
For layers more than $64$ (\ie, $10^3$ and $10^4$ layers), the motivation for them is as follows: 1) to show SGC can unavoidably suffer from over-smoothness when the layers are deep enough; 2) to show our model will not suffer from over-smoothness even with a sufficiently large depth, conforming to the theoretical analysis. 

Because of the high time and space complexity  for our model with $10^3$ layers and above, we do not give the comparison results for the case. Anyhow, by compromising the  performance and reducing the time and space complexity, we present those comparison results in the scenarios with noisy features at  $10^3$  and  $10^4$ layers.  We trust those comparisons  can similarly  carry over to the non-noisy scenarios and achieve even better accuracy.

\begin{table*}[htbp]
\small
\centering
\begin{tabular}{c|ccc|ccc}
\toprule
Method &  \multicolumn{3}{c|}{AmazonComputers} & \multicolumn{3}{c}{AmazonPhoto} \\ \cmidrule{1-7} 
\#Layers & 16 & 32 & 64 & 16 & 32 & 64  \\ \midrule
GCN &  64.02±2.38 & 58.30±3.35 & 37.58±0.05 & 70.81±3.48 & 58.47±8.80 & 50.21±3.99 \\
SGC &  37.48±0.07 & 37.44±0.12 & 37.50±0.08 & 35.64±6.11 & 26.08±1.39 & 24.57±1.32 \\
ChebNet &
65.42±4,04&	58.58±3.29&	50.12±2.15&
69.63±6.90&	65.28±3.55&	64,83±0.72

\\
GAT &
76.90±0.49&	76.05±0.20&	37.18±0.53&	89.27±0.29&	83.73±0.74&	 25.36±0.20
\\
GIN &
71.91±2.71&	39.18±3.77&	37.50±0.09&
82.25±2.04&	65.98±2.14&	25.27±0.13

\\
PairNorm &  77.41±1.85 & 74.96±2.09 & 74.35±1.82 & 82.72±1.19 & 82.66±2.47 & 79.55±2.51 \\
GCNII &  38.88±4.26 & 37.56±0.43 & 37.50±0.08 & 68.37±6.61 & 62.95±9.41 & 65.12±2.86 \\
JKNet &  60.76±5.10 & 67.99±5.07 & 67.78±4.79 & 74.86±8.45 & 78.42±6.95 & 79.73±7.26 \\
GPRGNN &  76.07±2.28 & 41.94±9.95 & 78.30±2.51 & 91.55±0.43 & 91.74±0.81 &91.28±0.88 \\
DAGNN &  80.33±1.04 & 79.73±3.63 & 79.23±2.36 & 90.81±0.59 & 89.96±1.16 & 87.86±0.70 \\
APPNP &  39.41±5.80 & 43.02±10.16 & 41.42±7.50 & 64.59±20.09 & 59.62±23.27 & 63.63±19.28 \\
BernNet &
84.23±1.62&
81.86±1.63&
12.81±12.10&
89.17±1.91&
91.59±0.15&
16.53±5.78
\\\midrule
Ours(GCN) &  \underline{85.24±0.05} & \textbf{85.21±0.11} & \textbf{85.13±0.08} & 91.98±0.09 & \textbf{92.06±0.19} & \textbf{92.03±0.12} \\ 
Ours(GAT) &
84.88±0.14&	84.82±0.27&	\underline{85.04±0.15}&
\textbf{92.04±0.07}&	91.98±0.07&	\underline{92.00±0.05}

\\
Ours(GIN) &
\textbf{85.30±0.13}&	\underline{85.16±0.12}&	\textbf{85.13±0.08}&
\underline{92.00±0.03}&	\underline{92.03±0.05}&	91.98±0.04
\\
\bottomrule
\end{tabular}
\caption{  Node classification accuracy comparison on Computers and Photo}
\label{tab_appendix:comparative_experiments_Computers_Photos}
\end{table*}

\subsection{Hyper-parameter Configurations}

It's difficult to well-tune all the hyper-parameters simultaneously, so we do it  in three stages: 1)    Tune hyper-parameters in \textit{R-SoftGraphAIN} but with all other hyper-parameters fixed; 2)  Tune hyper-parameters in \textit{Curriculum Learning} but with all other hyper-parameters fixed; 3) fine-tune other hyper-parameters excepting that of  \textit{R-SoftGraphAIN} and  \textit{Curriculum Learning} (\eg, learning rate, dropout rate, and weight decay).

In every stage above, we choose the hyper-parameters according to their best validation performance.
The search spaces of all meaningful hyper-parameters are listed in Tab.~\ref{tab_appendix:hyper-parameters_search_spaces}, and we omit some unimportant ones because they are actually robust to model performance.
For every kind of experiment, we search in the space randomly for $500$ times in total (\ie, the sum of all layers).
Additionally, we report all hyper-parameters configurations used by \textit{Ours(GCN)} for comparative experiments in Tab.~\ref{tab_appendix:Hyper_parameters_configurations}.
Note that the exactly same hyper-parameters are employed for \textit{Ours(GAT)} and \textit{Ours(GIN)} to reduce the burden of computational resources.
In other words, with further appropriate hyper-parameter tuning, these two versions might have more outstanding performance, possibly due to the fact that the encoders (\ie, GAT and GIN) are more powerful than GCN itself.

In addition, note that in our experiments, we do not directly use the hyper-parameters $\alpha, \beta, \gamma$ in Tab.~\ref{tab_appendix:Hyper_parameters_configurations}.
In fact, we use their values after normalization, \ie, letting $\alpha = \alpha / \left( \alpha + \beta + \gamma \right) \in [0,1]$ in order to satisfy the rule $\alpha + \beta + \gamma = 1$ in the main text (see Sec.~\ref{sec:Additional_Residual_Combination}).

\section{Extensive Visualizations}
\label{sec_appendix:more_visualization}

In this section, we plot the original features, embeddings output by SGC, and embeddings output by our model on eleven real-world graph benchmarks respectively in Fig.~\ref{fig_appendix:embeddings_cora_citeseer_pubmed}, Fig.~\ref{fig_appendix:embeddings_CS_Physics_Computers_Photo}, and Fig.~\ref{fig_appendix:embeddings_texas_wisconsin_cornell_actor}.
These demonstrate that our model obtains clearer clustering boundaries compared to SGC, which conforms to the fact that our model achieves a performance gain in comparison.

Moreover, Fig.~\ref{fig_appendix:vis_comp1} gives  visualizations of the embeddings produced by different deep models with  synthetic data that is generated as follows:  first choose several cluster centers $\{c_i\}$ in $\mathbb{R}^{3}$, then sample points   in $\mathcal{N}\left( c_i, I \right)$, where colors and edges are created following Bernoulli distributions to ensure that points in the same class tend to share a link and the same color.
These visualizations show that our model obtains clearer decision boundaries  compared to the baseline models including SGC and Pairnorm,  conforming to the performance gain of our model.

\section{Discussion on the Time and Space Complexities}
\label{sec_appendix:time_and_space_complexities}
The theoretical time complexity is analysed $O(L d^2 d_0)$ where $d,d_0 \ll n$ if partial-SVD or truncated-SVD~\cite{p117_TruncatedSVD} is utilized.
And it would become $O(L d^3)$ with a full SVD decomposition.
However, it is efficient under the high-parallelizability implemented via Pytorch.
The theoretical space complexity is $O\left(L\left(n+d\right)d\right)$.

\section{Additional Hyper-parameter Studies}
\label{sec_appendix:Additional_Hyper-parameter_Studies}
In this section, to help readers better understand the properties of our method, we further conduct some additional hyper-parameter studies on the hyper-parameter $\alpha$, $\gamma$, and $k_{KNN}$, which we think are somewhat important in our framework.
We study these hyper-parameters based on all versions, \ie, Ours(GCN), Ours(GAT), and Ours(GIN).
We visualize the line graph considering the effects of them on the final results (see Fig.~\ref{app_fig:hyper-parameter_gamma} and Fig.~\ref{app_fig:hyper-parameter_k}).

Fig.~\ref{app_fig:hyper-parameter_gamma} illustrates how the trade-off between residual connections and initial connections influences the final performance, considering fixing the hyper-parameter $\beta$.
Note that studying the hyper-parameter $\gamma$ is equivalent to studying the hyper-parameter $\alpha$ in this situation.
From this figure, we can see that in some datasets, $\gamma$ should not be too small; otherwise, the performance would decline or fluctuate.
However, a too-high value of $\gamma$ also hurt the performance in some cases (\eg, nearly all versions on Cora and Ours(GAT) on Pubmed).
Therefore, we had better keep its value ranging between $0.3$ and $0.7$ on Cora, Citeseer, and Pubmed, because it seems somewhat stable in this range.
And definitely, we can specifically and carefully tune it according to the validation sets of all these datasets.

From Fig.~\ref{app_fig:hyper-parameter_k}, we can clearly observe a relatively stable line trend between the hyper-parameter $k_{KNN}$ and the model performance.
In most cases, we think the range $[5,7]$ is quite appropriate despite some exceptions.
However, according to this line graph, we generally think that it is not necessary to carefully select or tune this hyper-parameter, though one can definitely do it.


\section{Discussion On The Connections Between R-SoftGraphAIN and SmoothCurriculum}
\label{sec_appendix:connection_between_two_parts}
In this section, to help readers better understand our method, we briefly discuss the connection of our proposed two parts, \ie, \textit{R-SoftGraphAIN} and \textit{SmoothCurriculum}.
At first glimpse, the proposed two parts have little essential relationship, because both of them can be applied to different graph encoders to improve their performance.
But under the intrinsic design and basic motivation on facilitating a better deep GNN model, they can be parts that get along with each other quite well, and thus can be viewed as a whole part.
Specifically, one can easily notice that, according to its properties, R-SoftGraphAIN can preserve knowledge into $d$ (or $d_0$, more accurately) different channels and organize them well, \eg, the information in different channels would not intervene with each other.
With the aid of it, GNN encoders can definitely improve performance when layers go deep.
However, without the help of curriculum learning, the supervised knowledge (\ie, label information) absorbed by this module keeps constantly the same.
Therefore, it would organize this knowledge adaptively and in its own way, which is not quite understandable and depends on both the feature and the input graph structure.
But under the guide of the proposed curriculum learning framework, it will be gradually fed with the knowledge that is from easy to hard, from low-frequency to high-frequency, from global to local, and from generalizable to relatively noisy or unnecessary.
So fortunately, it has some opportunities to load different knowledge hierarchies into different channels and intuitively avoid some unnecessary forgetting.
For example, it could employ the first several channels to differentiate relatively larger communities, and use additional channels to more accurately locate neighborhoods and nodes themselves.
Besides, to reduce the negative influence of noise, it can stop learning early when necessary.
On the other hand, if using the curriculum learning framework individually and, \eg, applying it to GCN, GCN might have only one channel, and thus the new knowledge might replace the old one and the model tends to forget some basic yet important and generalizable label information, assuming the graph is connected.
But with the help of R-SoftGraphAIN, it might help organize this information well.
That's why we think these two parts can be viewed as a whole to significantly improve the model's performance.

\begin{table*}[t]
\small
\centering
\begin{tabular}{c|ccc|ccc}
\toprule
Method & \multicolumn{3}{c|}{Texas} & \multicolumn{3}{c}{Wisconsin} \\ \cmidrule{1-7} 
\#Layers & 16 & 32 & 64 & 16 & 32 & 64  \\ \midrule
GCN & 62.16±2.70 & 62.16±2.70 & 62.16±2.70 & 57.84±4.90 & 57.84±4.90 & 57.84±4.90  \\
SGC & 62.03±0.59 & 56.41±4.25 & 56.96±5.28 & 52.25±7.19 & 51.29±6.44 & 52.16±7.08  \\
ChebNet &
69.37±1.56&	
64.86±0.00&	
64.86±0.00&
71.90±3.00&
52.94±0.00&
52.94±0.00

\\
GAT &
64.32±2.96&	65.41±1.21&	64.86±0.00&	53.33±1.64&	53.73±1.07&	53.33±0.88
\\
GIN &
52.97±6.78&	62.17±4.44&	60.00±4.01&
47.06±3.67&	47.06±2.40&	50.98±4.38

\\
PairNorm & 32.70±17.41 & 41.08±18.04 & 40.68±16.71 & 44.51±11.33 & 52.84±8.97 & 52.94±11.35 \\
GCNII & 69.59±6.10 & 69.19±6.56 & 65.41±2.02 & 71.86±1.89 & 70.31±4.75 & 59.02±0.78  \\
JKNet & 65.41±3.03 & 61.08±6.23 & 66.49±2.76 & 55.98±3.53 & 52.76±5.69 & 56.08±4.04  \\
GPRGNN & 62.70±2.65 & 62.27±4.97 & 61.08±1.99 & 67.94±3.58 & 71.35±5.56 & 64.90±2.77  \\
DAGNN & 61.08±1.99 & 57.68±5.07 & 60.27±2.43 & 55.49±3.78 & 50.84±6.62 & 51.76±4.78  \\
APPNP & 64.46±3.11 & 60.68±4.50 & 64.32±2.78 & 59.31±3.71 & 54.24±5.94 & 59.90±3.00  \\
BernNet &
76.22±2.02&
61.08±1.32&
16.76±7.33&
\underline{80.00±2.88}&
63.53±4.74&
24.71±7.60
\\\midrule
Ours(GCN) & \textbf{85.95±1.21} & \textbf{85.41±1.48} & \textbf{84.86±1.48} & \textbf{82.35±1.39} & \textbf{83.14±1.75} & \textbf{84.71±0.88}  \\ 
Ours(GAT) &
79.46±1.48&	78.92±4.83&	78.38±1.91&
72.94±4.25&	73.73±2.97&	74.90±1.64

\\
Ours(GIN) &
\underline{82.70±1.58}&	\underline{82.57±1.48}&	\underline{83.12±1.33}&
\textbf{82.35±1.39}&	\underline{81.78±2.15}&	\underline{81.20±1.21}
\\
\bottomrule
\end{tabular}
\caption{Node classification accuracy comparison on heterophilous graphs: Texas and Wisconsin}
\label{tab_appendix:comparative_experiments_Texas_and_Wisconsin}
\end{table*}

\begin{table*}[t]
\small
\centering
\begin{tabular}{c|ccc|ccc}
\toprule
Method & \multicolumn{3}{c|}{Cornell} & \multicolumn{3}{c}{Actor} \\ \cmidrule{1-7} 
\#Layers & 16 & 32 & 64 & 16 & 32 & 64  \\ \midrule
GCN &  56.76±2.70 & 56.76±2.70 & 56.76±2.70 & 25.16±0.30 & 25.16±0.30 & 25.16±0.30 \\
SGC &  55.41±1.35 & 58.57±3.44 & 55.41±1.35 & 25.88±0.49 & 26.17±1.15 & 25.88±0.49 \\
ChebNet &
55.86±1.56&	55.86±1.56&	52.25±1.56&
34.71±0.11&
25.46±0.00&
25.46±0.00
\\
GAT &
53.51±2.26&	54.05±0.00&	55.14±1.48&	26.70±0.13&	25.54±0.09&	25.62±0.13
\\
GIN &
55.68±2.42&	54.59±1.21&	55.14±1.48&
24.78±1.08&	24.36±2,48&	23.45±2.73
\\
PairNorm &  35.95±16.77 & 36.89±18.63 & 40.68±12.89 & 22.84±2.29 & 24.33±1.60 & 23.23±2.84 \\
GCNII &  66.49±5.88 & 74.16±6.48 & 56.92±2.02 & 33.79±0.68 & 34.28±1.12 & 34.64±0.71 \\
JKNet &  50.27±3.95 & 57.30±4.95 & 51.49±4.40 & 29.23±0.65 & 28.80±0.97 & 28.26±0.43 \\
GPRGNN &  53.11±3.45 & 58.27±3.96 & 52.16±3.74 & 32.83±0.88 & 29.88±1.82 & 32.43±0.49 \\
DAGNN &  55.27±2.49 & 58.43±3.93 & 52.43±5.01 & 27.66±0.67 & 27.73±1.08 & 25.45±0.64 \\
APPNP &  56.35±2.46 & 58.43±3.74 & 54.69±2.55 & 28.38±0.79 & 28.65±1.28 & 28.19±0.97 \\
BernNet &
72.97±1.71&
52.43±9.46&
11.89±1.32&
\underline{36.31±0.39}&
28.19±1.10&
20.66±3.10
\\\midrule
Ours(GCN)&  \textbf{82.16±1.48} & \textbf{82.70±1.48} & \textbf{82.62±1.21} & \textbf{38.16±0.15} & \textbf{38.42±0.23} & \textbf{38.49±0.34} \\ 
Ours(GAT) &
\underline{81.62±2.26}&	77.84±2.26&	\underline{81.62±1.21}&
35.04±0.42&	\underline{34.97±0.34}&	\underline{35.54±0.61}
\\
Ours(GIN) &
80.00±1.48&	\underline{81.08±1.91}&	81.08±1.91&
34.46±0.51&	34.51±0.27&	34.50±0.61
\\
\bottomrule
\end{tabular}
\caption{  Node classification accuracy comparison on heterophilous graphs:  Cornell and Actor}
\label{tab_appendix:comparative_experiments_Cornell_and_Actor}
\end{table*}

\begin{table*}[t]
\small
\centering
\begin{tabular}{c|cc|cc|cc}
\toprule
Method & \multicolumn{2}{c|}{Cora} & \multicolumn{2}{c|}{Citeseer} & \multicolumn{2}{c}{Pubmed} \\ \cmidrule{1-7} 
 \#Layers & 32 & 64 & 32 & 64 & 32 & 64 \\ \midrule
w.o. SG & 83.30±0.12 & \underline{84.60±0.29} & 72.98±0.50 & 72.62±0.98 & 80.26±0.98 & 80.30±0.23 \\
w.o. RC & 82.88±0.79 & 82.40±0.33 & 72.82±1.01 & 71.04±0.55 & 78.84±0.59 & 78.60±0.27 \\
w.o. R-SG & 40.22±5.71 & 36.74±4.36 & 29.32±3.47 & 27.70±2.94 & 46.28±4.91 & 28.61±3.64 \\
w.o. LS & \underline{83.96±0.65} & 84.00±0.22 & \underline{73.64±0.68} & \underline{74.34±0.79} & \underline{81.04±0.21} & \underline{80.86±0.79} \\
w.o. CL &  82.40±0.65 & 82.58±0.87 & 73.20±0.77 & 72.54±0.42 & 79.34±0.39 & 79.00±0.54 \\
Ours(GCN) & \textbf{85.12±0.15} & \textbf{84.87±0.26} & \textbf{74.42±0.26} & \textbf{74.50±0.23} & \textbf{81.28±0.18} & \textbf{81.58±0.66} \\ \bottomrule
\end{tabular}
\caption{Ablation Studies on Cora, Citeseer, and Pubmed}
\label{tab_appendix:ablation_studies_Cora_Citeseer_Pubmed}
\end{table*}


\begin{table*}[t]
\small
\centering
\begin{tabular}{c|cc|cc}
\toprule
Method & \multicolumn{2}{c|}{CoauthorCS} & \multicolumn{2}{c}{CoauthorPhysics} \\ \cmidrule{1-5} 
\#Layers & 32 & 64 & 32 & 64 \\ \midrule
w.o. SG & 91.79±0.07 & 91.70±0.07 & \underline{94.23±0.25} & \underline{94.25±0.13} \\
w.o. RC & 91.40±0.26 & 89.03±0.62 & 92.99±0.19 & 92.59±0.63 \\
w.o. R-SG & 51.61±18.23 &	28.51±4.70 &	77.20±11.95 &	60.94±8.23\\
w.o. LS & \underline{92.02±0.05} & \underline{91.95±0.08} & 94.10±0.11 & 94.06±0.08 \\
w.o. CL & 89.60±0.20 & 89.23±0.12 & 92.46±0.61 & 92.22±0.56
 \\
Ours(GCN) & \textbf{92.50±0.09} & \textbf{92.41±0.07} & \textbf{94.45±0.06} & \textbf{94.52±0.04} \\ \bottomrule
\end{tabular}
\caption{Ablation Studies on CS aand Physics}
\label{tab_appendix:ablation_studies_CS_Physics_Computers_and_Photo_1}
\end{table*}

\begin{table*}[t]
\small
\centering
\begin{tabular}{c|cc|cc}
\toprule
Method & \multicolumn{2}{c|}{AmazonComputers} & \multicolumn{2}{c}{AmazonPhoto} \\ \cmidrule{1-5} 
\#Layers & 32 & 64 & 32 & 64 \\ \midrule
w.o. SG & 82.34±0.67 & 84.50±0.28 & 90.26±1.57 & \underline{91.69±0.37} \\
w.o. RC & 83.78±0.52 & 84.25±0.37 & 91.42±0.32 & 91.14±0.48 \\
w.o. R-SG & 61.52±8.27 &	49.94±9.69 &	83.77±9.52 &	70.09±12.39 \\
w.o. LS & \underline{84.52±0.93} & \underline{84.91±0.16} & \underline{91.50±0.30} & 91.64±0.40 \\
w.o. CL & 84.00±0.97 & 84.44±0.56 & 90.58±0.94 & 90.41±1.07
 \\
Ours(GCN) & \textbf{85.21±0.11} & \textbf{85.13±0.08} & \textbf{92.06±0.19} & \textbf{92.03±0.12} \\ \bottomrule
\end{tabular}
\caption{Ablation Studies on Computers and Photo}
\label{tab_appendix:ablation_studies_CS_Physics_Computers_and_Photo_2}
\end{table*}


\begin{table*}[t]
\centering
\begin{tabular}{c|cc}
\toprule
Method & \multicolumn{2}{c}{OGBN-ArXiv} \\ \cmidrule{1-3} 
\#Layers & 32 & 64 \\ \midrule
GCN & 46.38±3.87 & 42.95±3.93 \\
SGC & 34.22±0.04 & 23.14±7.85 \\
ChebNet &
41.00±1.59&	35.16±2.23
\\
GAT &
59.78±0.25&	36.63±1.71
\\
GIN &
65.17±0.31&	60.56±0.57
\\
PairNorm & 63.32±0.97 & 43.57±1.24 \\
GCNII & 72.60±0.25 & 70.07±0.14 \\
JKNet & 66.31±0.63 & 65.80±0.27 \\
GPRGNN & 70.18±0.16 & 69.98±0.15 \\
DAGNN & 71.46±0.27 & 70.58±0.12 \\
APPNP & 66.94±0.26 & 66.90±0.15 \\ 
BernNet &
45.16±0.33&	37.18±2.52
\\\midrule
Ours(GCN) & 74.07±0.08 & 73.95±0.06 \\ 
Ours(GAT) &
\underline{74.37±0.07}&	\underline{74.41±0.05}
\\
Ours(GIN) &
\textbf{75.02±0.12}&	\textbf{74.85±0.09}
\\
\bottomrule
\end{tabular}
\caption{Node classification accuracy comparison on the large-scale graph: OGBN-ArXiv}
\label{tab_appendix:arxiv}
\end{table*}

\begin{table*}[t]
\centering
\begin{tabular}{c|ccc}
\toprule
 &  &  &  \\
\multirow{-2}{*}{Method} & \multirow{-2}{*}{Cora} & \multirow{-2}{*}{Citeseer} & \multirow{-2}{*}{Pubmed} \\ \midrule
GCN & 52.53±8.94 & 29.98±6.25 & 48.75±9.33 \\
SGC & 76.16±0.27 & 69.62±0.63 & 78.09±0.51 \\
PairNorm & 80.50±0.19 & 63.34±5.93 & 80.55±3.16 \\
GCNII & {\color[HTML]{282828} 88.08±0.25} & {\color[HTML]{282828} 75.47±0.32} & {\color[HTML]{282828} 89.53±0.16} \\
JKNet & 85.05±0.47 & 74.15±0.70 & 89.25±0.39 \\
GPRGNN & 87.47±0.35 & 74.29±0.48 & \underline{90.61±0.12} \\
DAGNN & 87.67±0.39 & 73.82±0.49 & 87.21±0.13 \\
APPNP & {\color[HTML]{282828} \underline{88.28±0.63}} & {\color[HTML]{282828} \underline{75.76±0.51}} & {\color[HTML]{282828} 89.05±0.11} \\ 
Ours(GCN) & {\color[HTML]{282828} \textbf{90.16±0.33}} & {\color[HTML]{282828} \textbf{77.23±0.35}} & {\color[HTML]{282828} \textbf{91.01±0.12}} \\ \bottomrule
\end{tabular}
\caption{Results comparison of node classification tasks in multiple random splits ($60\%/20\%/20\%$) on three citation graphs}
\label{tab_appendix:multi_splits}
\end{table*}

\begin{table*}[t]
\centering
\begin{tabular}{c|cccc}
\toprule
Method & Texas & Wisconsin & Cornell & Actor \\ \midrule
GAT~\cite{p8_GAT} & 61.62 & 54.71 & 59.46 & 28.06 \\
{\color[HTML]{282828} Gemo-GCN~\cite{p111_Geom-GCN}} & 67.57 & 64.12 & 60.81 & 31.63 \\
{\color[HTML]{282828} H2GCN~\cite{p100_Beyond_homophily}} & 84.86 & 86.67 & 82.16 & 35.86 \\
FAGCN~\cite{p112_FAGCN} & 84.32 & 83.33 & 81.35 & 35.74 \\
GCA~\cite{p113_GCA} & 59.46 & 50.78 & 55.41 & 29.65 \\
BGRL~\cite{p114_BGRL} & 59.19 & 52.35 & 57.30 & 29.86 \\
VGAE~\cite{p115_VGAE} & 59.20 & 56.67 & 59.19 & 26.99 \\
FSGNN*~\cite{p116_FSGNN} &  \underline{87.30} & \textbf{88.43} & \underline{88.11} & 35.75 \\ \midrule
Ours(GCN) & 85.95 & 84.71 & 82.70 & \textbf{38.49} \\
Ours(GCN)* & \textbf{94.74} & \underline{88.24} & \textbf{91.05} & \underline{37.73} \\ \bottomrule
\end{tabular}
\caption{More result comparison on these heterophilous graphs (models with the suffix $*$ are evaluated in multiple random splits, \ie, $60\%/20\%/20\%$)}
\label{tab_appendix:more_heterophiliy}
\end{table*}

\begin{table*}[t]
\small
\centering
\begin{tabular}{c|cccccccc|cc}
\toprule
\diagbox{Method}{Cora}   & 2 & 4 & 8 & 12 & 16 & 24 & 32 & 64 & 1000 & 10000 \\ \midrule
GCN & 54.8 & \underline{66.1} & \underline{72.2} & 70.2 & 65.6 & 56.6 & 56.0 & 31.9 & 31.9 & 31.9 \\
SGC & 55.0 & 65.2 & 72.1 & \underline{75.1} & \underline{75.0} & \underline{77.0} & \underline{75.6} & \underline{75.6} & 16.8 & 15.7 \\
GCNII & 40.7 & 49.4 & 60.1 & 65.8 & 62.4 & 62.6 & 65.9 & 62.4 & \underline{65.7} & \underline{63.4} \\
Pairnorm & \underline{56.6} & 60.6 & 63.2 & 67.9 & 70.5 & 70.8 & 61.9 & 35.5 & 54.4 & 31.8 \\
Ours(GCN) & \textbf{69.9} & \textbf{72.4} & \textbf{74.4} & \textbf{78.3} & \textbf{78.8} & \textbf{78.5} & \textbf{77.8} & \textbf{76.7} & \textbf{76.6} & \textbf{75.0} \\ \bottomrule
\end{tabular}
\caption{Accuracy of different models with varying layers in node classification tasks with noisy features on Cora}
\label{tab_appendix:noisy_features_Cora}
\end{table*}

\begin{table*}[t]
\small
\centering
\begin{tabular}{c|cccccccc|cc}
\toprule
\diagbox{Method}{Citeseer}  & 2 & 4 & 8 & 12 & 16 & 24 & 32 & 64 & 1000 & 10000 \\ \midrule
GCN & 36.1 & \underline{45.4} & 49.4 & 48.4 & 47.6 & 43.3 & 43.8 & 23.3 & 18.1 & 18.1 \\
SGC & \underline{37.9} & 44.0 & \underline{51.1} & \textbf{52.3} & \underline{54.0} & \underline{53.0} & \underline{55.2} & \textbf{56.0} & \underline{47.0} & 22.4 \\
GCNII & 32.8 & 33.5 & 34.7 & 43.4 & 45.1 & 42.8 & 45.9 & \underline{46.8} & 48.5 & \underline{46.8} \\
Pairnorm & 37.5 & 43.4 & 48.2 & 47.7 & 46.9 & 44.3 & 38.9 & 34.8 & 34.4 & 34.4 \\
Ours(GCN) & \textbf{46.5} & \textbf{49.4} & \textbf{53.1} & \underline{51.8} & \textbf{54.6} & \textbf{55.1} & \textbf{55.4} & \underline{55.7} & \textbf{55.0} & \textbf{53.5} \\ \bottomrule
\end{tabular}
\caption{Accuracy of different models with varying layers in node classification tasks with noisy features on Citeseer}
\label{tab_appendix:noisy_features_Citeseer}
\end{table*}

\begin{table*}[t]
\small
\centering
\begin{tabular}{c|cccccccc|cc}
\toprule
\diagbox{Method}{Pubmed} & 2 & 4 & 8 & 12 & 16 & 24 & 32 & 64 & 1000 & 10000 \\ \midrule
GCN & 40.4 & 48.0 & 57.3 & 61.2 & 60.4 & 59.4 & 57.3 & 44.8 & 40.7 & OOM \\
SGC & 40.6 & 47.6 & \underline{62.3} & \underline{69.2} & \underline{71.2} & \underline{74.6} & \underline{74.9} & \underline{78.6} & 34.1 & \underline{34.1} \\
GCNII & 39.7 & 41.3 & 41.2 & 40.6 & 40.3 & 40.0 & 41.0 & 39.7 & 39.3 & OOM \\
Pairnorm & \underline{41.3} & \underline{48.8} & 61.4 & 65.7 & 68.1 & 67.4 & 65.3 & 66.6 & \underline{65.7} & OOM \\
Ours(GCN) & \textbf{45.7} & \textbf{52.7} & \textbf{70.1} & \textbf{69.8} & \textbf{73.9} & \textbf{76.2} & \textbf{78.0} & \textbf{78.5} & \textbf{71.0} & \textbf{71.0} \\ \bottomrule
\end{tabular}
\caption{Accuracy of different models with varying layers in node classification tasks with noisy features on Pubmed}
\label{tab_appendix:noisy_features_Pubmed}
\end{table*}

\begin{sidewaystable*}[t]
\small
\centering
\begin{tabular}{c|c}
\toprule
Dataset & Hyper-parameters Configurations \\ \midrule
Cora 
& 
\makecell[c]
{
$d=64$,
$\alpha = 0.9$,
$\beta = 0.2$,
$\gamma=0.5$,
$a=0.01$,
$b=1$,
$n_T=10$,
$p=0.4$,
$q=0.3$,
\\
$\mathrm{mask\_ratio}=10^{-4}$,
$K_{knn}=7$,
$\gamma'=0.1$,
$d_0/d=0.99$,
$\mathrm{epoch}=300$,
\\
$\mathrm{lr=0.01}$,
$\mathrm{lr\_decay}=100$,
$\mathrm{dropout}=0.7$,
$\mathrm{weight\_decay}= 10^{-4}$
} 
\\ 
\midrule
Citeseer
&
\makecell[c]{
$d=256$,
$\alpha = 0.3$,
$\beta = 0.4$,
$\gamma=5$,
$a=0.6$,
$b=0.9$,
$n_T=100$,
$p=0.6$,
$q=0.5$,
\\
$\mathrm{mask\_ratio}=10^{-5}$,
$K_{knn}=4$,
$\gamma'=0.7$,
$d_0/d=0.8$,
$\mathrm{epoch}=300$,
\\
$\mathrm{lr=0.05}$,
$\mathrm{lr\_decay}=150$,
$\mathrm{dropout}=0.6$,
$\mathrm{weight\_decay}= 10^{-4}$
}
\\
\midrule
Pubmed 
&
\makecell[c]{
$d=64$,
$\alpha = 0.3$,
$\beta = 0.5$,
$\gamma=5$,
$a=0.8$,
$b=0.95$,
$n_T=20$,
$p=0.99$,
$q=0.3$,
\\
$\mathrm{mask\_ratio}=10^{-5}$,
$K_{knn}=7$,
$\gamma'=1$,
$d_0/d=0.8$,
$\mathrm{epoch}=300$,
\\
$\mathrm{lr=0.05}$,
$\mathrm{lr\_decay}=250$,
$\mathrm{dropout}=0.5$,
$\mathrm{weight\_decay}=5\times 10^{-6}$
}
\\
\midrule
Coauthor CS 
&
\makecell[c]{
$d=64$,
$\alpha = 0.01$,
$\beta = 0.5$,
$\gamma=5$,
$a=0.7$,
$b=0.8$,
$n_T=10$,
$p=0.9$,
$q=0.6$,
\\
$\mathrm{mask\_ratio}=0.3$,
$K_{knn}=2$,
$\gamma'=0.01$,
$d_0/d=0.95$,
$\mathrm{epoch}=300$,
\\
$\mathrm{lr=0.1}$,
$\mathrm{lr\_decay}=200$,
$\mathrm{dropout}=0.2$,
$\mathrm{weight\_decay}=5\times 10^{-5}$,
}
\\ 
\midrule
Coauthor Physics 
&
\makecell[c]{
$d=128$,
$\alpha = 0.2$,
$\beta = 0.4$,
$\gamma=10$,
$a=0.8$,
$b=0.8$,
$n_T=50$,
$p=0$,
$q=0.8$,
\\
$\mathrm{mask\_ratio}=0.4$,
$K_{knn}=7$,
$\gamma'=0.3$,
$d_0/d=0.96$,
$\mathrm{epoch}=300$,
\\
$\mathrm{lr=0.01}$,
$\mathrm{lr\_decay}=50$,
$\mathrm{dropout}=0.4$,
$\mathrm{weight\_decay}=10^{-6}$,
}
\\ 
\midrule
Amazon Photo 
&
\makecell[c]{
$d=256$,
$\alpha = 0.01$,
$\beta = 0.7$,
$\gamma=0.5$,
$a=0.3$,
$b=0.85$,
$n_T=5$,
$p=0.6$,
$q=0.8$,
\\
$\mathrm{mask\_ratio}=10^{-3}$,
$K_{knn}=5$,
$\gamma'=0.3$,
$d_0/d=0.93$,
$\mathrm{epoch}=300$,
\\
$\mathrm{lr=0.02}$,
$\mathrm{lr\_decay}=300$,
$\mathrm{dropout}=0.1$,
$\mathrm{weight\_decay}=0$
}
\\ 
\midrule
Amazon Computers 
&
\makecell[c]{
$d=128$,
$\alpha = 0.8$,
$\beta = 0.4$,
$\gamma=5$,
$a=0.2$,
$b=1$,
$n_T=100$,
$p=0.7$,
$q=0.3$,
\\
$\mathrm{mask\_ratio}=10^{-5}$,
$K_{knn}=1$,
$\gamma'=0.7$,
$d_0/d=0.96$,
$\mathrm{epoch}=300$,
\\
$\mathrm{lr=0.1}$,
$\mathrm{lr\_decay}=50$,
$\mathrm{dropout}=0.6$,
$\mathrm{weight\_decay}=10^{-6}$
}
\\ 
\midrule
Texas 
& 
\makecell[c]{
$d=64$,
$\alpha = 0.1$,
$\beta = 0.01$,
$\gamma=0.7$,
$a=0.01$,
$b=0.8$,
$n_T=100$,
$p=0.1$,
$q=0.4$,
\\
$\mathrm{mask\_ratio}=10^{-4}$,
$K_{knn}=2$,
$\gamma'=2$,
$d_0/d=0.99$,
$\mathrm{epoch}=300$,
\\
$\mathrm{lr=0.01}$,
$\mathrm{lr\_decay}=300$,
$\mathrm{dropout}=0.3$,
$\mathrm{weight\_decay}=5\times 10^{-6}$
}
\\ 
\midrule
Wisconsin 
& 
\makecell[c]{
$d=256$,
$\alpha = 0.3$,
$\beta = 0.1$,
$\gamma=10$,
$a=0.2$,
$b=0.8$,
$n_T=20$,
$p=0.9$,
$q=0.7$,
\\
$\mathrm{mask\_ratio}=10^{-5}$,
$K_{knn}=1$,
$\gamma'=2$,
$d_0/d=0.6$,
$\mathrm{epoch}=300$,
\\
$\mathrm{lr=0.05}$,
$\mathrm{lr\_decay}=300$,
$\mathrm{dropout}=0.8$,
$\mathrm{weight\_decay}=0$
}
\\ 
\midrule
Cornell 
&
\makecell[c]{
$d=128$,
$\alpha = 0.6$,
$\beta = 0.2$,
$\gamma=10$,
$a=0.01$,
$b=0.9$,
$n_T=10$,
$p=0.9$,
$q=0.4$,
\\
$\mathrm{mask\_ratio}=10^{-3}$,
$K_{knn}=1$,
$\gamma'=0.5$,
$d_0/d=0.5$,
$\mathrm{epoch}=300$,
\\
$\mathrm{lr=0.01}$,
$\mathrm{lr\_decay}=300$,
$\mathrm{dropout}=0.3$,
$\mathrm{weight\_decay}=5\times 10^{-6}$
}
\\ 
\midrule
Actor 
&
\makecell[c]{
$d=64$,
$\alpha = 100$,
$\beta = 10$,
$\gamma=100$,
$a=0.1$,
$b=0.85$,
$n_T=20$,
$p=0.3$,
$q=0.1$,
\\
$\mathrm{mask\_ratio}=10^{-4}$,
$K_{knn}=7$,
$\gamma'=2$,
$d_0/d=0.5$,
$\mathrm{epoch}=300$,
\\
$\mathrm{lr=0.005}$,
$\mathrm{lr\_decay}=300$,
$\mathrm{dropout}=0.3$,
$\mathrm{weight\_decay}=0$
}
\\ 
\midrule
OGBN-ArXiv
&
\makecell[c]{
$d=128$,
$\alpha = 0.6$,
$\beta = 0.3$,
$\gamma=0.9$,
$a=0.3$,
$b=0.95$,
$n_T=1$,
$p=0.6$,
$q=0.5$,
\\
$\mathrm{mask\_ratio}=10^{-4}$,
$K_{knn}=1$,
$\gamma'=100$,
$d_0/d=0.5$,
$\mathrm{epoch}=1000$,
\\
$\mathrm{lr=0.01}$,
$\mathrm{lr\_decay}=1000$,
$\mathrm{dropout}=0.2$,
$\mathrm{weight\_decay}=0$
}
\\
\bottomrule
\end{tabular}
\caption{Hyper-parameters Configurations of our model \textit{Ours(GCN)} with $32$ layers for node classification tasks on twelve public graph benchmarks. The exactly same hyper-parameters are employed for \textit{Ours(GAT)} and \textit{Ours(GIN)} to reduce the burden of computational resources.}
\label{tab_appendix:Hyper_parameters_configurations}
\end{sidewaystable*}

\begin{sidewaystable*}[t]
\small
\centering
\begin{tabular}{c|c}
\toprule
Hyper-parameter Names & Hyper-parameters Ranges \\ \midrule
$d$
& 
\makecell[c]
{$64,128,256$}
\\ 
\midrule
$\alpha$
&
\makecell[c]
{
$0.01,0.1,0.2,0.3,0.4,0.5,0.6,0.7,0.8,0.9,5,10,100$
}
\\
\midrule
$\beta$
&
\makecell[c]
{
$0.01,0.1,0.2,0.3,0.4,0.5,0.6,0.7,0.8,0.9,5,10,100$
}
\\
\midrule
$\gamma$
&
\makecell[c]
{
$0.01,0.1,0.2,0.3,0.4,0.5,0.6,0.7,0.8,0.9,5,10,100$
}
\\ 
\midrule
$a$
&
\makecell[c]
{
$0.01,0.1,0.2,0.3,0.4,0.5,0.6,0.7,0.8,0.9,0.99$
}
\\ 
\midrule
$b$
&
\makecell[c]
{
$1.0,0.95,0.9,0.85,0.8$
}
\\ 
\midrule
$n_T$
&
\makecell[c]
{
$5,10,20,50,100$
}
\\ 
\midrule
$p$
& 
\makecell[c]
{
$0,0.1,0.2,0.3,0.4,0.5,0.6,0.7,0.8,0.9,0.99$
}
\\ 
\midrule
$q$
& 
\makecell[c]
{
$0,0.1,0.2,0.3,0.4,0.5,0.6,0.7,0.8,0.9,0.99$
}
\\ 
\midrule
$K_{knn}$
&
\makecell[c]
{
$1,2,3,4,5,6,7,8$
}
\\ 
\midrule
$\gamma'$
&
\makecell[c]
{
$0.001,0.01,0.1,0.3,0.5,0.7,1,2,5,10,100$
}
\\
\midrule
$d_0/d$
&
\makecell[c]
{
$0.1,0.2,0.3,0.4,0.5,0.6,0.7,0.8,0.9,0.91,0.92,0.93,0.94,0.95,0.96,0.97,0.98,0.99,1$
}
\\
\midrule
$\mathrm{lr}$
&
\makecell[c]
{
$0.1,0.01,0.05,0.001,0.005$
}
\\
\midrule
$\mathrm{lr\_decay}$
&
\makecell[c]
{
$50,100,150,200,250,300$
}
\\
\midrule
$\mathrm{dropout}$
&
\makecell[c]
{
$0.1,0.2,0.3,0.4,0.5,0.6,0.7,0.8,0.9$
}
\\
\midrule
 $\mathrm{weight\_decay}$
&
\makecell[c]
{
$0,10^{-4},5\times 10^{-4},10^{-5},5\times 10^{-5},10^{-6},5\times 10^{-6}$
}
\\
\midrule
 $\mathrm{mask\_ration}$
&
\makecell[c]
{
$10^{-5}, 10^{-4}, 10^{-3}, 10^{-2}, 0.1, 0.2, 0.3, 0.4, 0.5, 0.6$
}
\\
\bottomrule
\end{tabular}
\caption{The Hyper-parameter Search Spaces}
\label{tab_appendix:hyper-parameters_search_spaces}
\end{sidewaystable*}

\begin{figure*}[htbp]
\centering
\subfigure[Cora(Original Feature)]{
\begin{minipage}[b]{0.3\linewidth}
\centering
\includegraphics[width=4.3cm]{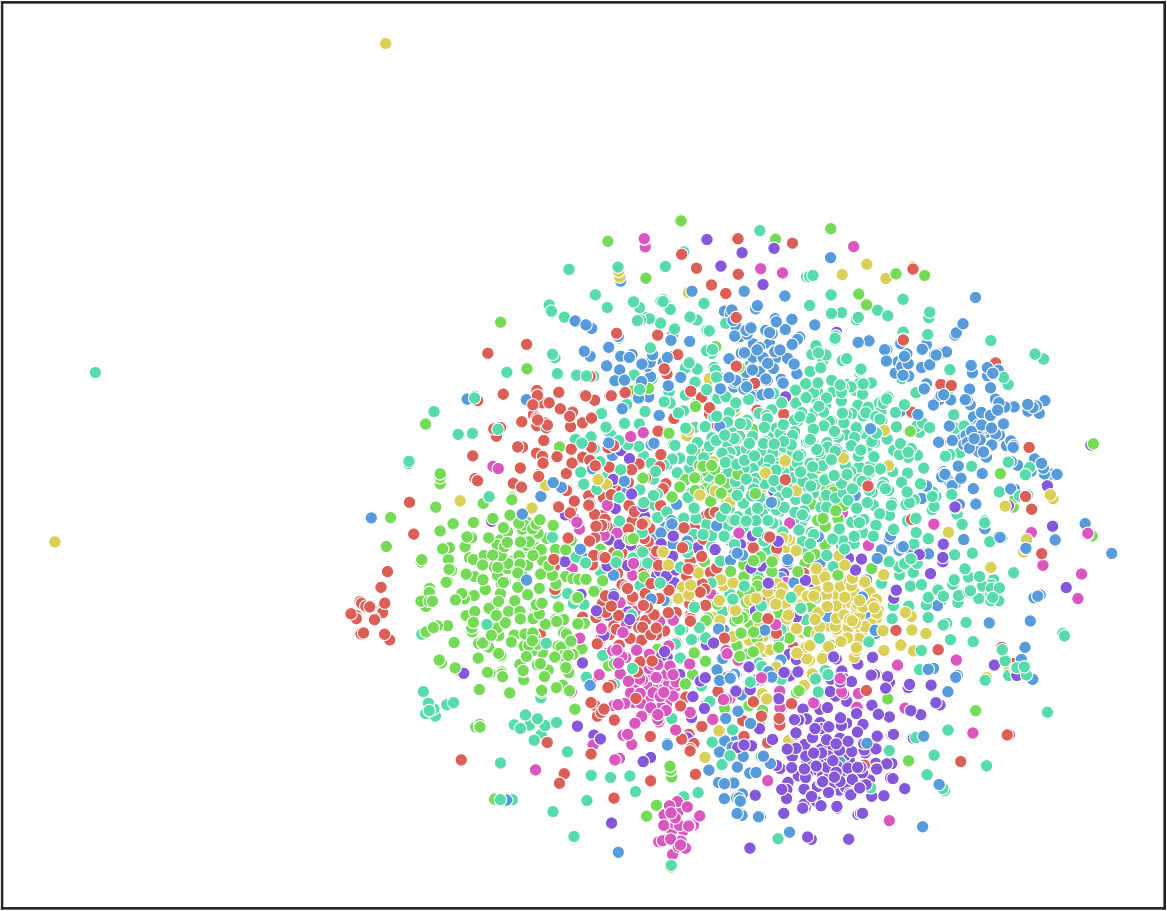}
\end{minipage}
}
\subfigure[Cora(SGC)]{
\begin{minipage}[b]{0.3\linewidth}
\centering
\includegraphics[width=4.3cm]{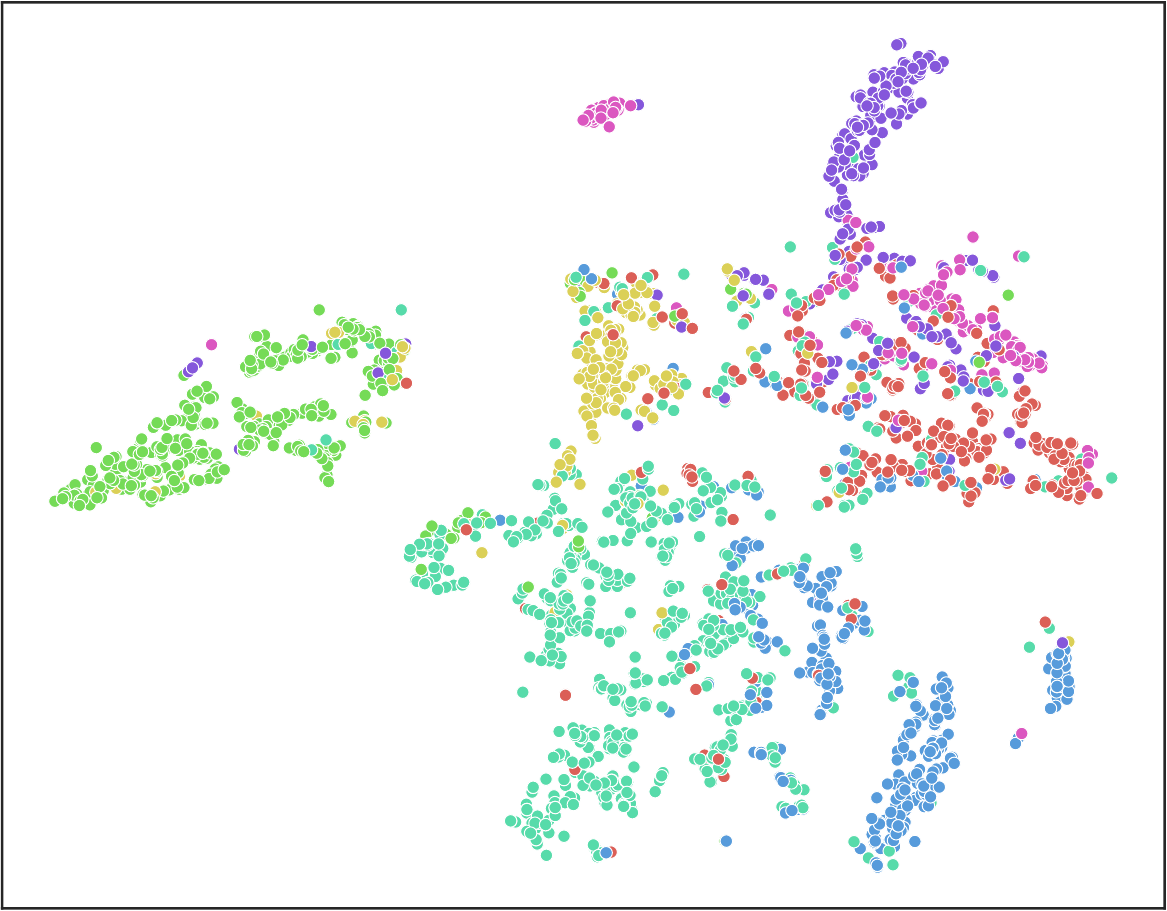}
\end{minipage}
}
\subfigure[Cora(Ours)]{
\begin{minipage}[b]{0.3\linewidth}
\centering
\includegraphics[width=4.3cm]{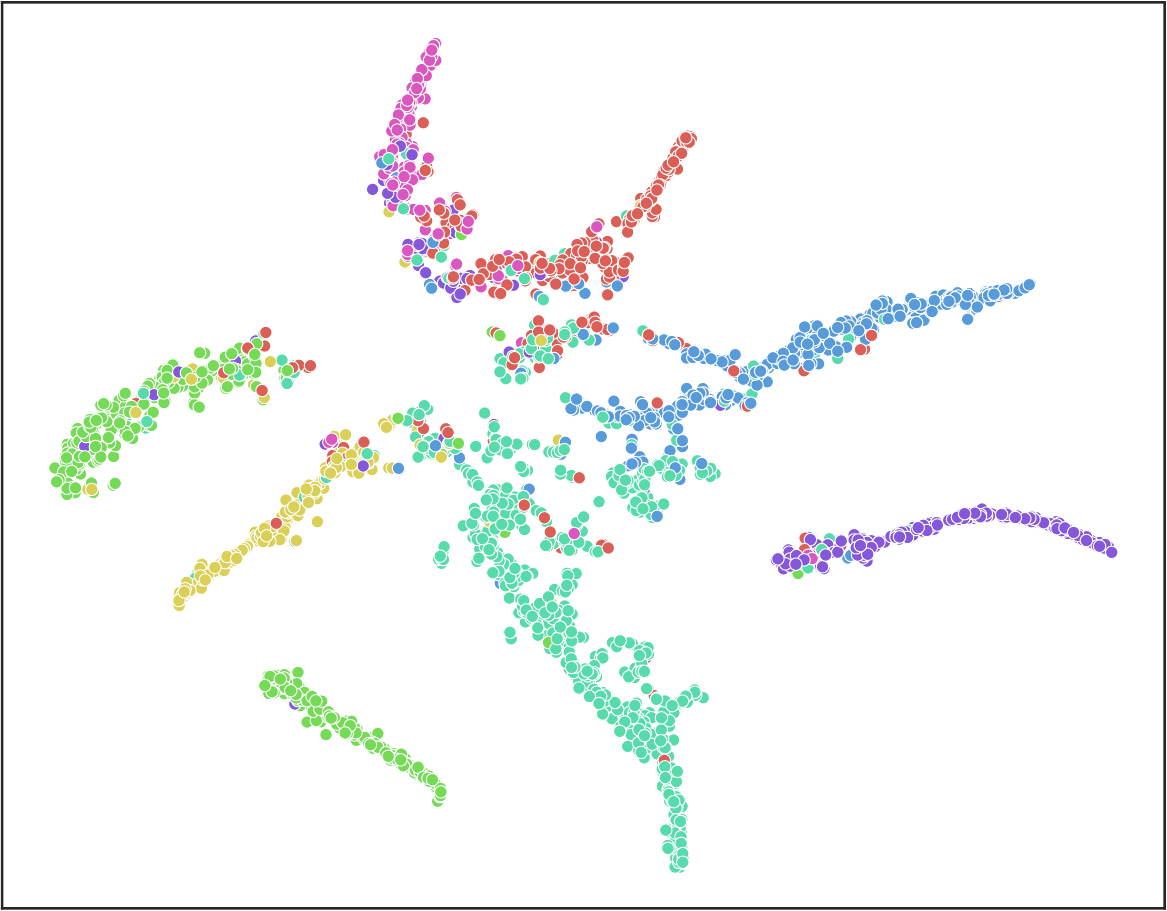}
\end{minipage}
}
\subfigure[Citeseer(Original Feature)]{
\begin{minipage}[b]{0.3\linewidth}
\centering
\includegraphics[width=4.3cm]{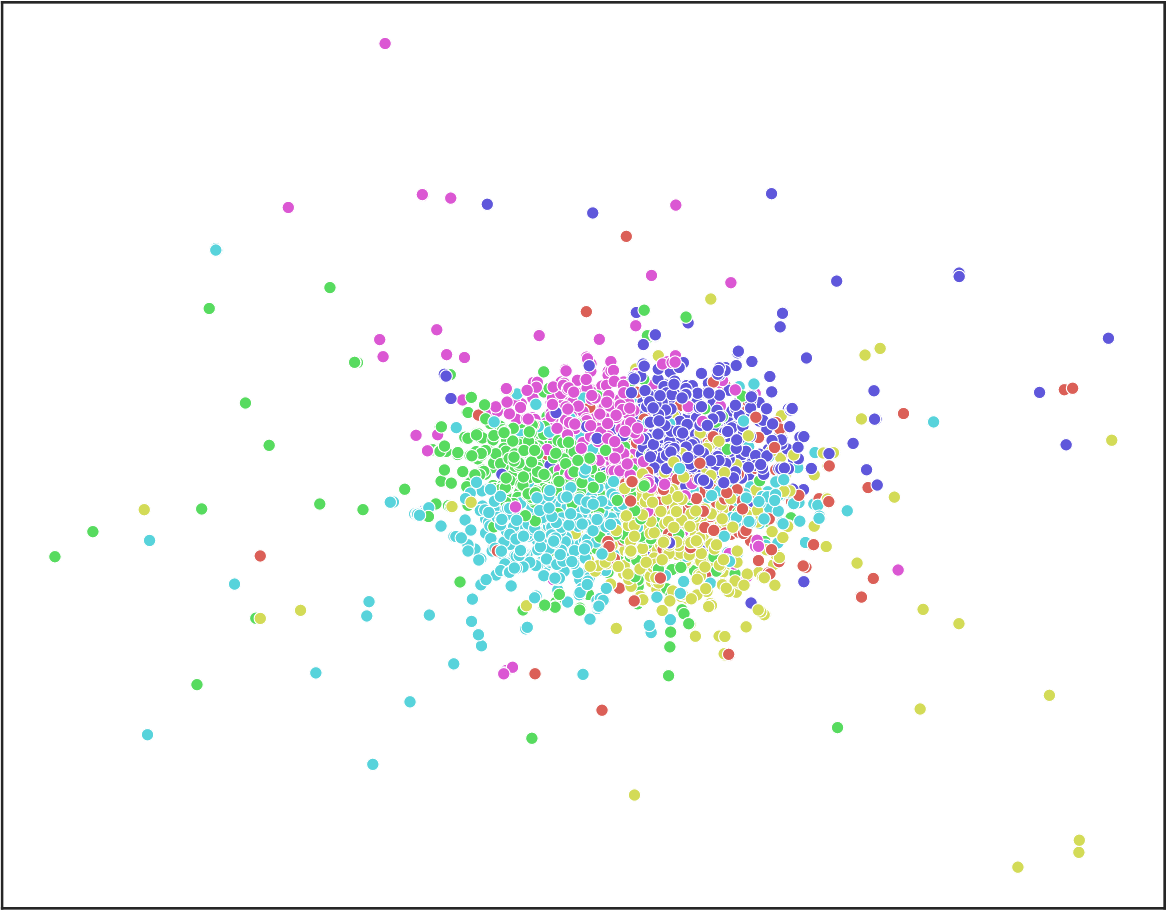}
\end{minipage}
}
\subfigure[Citeseer(SGC)]{
\begin{minipage}[b]{0.3\linewidth}
\centering
\includegraphics[width=4.3cm]{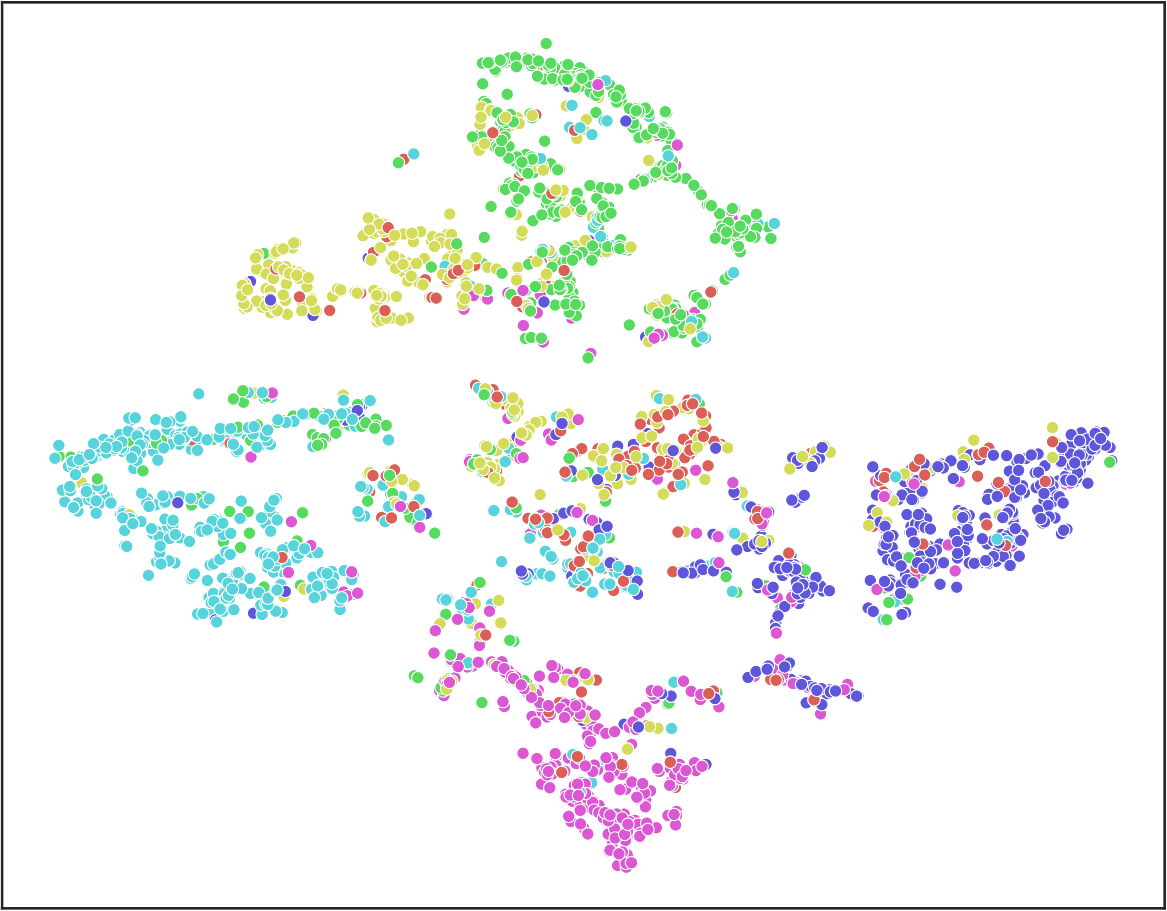}
\end{minipage}
}
\subfigure[Citeseer(Ours)]{
\begin{minipage}[b]{0.3\linewidth}
\centering
\includegraphics[width=4.3cm]{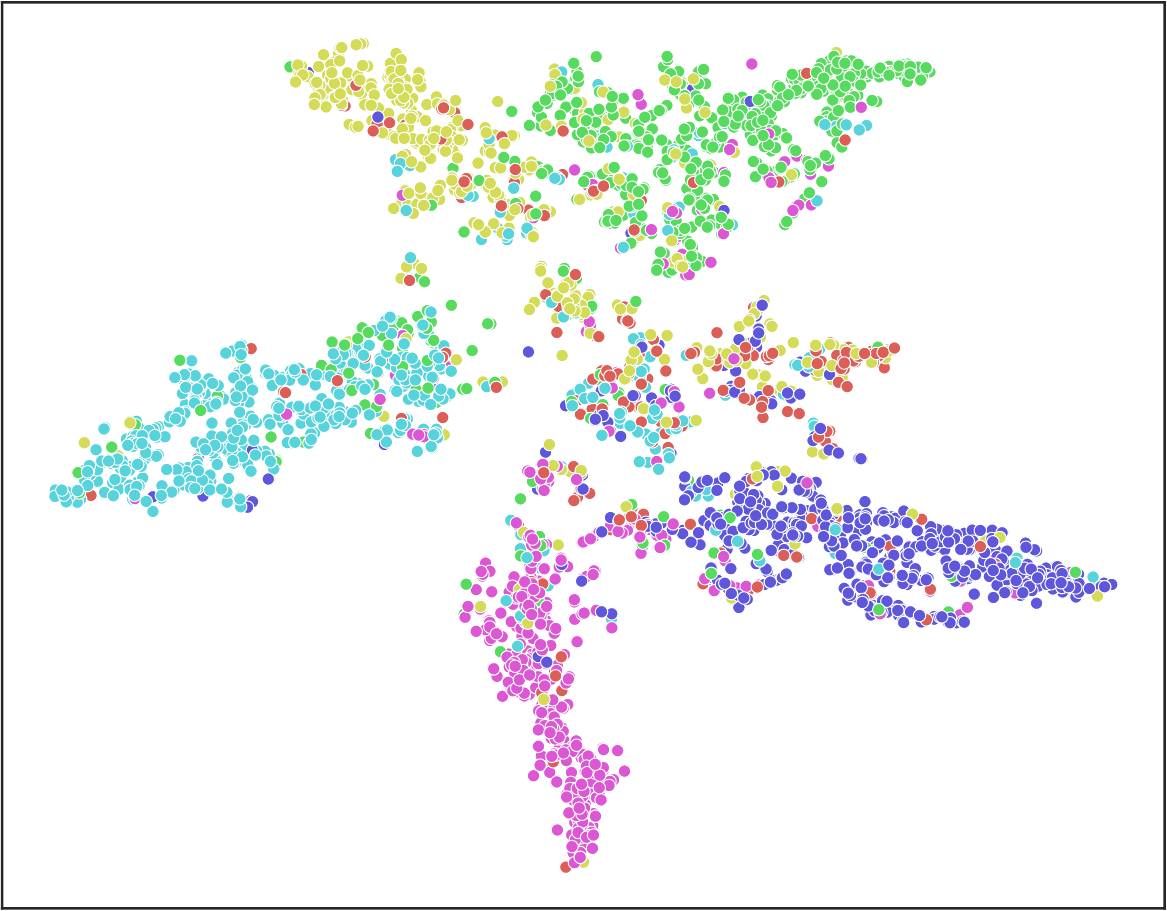}
\end{minipage}
}
\subfigure[Pubmed(Original Feature)]{
\begin{minipage}[b]{0.3\linewidth}
\centering
\includegraphics[width=4.3cm]{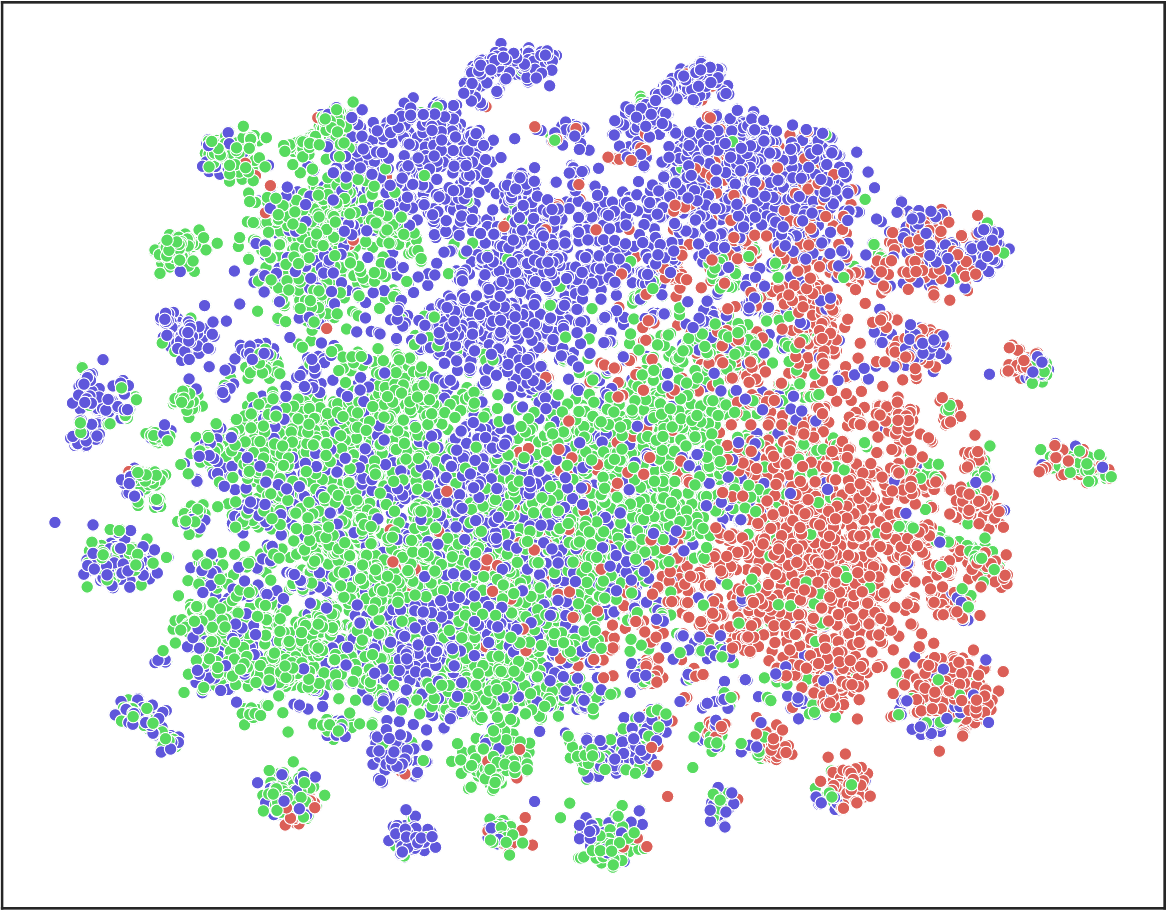}
\end{minipage}
}
\subfigure[Pubmed(SGC)]{
\begin{minipage}[b]{0.3\linewidth}
\centering
\includegraphics[width=4.3cm]{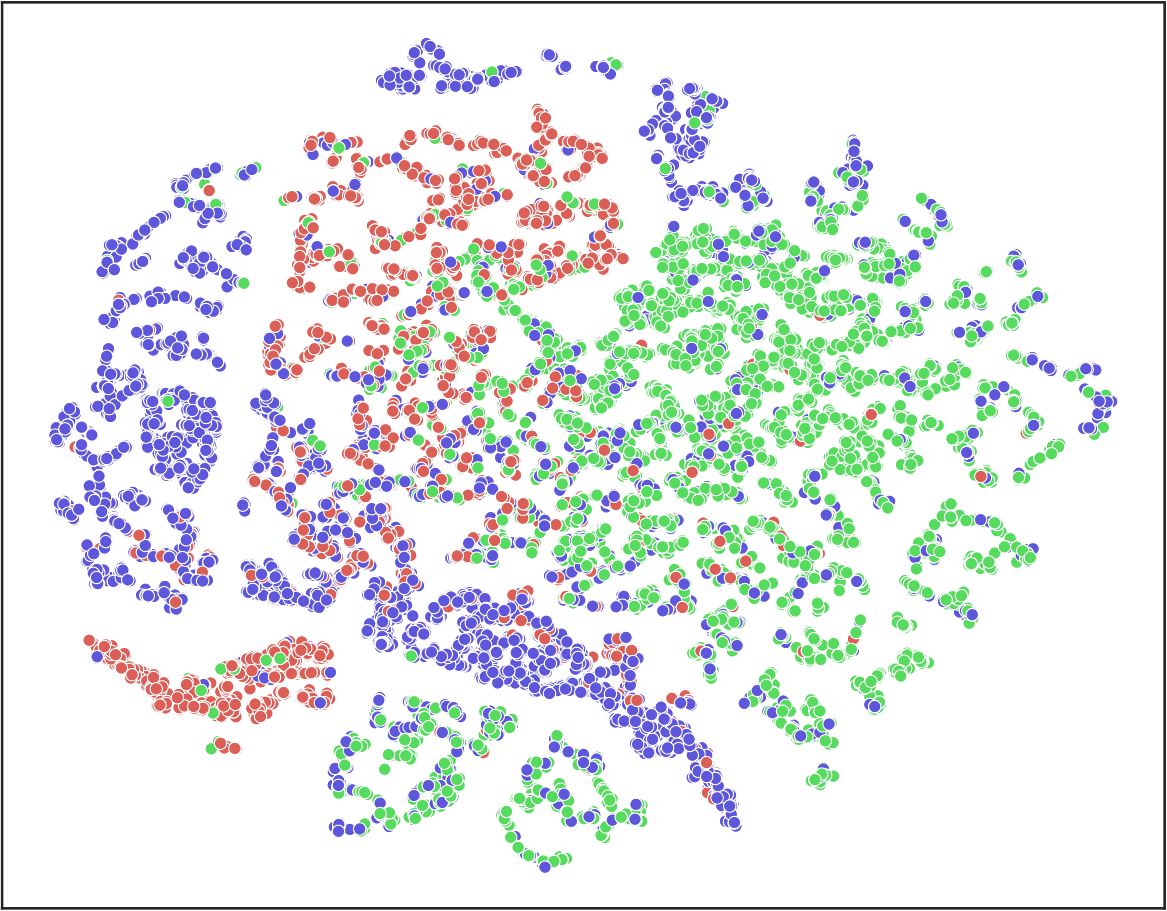}
\end{minipage}
}
\subfigure[Pubmed(Ours)]{
\begin{minipage}[b]{0.3\linewidth}
\centering
\includegraphics[width=4.3cm]{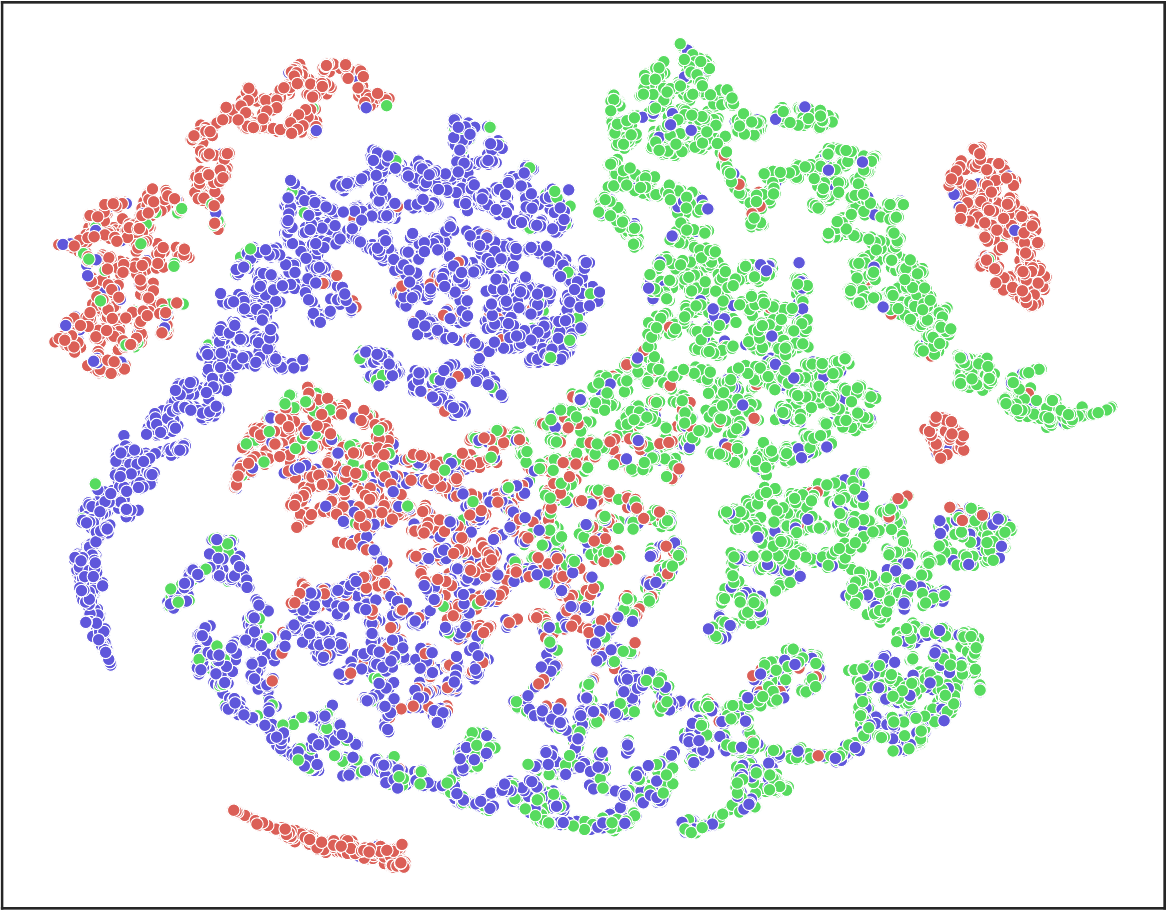}
\end{minipage}
}
\caption{Scatter plots of original features and embeddings output via 32-layer SGC and our model on real-world benchmarks Cora, Citeseer, and Pubmed}
\label{fig_appendix:embeddings_cora_citeseer_pubmed}
\end{figure*}



\begin{figure*}[htbp]
\centering
\subfigure[CoauthorCS(Original Feature)]{
\begin{minipage}[b]{0.3\linewidth}
\centering
\includegraphics[width=4.3cm]{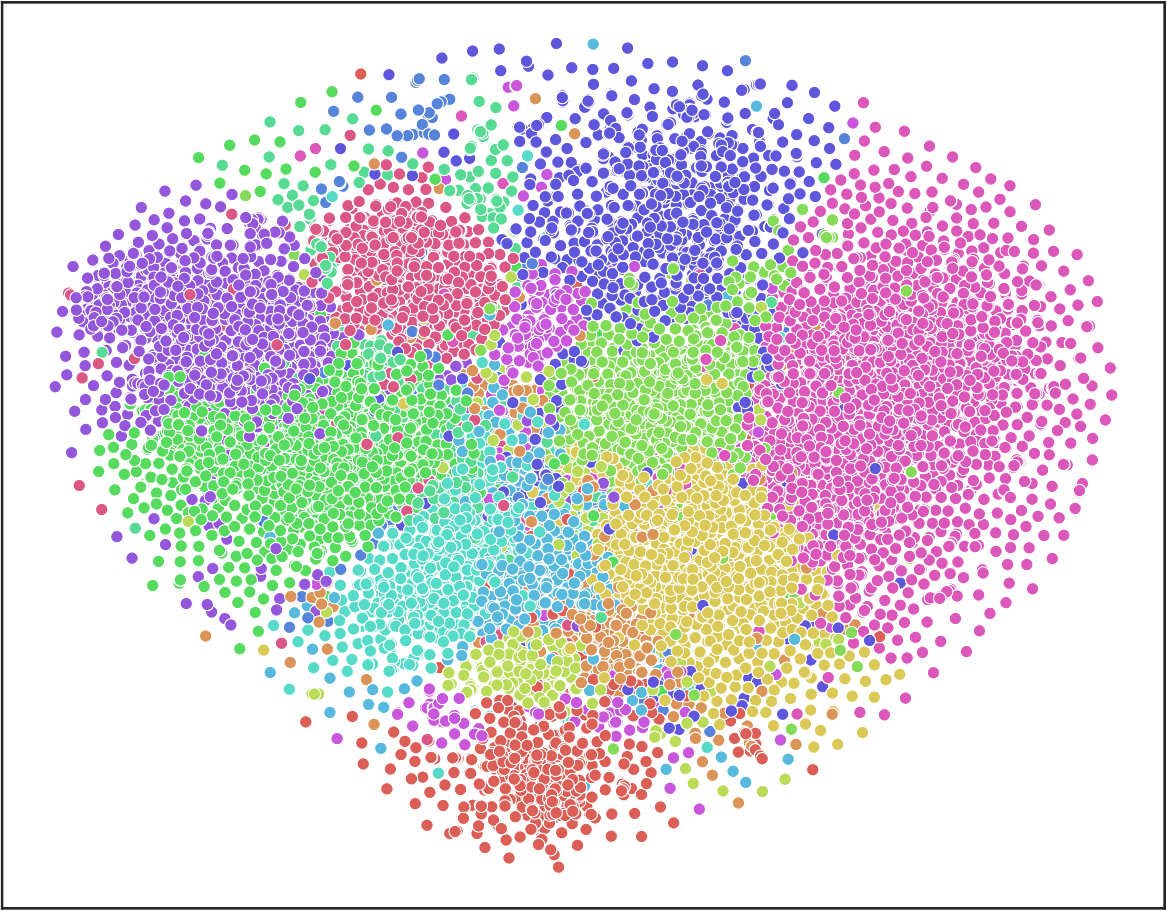}
\end{minipage}
}
\subfigure[CoauthorCS(SGC)]{
\begin{minipage}[b]{0.3\linewidth}
\centering
\includegraphics[width=4.3cm]{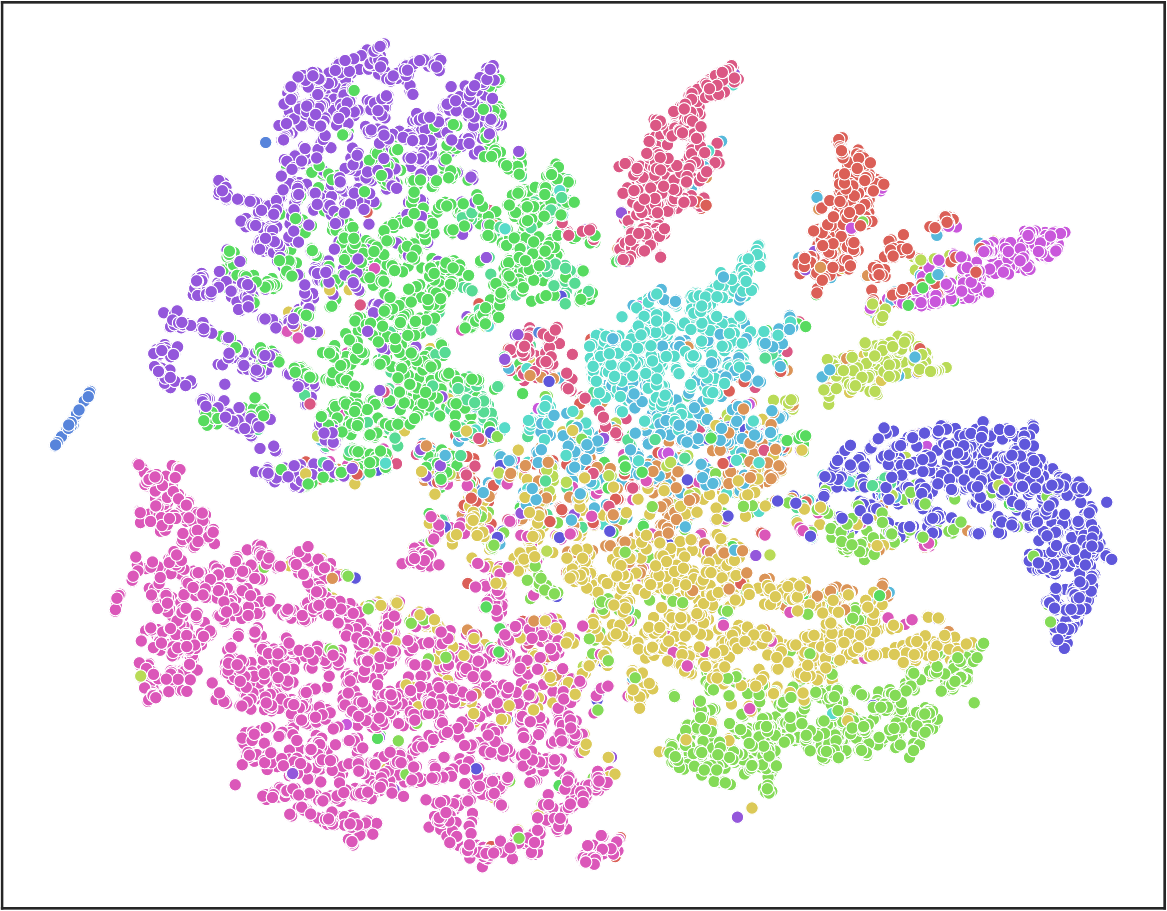}
\end{minipage}
}
\subfigure[CoauthorCS(Ours)]{
\begin{minipage}[b]{0.3\linewidth}
\centering
\includegraphics[width=4.3cm]{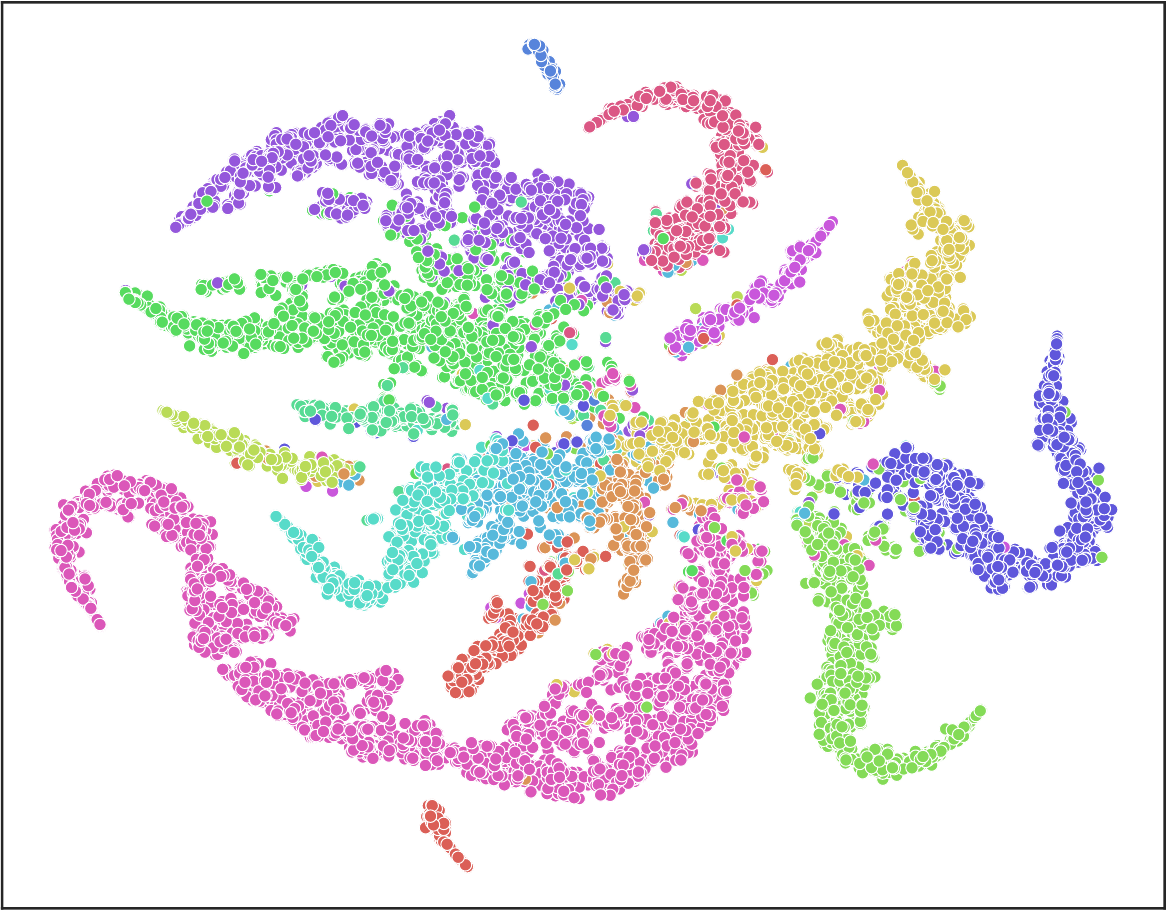}
\end{minipage}
}
\subfigure[CoauthorPhysics(Original Feature)]{
\begin{minipage}[b]{0.3\linewidth}
\centering
\includegraphics[width=4.3cm]{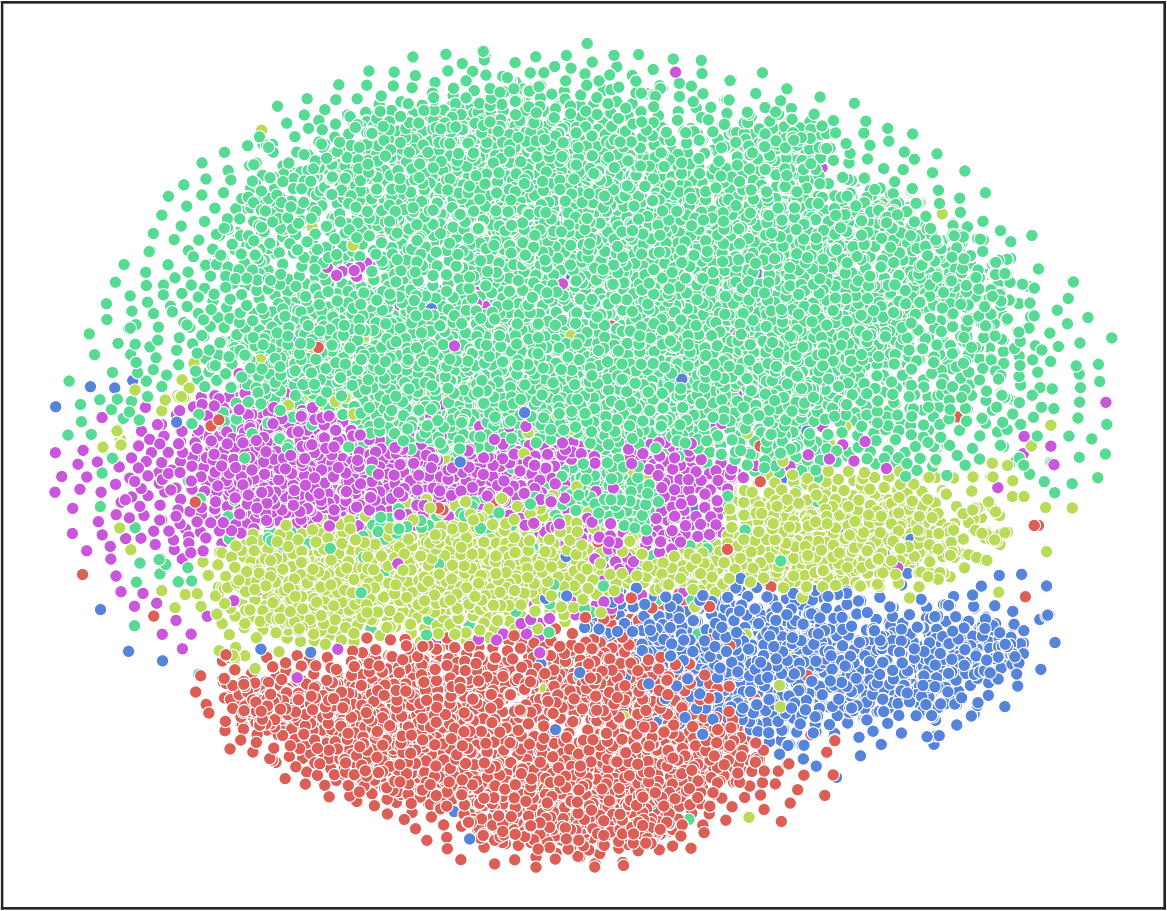}
\end{minipage}
}
\subfigure[CoauthorPhysics(SGC)]{
\begin{minipage}[b]{0.3\linewidth}
\centering
\includegraphics[width=4.3cm]{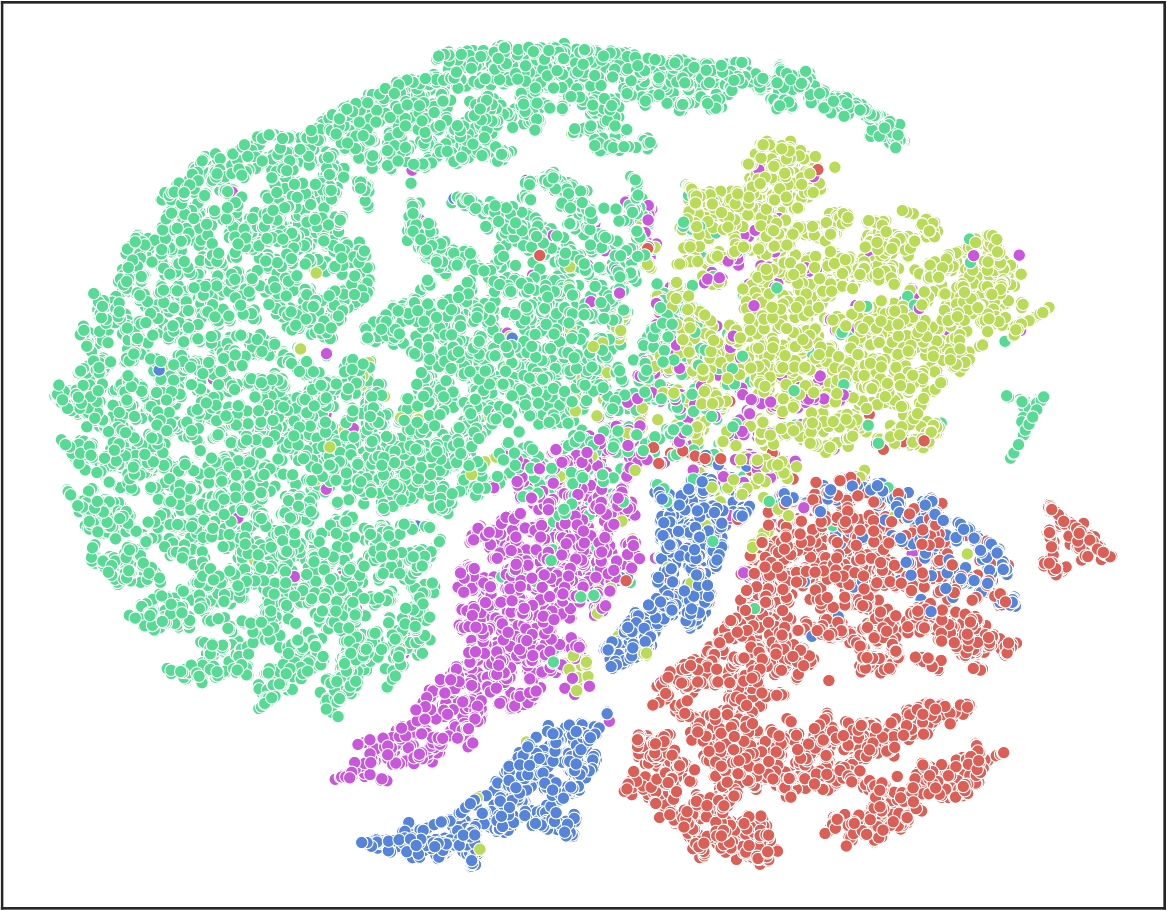}
\end{minipage}
}
\subfigure[CoauthorPhysics(Ours)]{
\begin{minipage}[b]{0.3\linewidth}
\centering
\includegraphics[width=4.3cm]{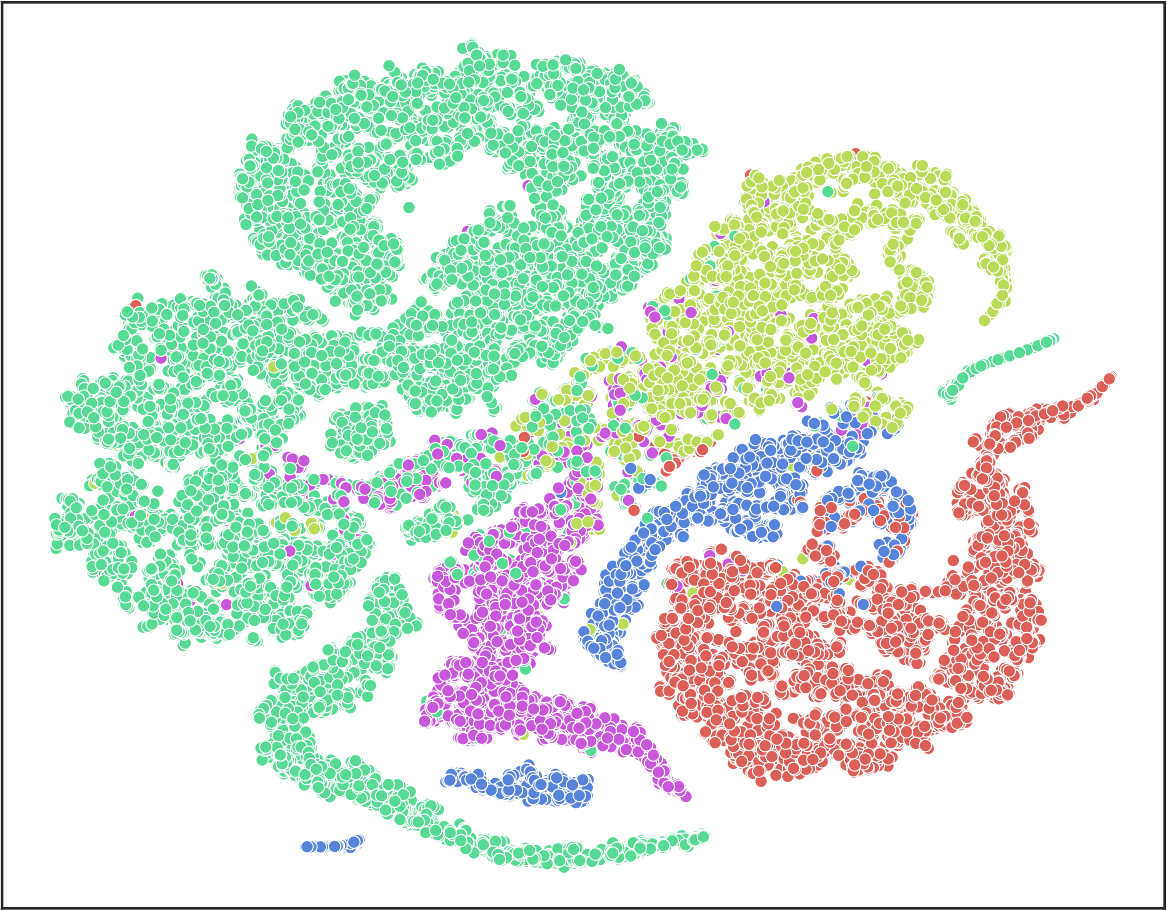}
\end{minipage}
}
 \subfigure[AmazonComputers(Original Feature)]{
\begin{minipage}[b]{0.3\linewidth}
\centering
\includegraphics[width=4.3cm]{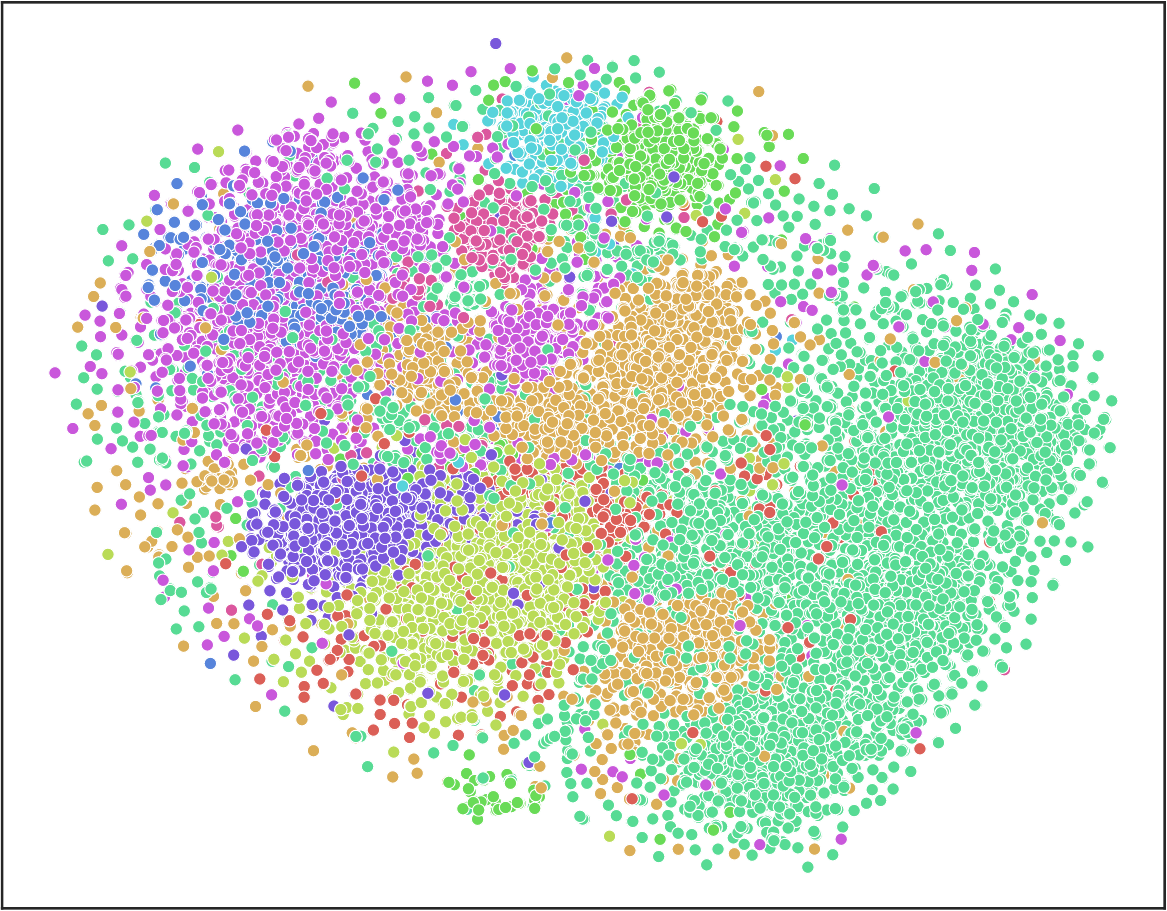}
\end{minipage}
}
\subfigure[AmazonComputers(SGC)]{
\begin{minipage}[b]{0.3\linewidth}
\centering
\includegraphics[width=4.3cm]{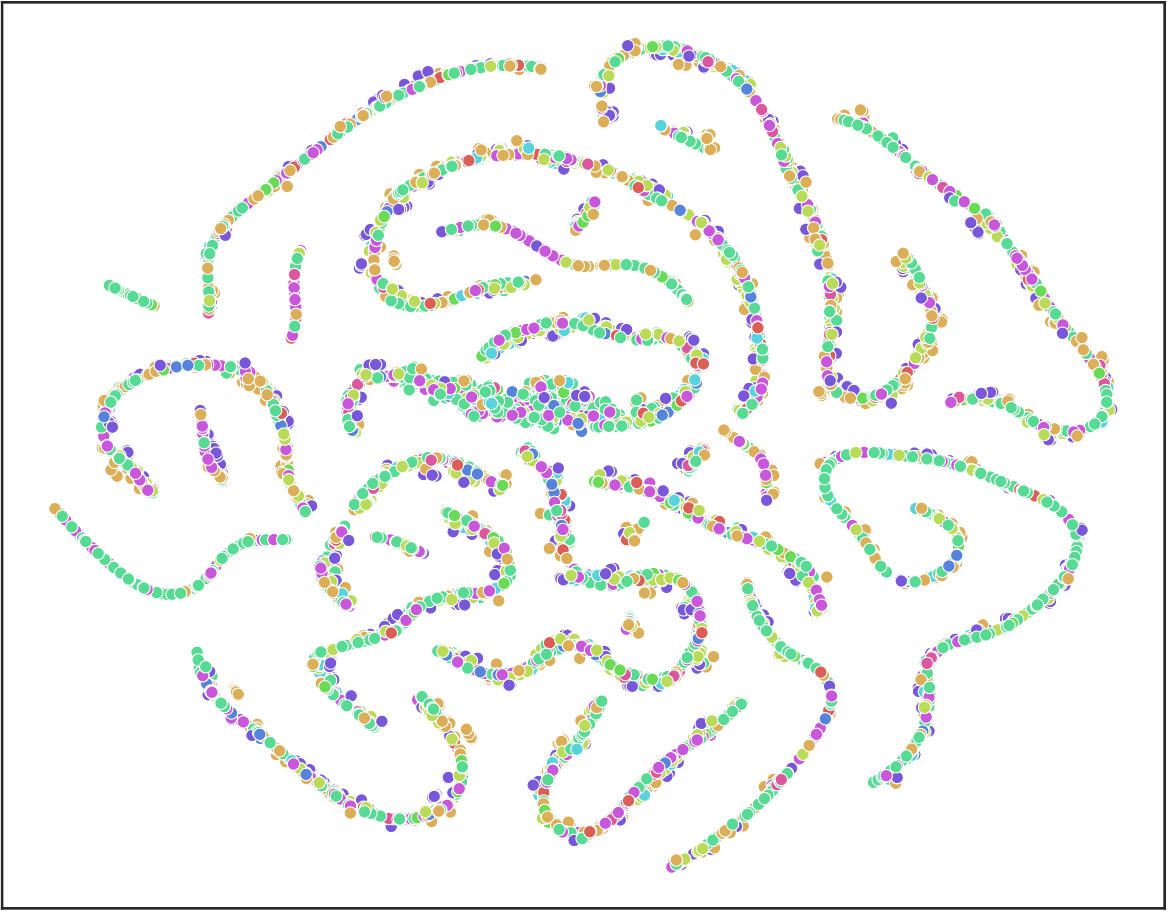}
\end{minipage}
}
\subfigure[AmazonComputers(Ours)]{
\begin{minipage}[b]{0.3\linewidth}
\centering
\includegraphics[width=4.3cm]{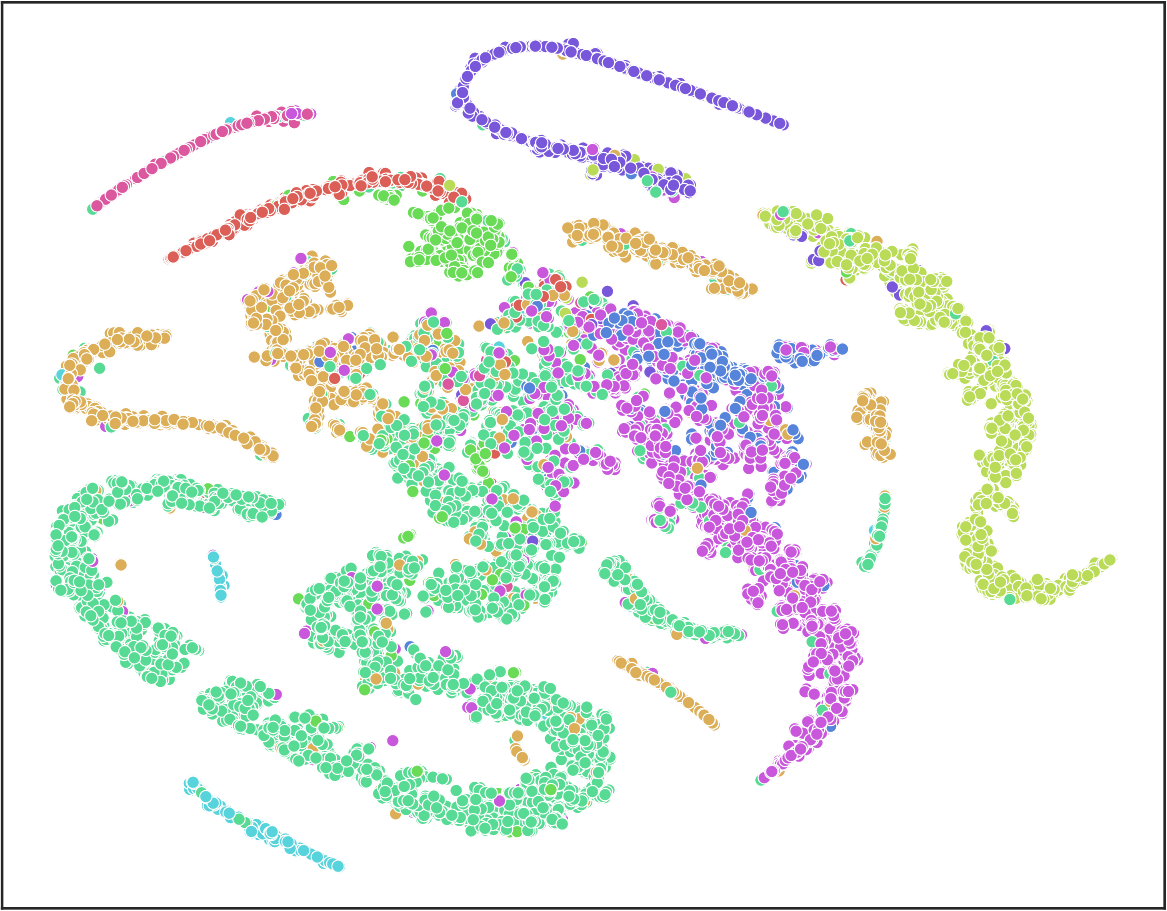}
\end{minipage}
}
 \subfigure[AmazonPhoto(Original Feature)]{
\begin{minipage}[b]{0.3\linewidth}
\centering
\includegraphics[width=4.3cm]{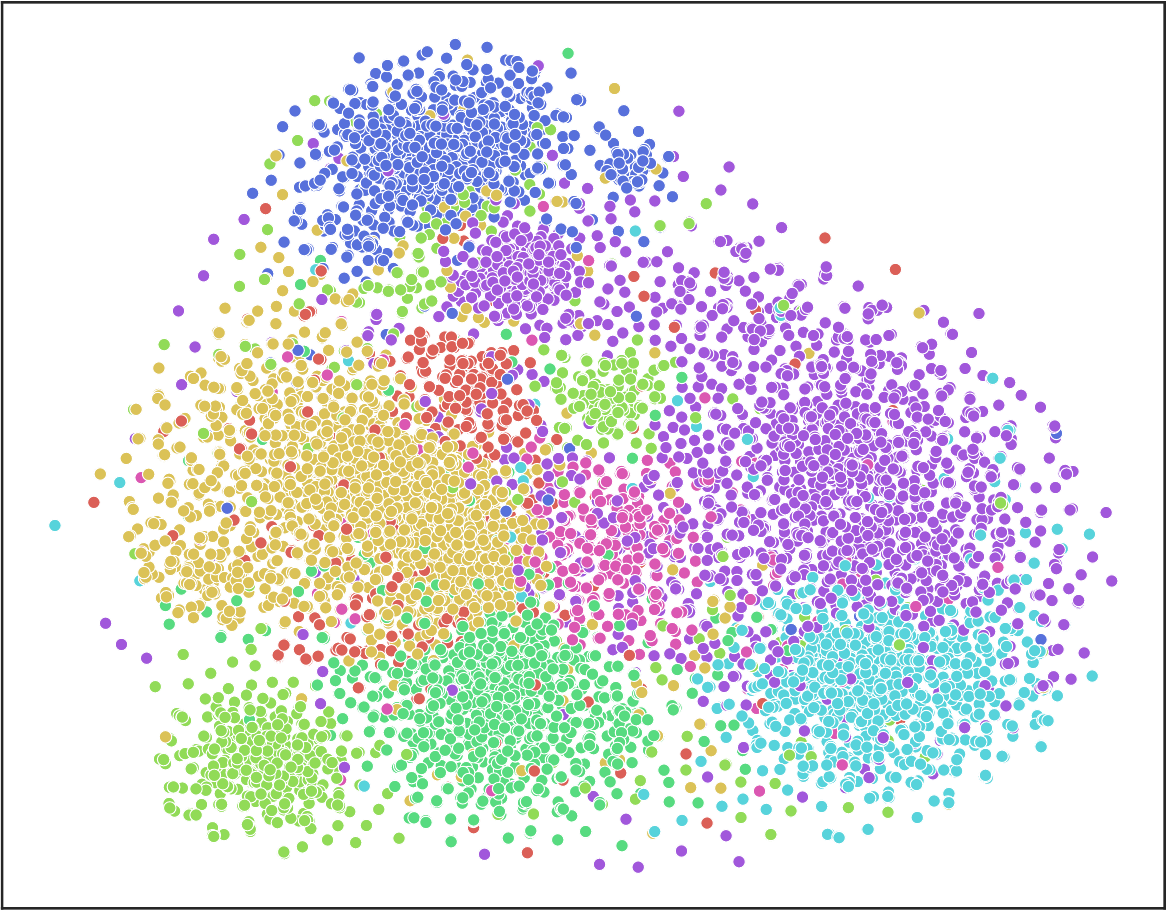}
\end{minipage}
}
\subfigure[AmazonPhoto(SGC)]{
\begin{minipage}[b]{0.3\linewidth}
\centering
\includegraphics[width=4.3cm]{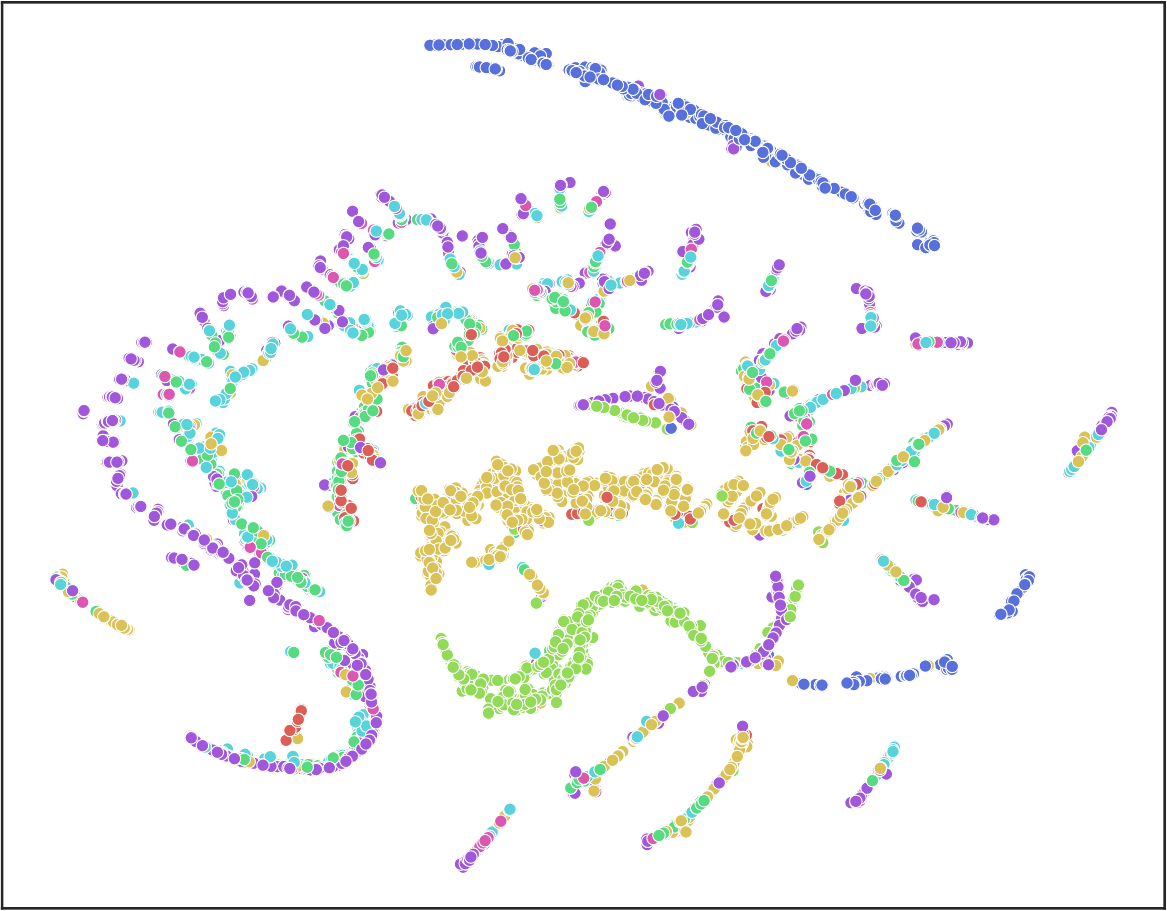}
\end{minipage}
}
\subfigure[AmazonPhoto(Ours)]{
\begin{minipage}[b]{0.3\linewidth}
\centering
\includegraphics[width=4.3cm]{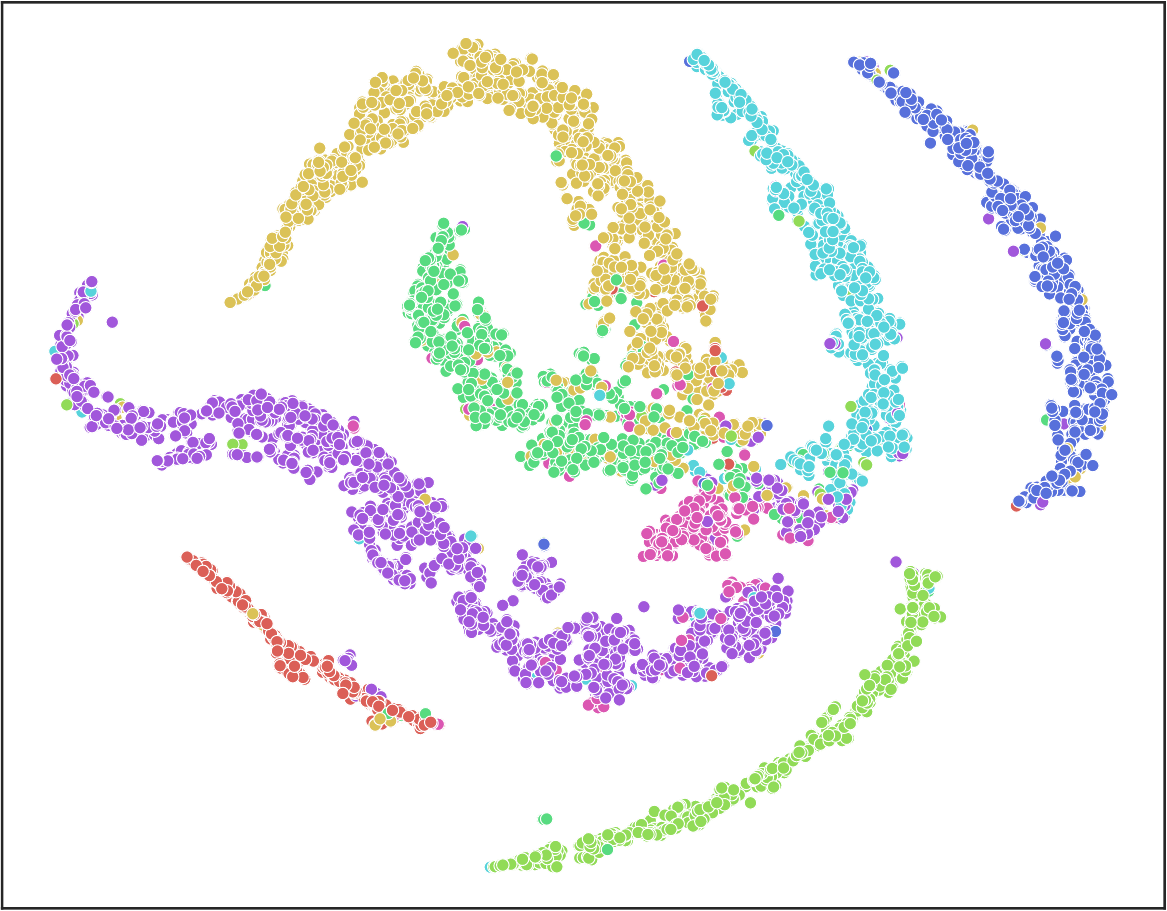}
\end{minipage}
}   
\caption{Scatter plots of original features and embeddings output via 32-layer SGC and our model on real-world benchmarks CS, Physics, Computers, and Photo}
\label{fig_appendix:embeddings_CS_Physics_Computers_Photo}
\end{figure*}

\begin{figure*}[htbp]
\centering
    
\subfigure[Texas(Original Feature)]{
\begin{minipage}[b]{0.3\linewidth}
\centering
\includegraphics[width=4.3cm]{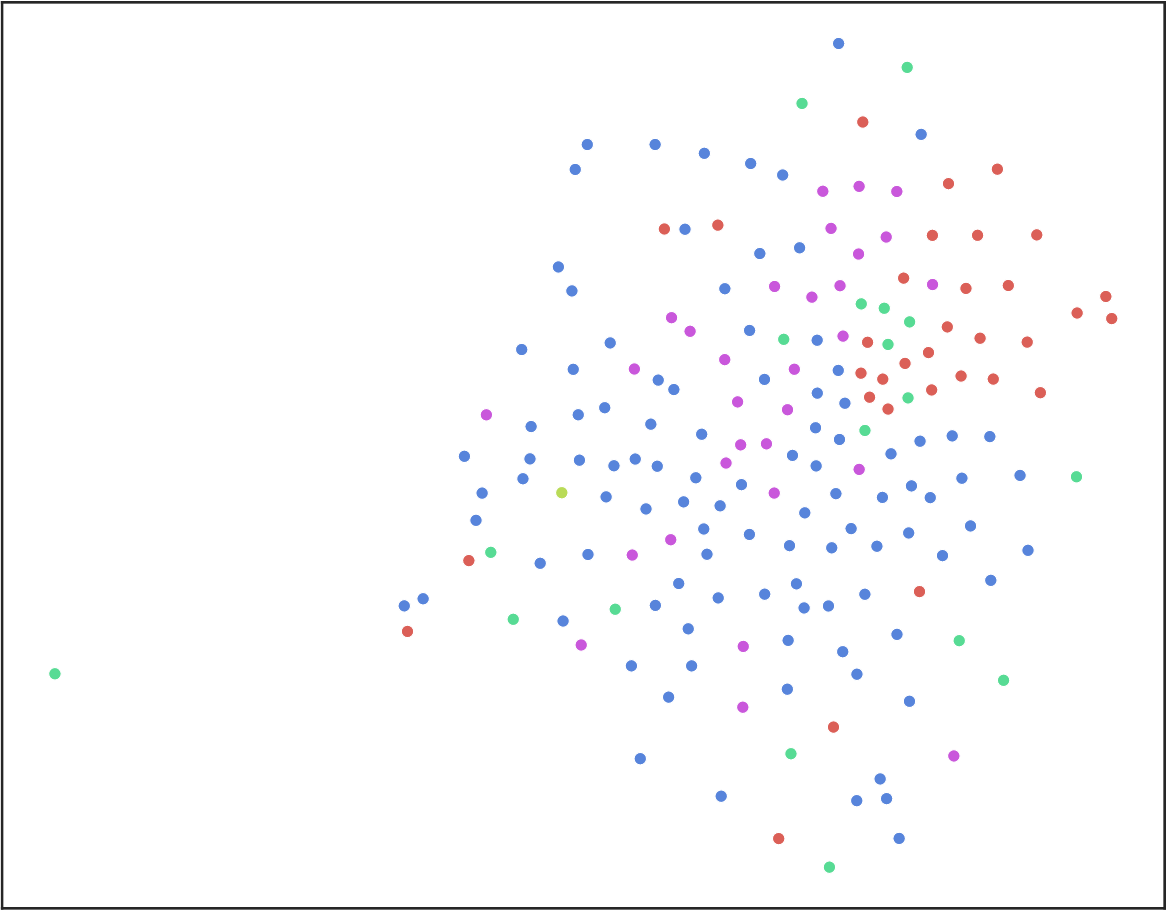}
\end{minipage}
}
\subfigure[Texas(SGC)]{
\begin{minipage}[b]{0.3\linewidth}
\centering
\includegraphics[width=4.3cm]{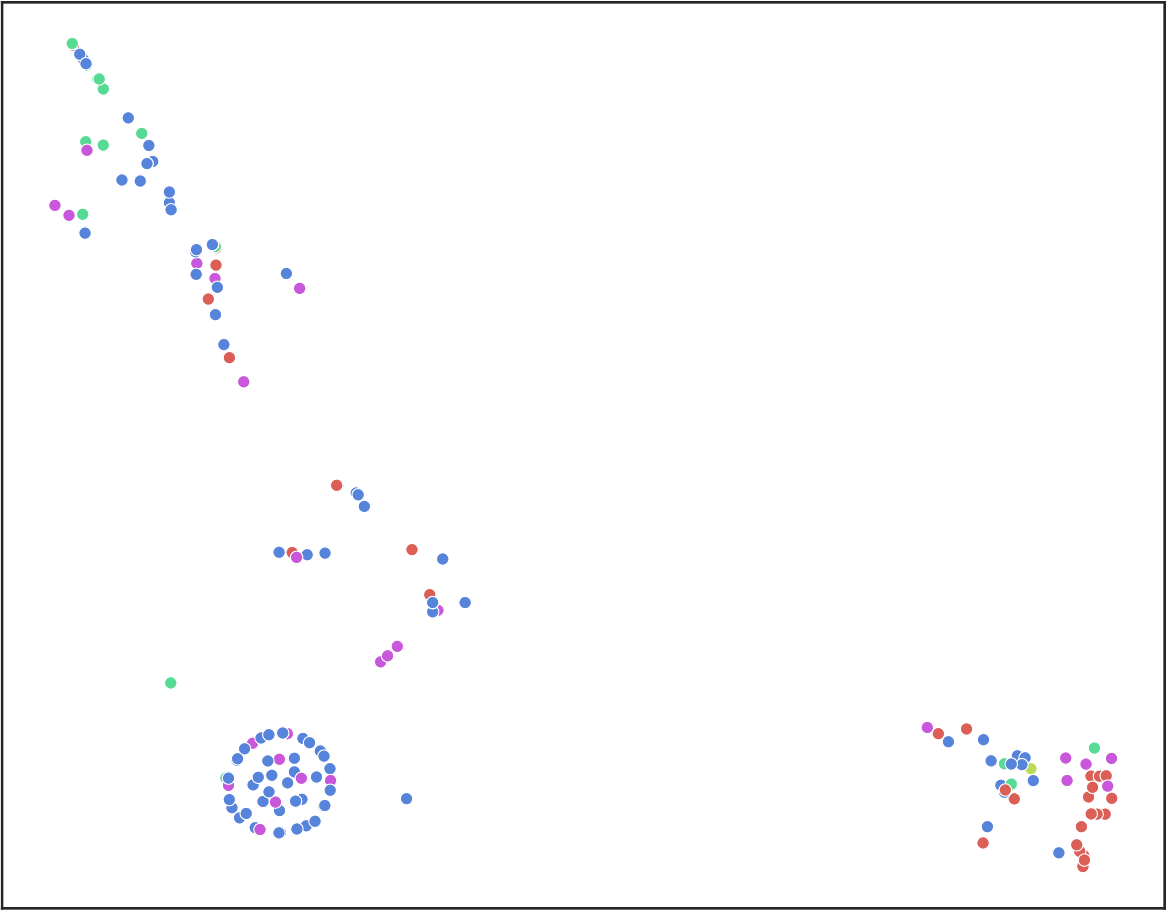}
\end{minipage}
}
\subfigure[Texas(Ours)]{
\begin{minipage}[b]{0.3\linewidth}
\centering
\includegraphics[width=4.3cm]{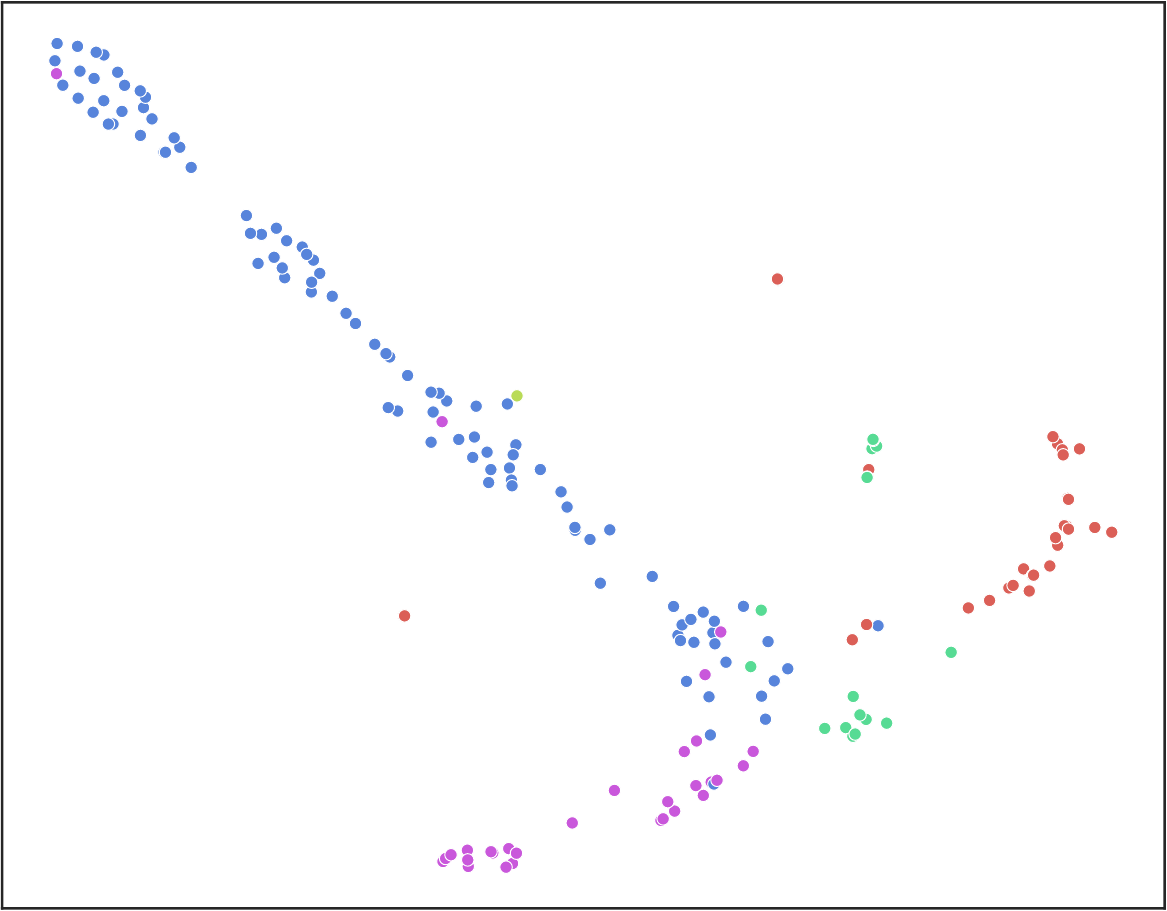}
\end{minipage}
}
\subfigure[Wisconsin(Original Feature)]{
\begin{minipage}[b]{0.3\linewidth}
\centering
\includegraphics[width=4.3cm]{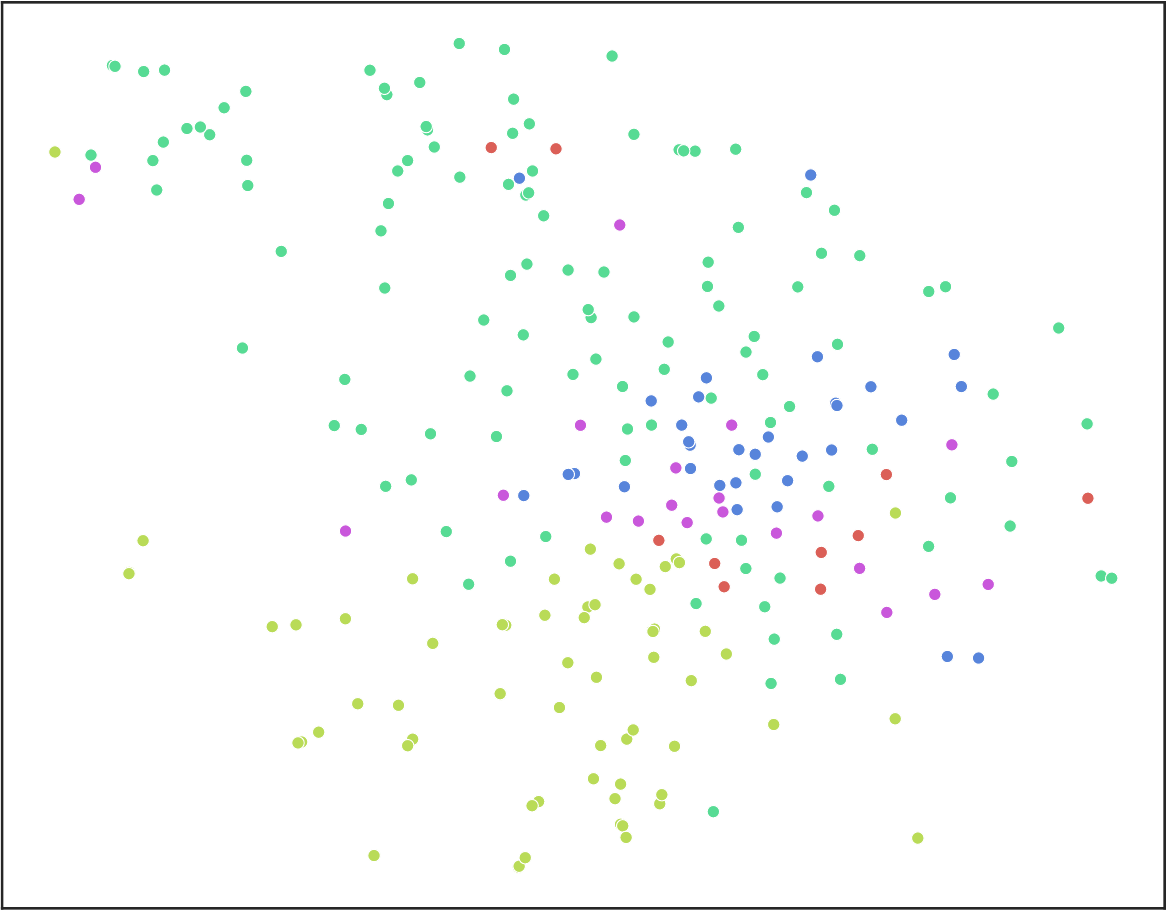}
\end{minipage}
}
\subfigure[Wisconsin(SGC)]{
\begin{minipage}[b]{0.3\linewidth}
\centering
\includegraphics[width=4.3cm]{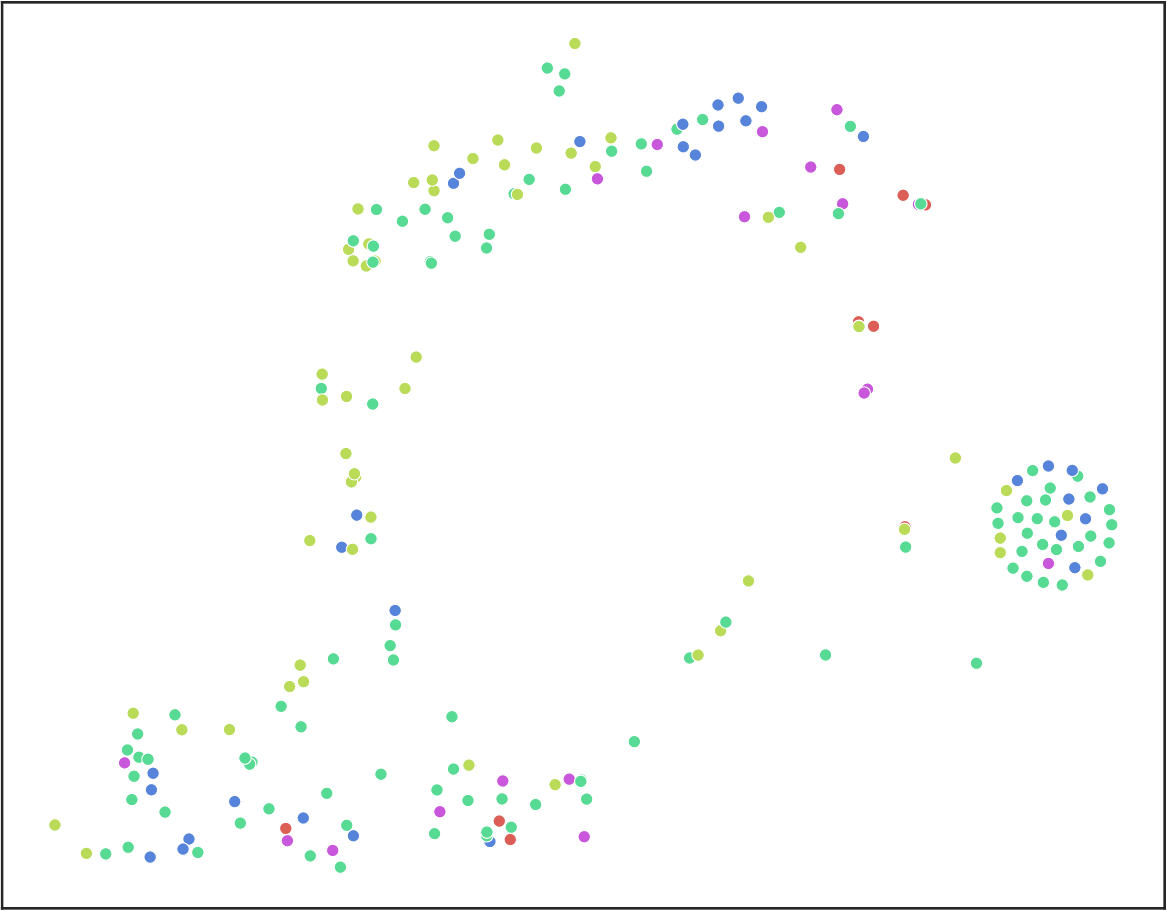}
\end{minipage}
}
\subfigure[Wisconsin(Ours)]{
\begin{minipage}[b]{0.3\linewidth}
\centering
\includegraphics[width=4.3cm]{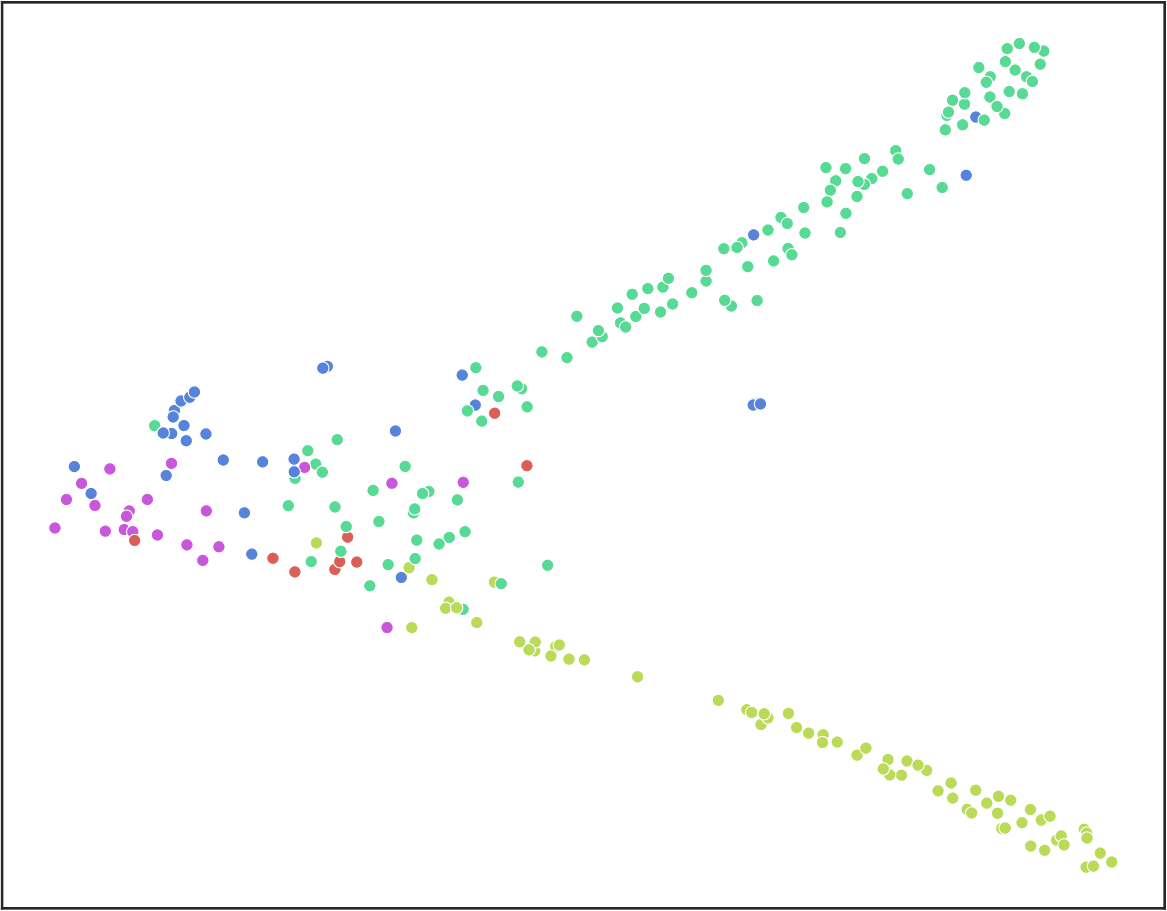}
\end{minipage}
}
\subfigure[Cornell(Original Feature)]{
\begin{minipage}[b]{0.3\linewidth}
\centering
\includegraphics[width=4.3cm]{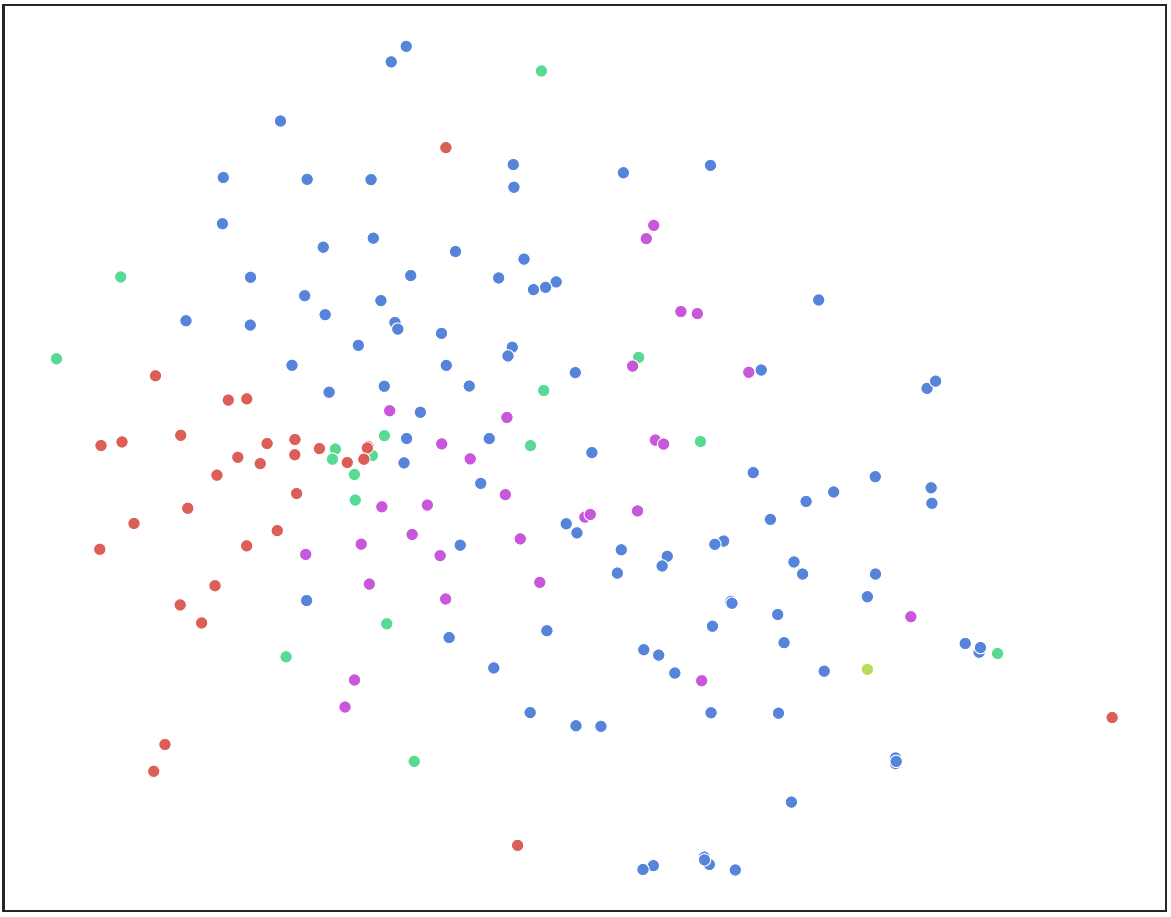}
\end{minipage}
}
\subfigure[Cornell(SGC)]{
\begin{minipage}[b]{0.3\linewidth}
\centering
\includegraphics[width=4.3cm]{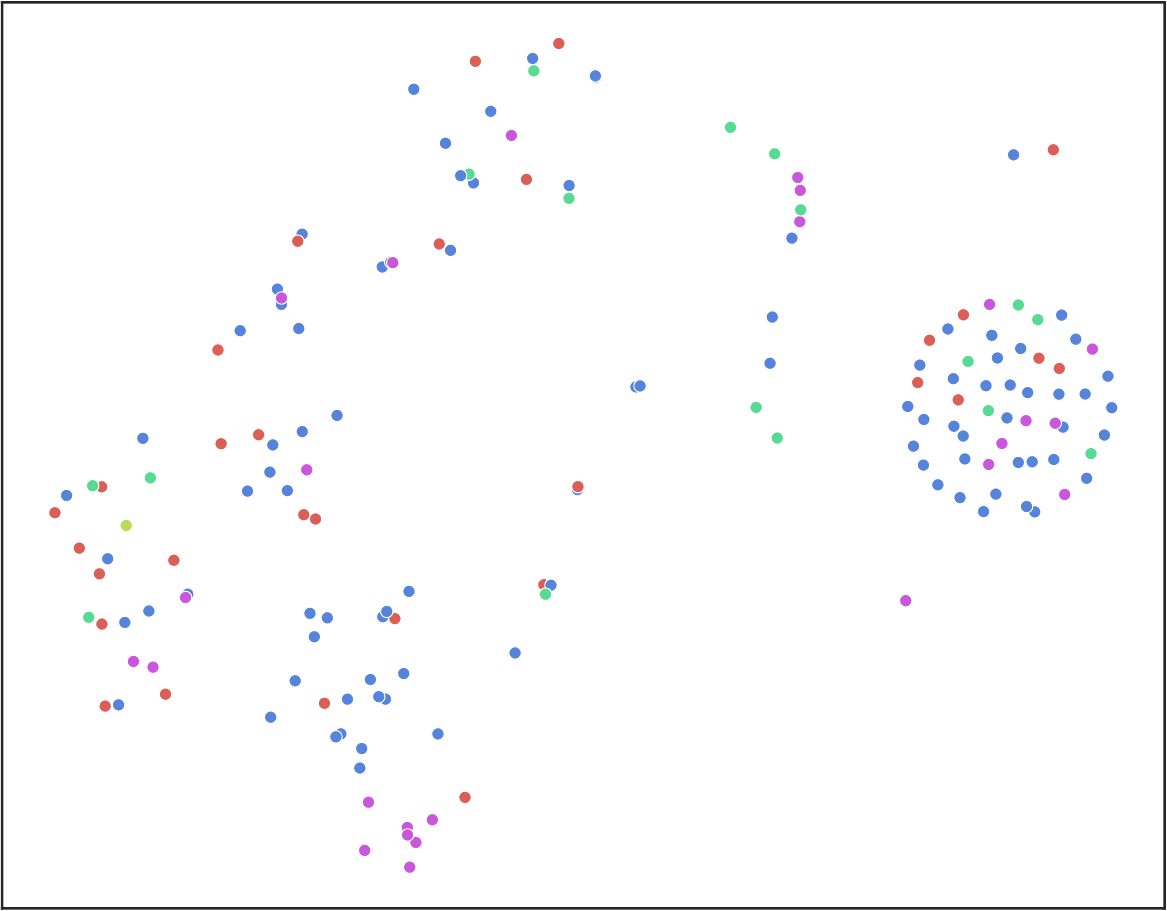}
\end{minipage}
}\subfigure[Cornell(Ours)]{
\begin{minipage}[b]{0.3\linewidth}
\centering
\includegraphics[width=4.3cm]{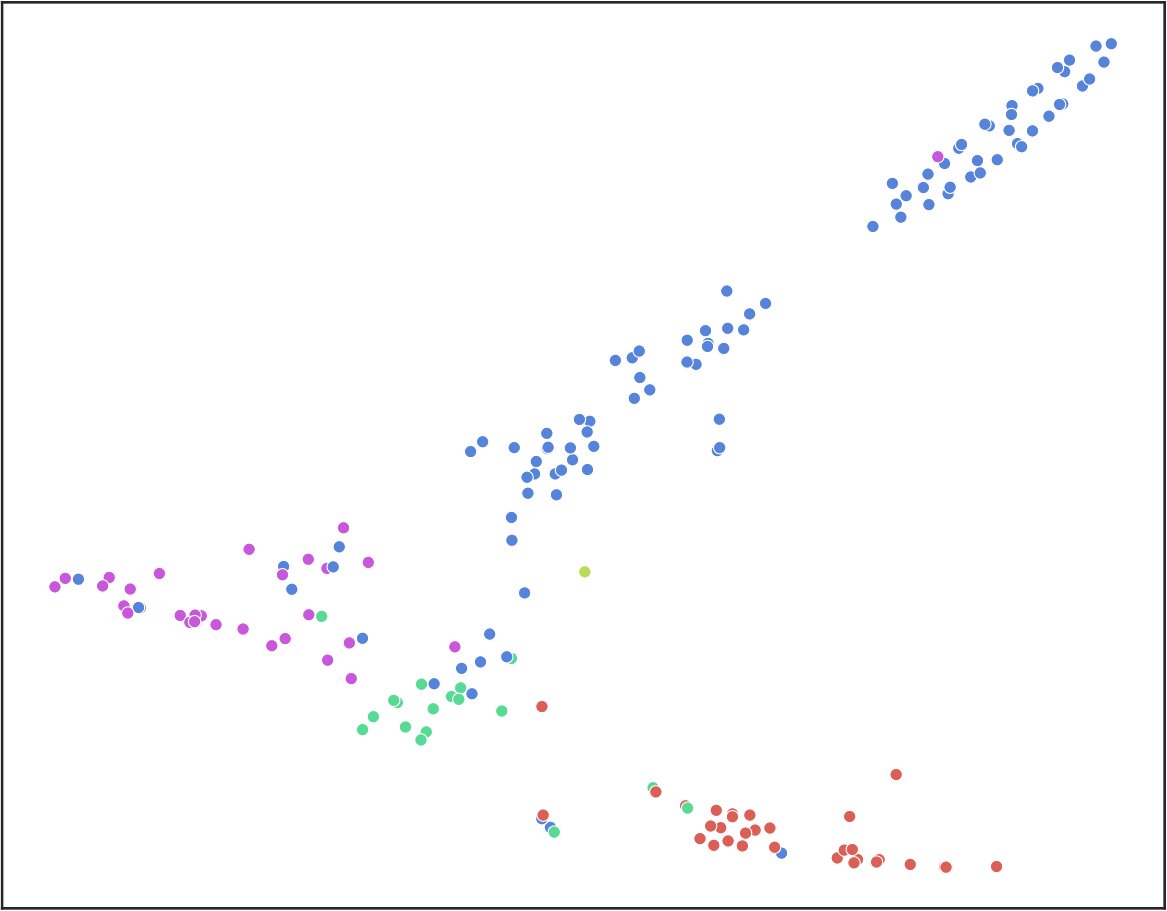}
\end{minipage}
} 
\subfigure[Actor(Original Feature)]{
\begin{minipage}[b]{0.3\linewidth}
\centering
\includegraphics[width=4.3cm]{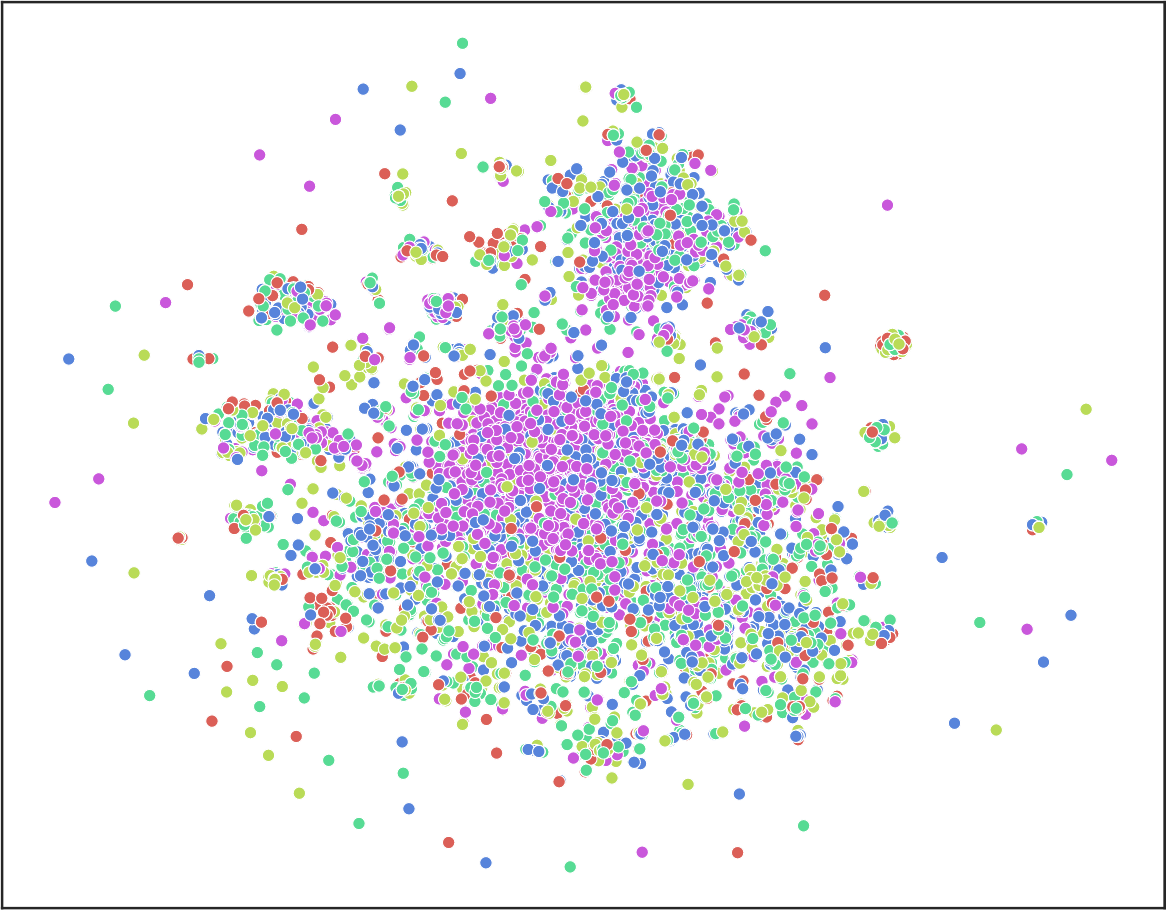}
\end{minipage}
}
\subfigure[Actor(SGC)]{
\begin{minipage}[b]{0.3\linewidth}
\centering
\includegraphics[width=4.3cm]{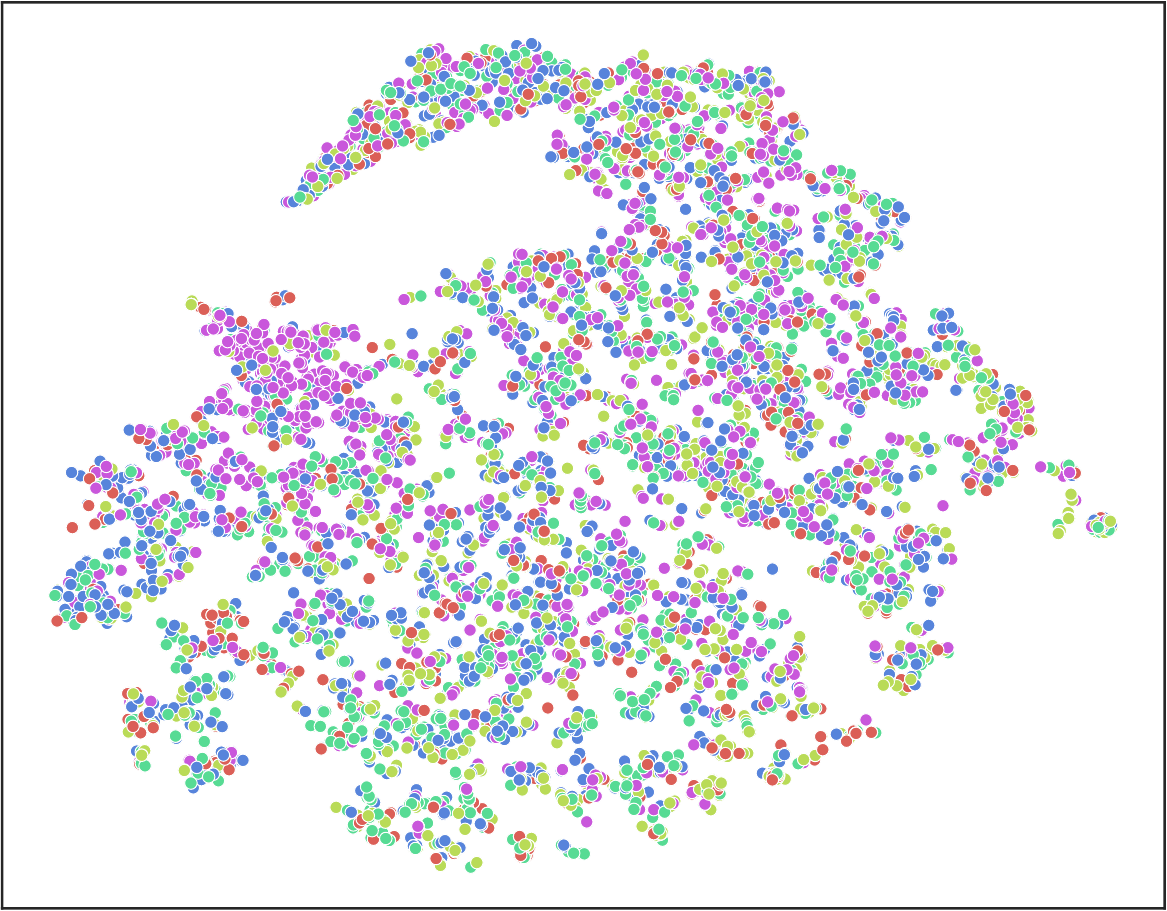}
\end{minipage}
}
\subfigure[Actor(Ours)]{
\begin{minipage}[b]{0.3\linewidth}
\centering
\includegraphics[width=4.3cm]{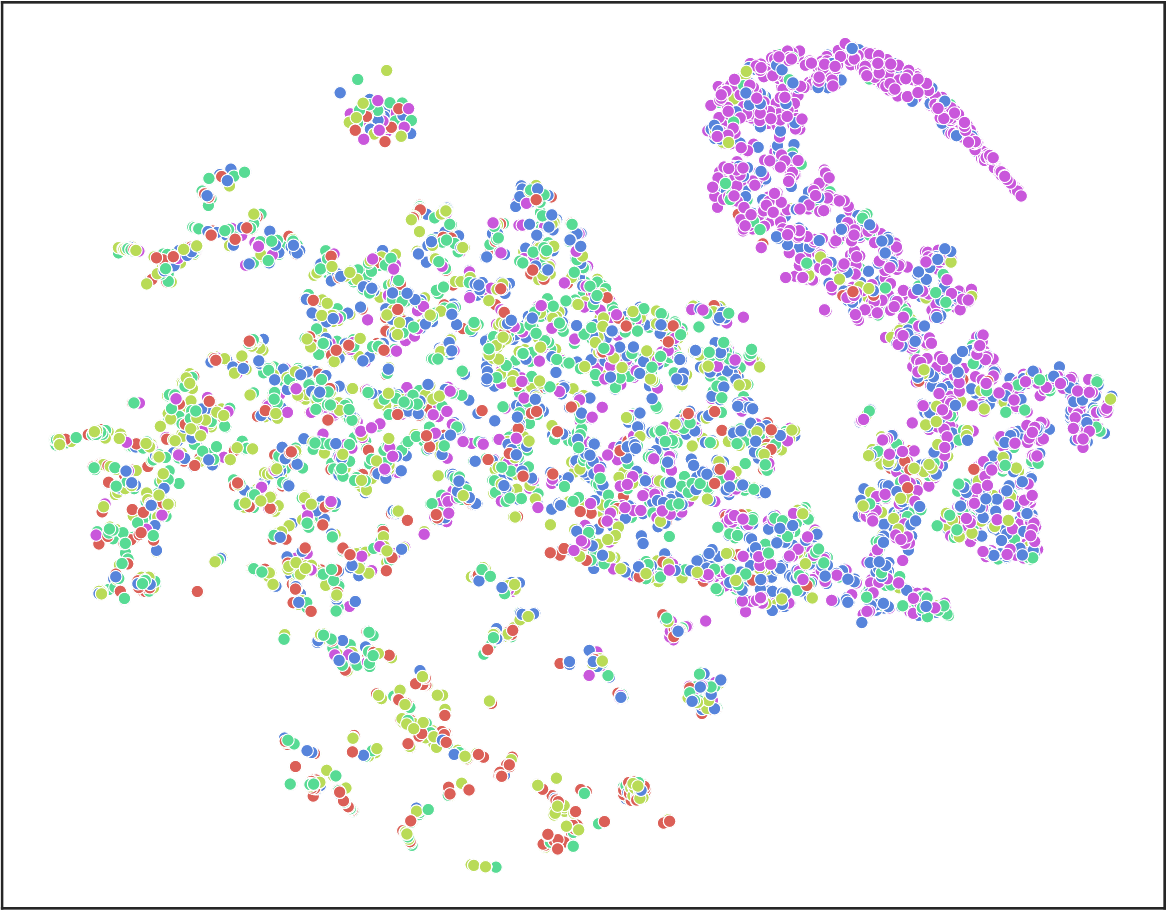}
\end{minipage}
}   
\caption{Scatter plots of original features and embeddings output via 32-layer SGC and our model on real-world benchmarks Texas, Wisconsin, Cornell, and Actor}
\label{fig_appendix:embeddings_texas_wisconsin_cornell_actor}
\end{figure*}

\end{document}